\documentclass[format=acmsmall, screen=true]{acmart}

\newif\ifsubmission
\submissionfalse

\usepackage{url}  %
\usepackage{graphicx}  %
\usepackage{amsmath,amsthm,xspace,amssymb}
\usepackage{latexsym}  %
\usepackage{wrapfig}
\usepackage{chngcntr}
\usepackage{ifthen}
\ifsubmission\else
\usepackage{hyperref}
\fi
\usepackage{multirow}
\usepackage[utf8]{inputenc}
\usepackage{dblfloatfix}
\usepackage{booktabs} 
\usepackage{microtype}
\usepackage{mathtools}
\usepackage{enumerate}
\usepackage{array}
\ifsubmission
\newcommand{\citet}[1]{\citeauthor{#1}~\shortcite{#1}}
\newcommand{\citep}{\cite}

\newcommand{\citeyearpar}[1]{(\citeyear{#1})}
\else
\fi
\newcommand{\wbox}{\mbox{$\sqcap$\llap{$\sqcup$}}}
\usepackage{tikz}
\newcommand{\nSAp}{n_{S,A,i}^+}

\newcommand{\Feas}{\textit{Feas}}

\ifsubmission
\fi

\usetikzlibrary{shapes,arrows.meta,calendar,matrix,positioning}
\newcommand{\commentout}[1]{}

\newif\iffullv
\fullvtrue

\iffullv
\newcommand{\shortv}{\commentout}
\newcommand{\fullv}[1]{#1}
\else
\newcommand{\shortv}[1]{#1}
\newcommand{\fullv}{\commentout}
\fi %

\newcommand{\IN}{\mathrm{I\!N}}
\newcommand{\IR}{\mathrm{I\!R}}
\def\beginsmall#1{\vspace{-\parskip}\begin{#1}\itemsep-\parskip}
\def\endsmall#1{\end{#1}\vspace{-\parskip}}

\ifsubmission\else
\hypersetup{linkcolor=blue!70!red,colorlinks=true}
\fi

\title[Information Acquisition Under Resource Limitations]{Information Acquisition Under Resource Limitations in a Noisy Environment}

\author{Matvey Soloviev}
\affiliation{%
  \institution{Computer Science Department, Cornell University}
  \streetaddress{107 Hoy Rd}
  \city{Ithaca}
  \state{NY}
  \postcode{14853}
  \country{USA}}
\email{msoloviev@cs.cornell.edu}
\author{Joseph Y. Halpern}
\affiliation{%
  \institution{Computer Science Department, Cornell University}
  \streetaddress{107 Hoy Rd}
  \city{Ithaca}
  \state{NY}
  \postcode{14853}
  \country{USA}}
\email{halpern@cs.cornell.edu}


\def\Tr{\mathop{\mathrm{Tr}}}

\newcommand{\x}{\times}

\makeatletter
\protected\def\specialmergetwolists{%
  \begingroup
  \@ifstar{\def\cnta{1}\@specialmergetwolists}
    {\def\cnta{0}\@specialmergetwolists}%
}
\def\@specialmergetwolists#1#2#3#4{%
  \def\tempa##1##2{%
    \edef##2{%
      \ifnum\cnta=\@ne\else\expandafter\@firstoftwo\fi
      \unexpanded\expandafter{##1}%
    }%
  }%
  \tempa{#2}\tempb\tempa{#3}\tempa
  \def\cnta{0}\def#4{}%
  \foreach \x in \tempb{%

    \xdef\cnta{\the\numexpr\cnta+1}%
    \gdef\cntb{0}%
    \foreach \y in \tempa{%
      \xdef\cntb{\the\numexpr\cntb+1}%
      \ifnum\cntb=\cnta\relax
        \xdef#4{#4\ifx#4\empty\else,\fi\x#1\y}%
        \breakforeach
      \fi
    }%
  }%
  \endgroup
}
\makeatother

\def\Pr{\mathop{\mathrm{Pr}_{D,\vec{\alpha}}}\nolimits}
\def\PrS#1{\mathop{\mathrm{Pr}_{D,\vec{\alpha},#1}}\nolimits}
\def\PrT#1{\mathop{\mathrm{Pr}_{#1}}\nolimits}
\def\Q{\mathop{\mathrm{Q}}\nolimits}

\def\E{\mathbb{E}}
\tikzset{>=Stealth}
\def\cpl{\mathop{\mathrm{cpl_{D,q,\m}}}}

\def\t{\mathrm{T}}
\def\f{\mathrm{F}}
\def\mv{\vec{\alpha}}
\def\m{\vec{\alpha}}
\def\mt{\approx}

\def\cf{\mathrm{cf}}

\def\subst#1#2{|_{#1=#2}} 

\renewcommand{\phi}{\varphi}
\renewcommand{\varepsilon}{\epsilon}

\begin{document}

\begin{abstract}
We introduce a theoretical model of information acquisition under
resource limitations in a noisy environment.
An agent must guess the truth value of a
given Boolean formula $\varphi$ after 
performing a bounded number of noisy tests of the truth values 
of variables in the formula.
We observe that, in general, the problem of finding an optimal testing
strategy for $\phi$ is hard, but we suggest a useful heuristic.
The techniques we use also give insight into two apparently unrelated,
but well-studied problems: (1) \emph{rational inattention},
that is, when it is rational to ignore pertinent information
(the
optimal strategy may involve hardly ever testing variables that are clearly
relevant to $\phi$), and (2) what makes a formula hard to learn/remember.
\end{abstract}

\maketitle

\renewcommand{\phi}{\varphi}
\renewcommand{\setminus}{\backslash}

\section{Introduction}
Decision-making is  typically subject to resource constraints.
However, an agent may be able to choose how to allocate his resources.
We consider a simple decision-theoretic framework in which to examine this
resource-allocation problem.   
Our framework is motivated by a variety of decision problems in which
multiple noisy signals are available for sampling, such as the following:
\begin{itemize}
\item An animal must decide
whether some food is safe to eat.  We assume that ``safe'' is characterised by
a 
Boolean formula $\phi$, which involves variables that describe (among
other things) the
presence of unusual smells or signs of other animals consuming the
same food. The animal can perform a limited number of
tests of the variables in $\phi$, but these tests
are noisy; if a test says that a variable $v$ is true, that does not
mean that $v$ is true, but only that it is true with some probability.
After the agent has exhausted his test budget, he must either guess
the truth value of $\phi$ or choose not to guess.  Depending on
his choice, 
he gets a payoff.  In this example, guessing that $\phi$ is
true amounts to guessing that the food is safe to eat.  There will be a
small positive payoff for guessing ``true'' if the food is indeed
safe, but a large negative payoff for guessing ``true'' if the food is
not safe to eat.  In this example we can assume a payoff of 0 if the
agent guesses ``false'' or does not guess, since both choices amount
to not eating the food.

\item A quality assurance team needs to certify a modular product,
say a USB memory stick, or send it back to the factory. Some subsystems,
such as the EEPROM cells, are redundant to an extent, and a limited
number of them not working is expected and does not stop the product
from functioning. Others, such as the USB controller chip, are unique;
the device will not work if they are broken. Whether the
device is good can be expressed as a Boolean combination of variables
that describe the goodness of 
its components. Time and financial considerations allow only
a limited number of tests to be performed, and tests themselves have
a probability of false negatives and positives. What parts should
be tested and how often?

\item A data scientist wants to perform a complex query on a very
 big database. A certain error rate is acceptable;  in any case,
 executing the query exactly is infeasible with the available hardware.
The selection criterion itself is a Boolean 
combination of some atomic predicates on the entries of the database, which
can be evaluated only using heuristics (which are
essentially probabilistic  algorithms). Given a query that asks for
rows that, for instance, satisfy the criterion $P_1 \wedge (P_2 \vee
P_3)$ in three predicates $P_i$, which heuristics should be run and
how often should they be run to attain the desired error rate?
\end{itemize}

We are interested in optimal strategies for each of these problems;
 that is,
what tests should the agent perform and in what order.  Unfortunately
(and perhaps not surprisingly), as we show, finding an optimal
strategy (i.e., one that obtains the highest expected payoff) is 
infeasibly hard.
We provide a heuristic that guarantees a positive expected payoff
whenever the optimal strategy gets a positive expected payoff.
Our analysis of this 
strategy also gives us the tools to examine two other problems of interest.

The first is \emph{rational inattention}, the notion
that in the face of limited resources it is sometimes rational to
ignore certain sources of information completely.
There has been a great deal of interest recently in this topic in
economics  \citep{Sims03,Wiederholt10}.  Here we show that optimal
testing strategies in our framework exhibit what can reasonably be
called rational inattention (which we typically denote RI from now
on). Specifically, 
our experiments show that
for a substantial fraction
of formulae, 
 an optimal strategy will hardly ever 
test variables that are clearly relevant to the outcome.
(Roughly speaking, ``hardly ever'' means that as the total number of
tests goes to infinity, the fraction of tests devoted to these
relevant variables goes to 0.)
For example, consider the
formula $v_1 \lor v_2$.  Suppose that the tests for $v_1$ and $v_2$ are
equally noisy, so there is no reason to prefer one to the other for
the first test.  But for 
certain choices
 of payoffs, we show that
if we start by testing $v_2$, then all subsequent tests should also test
$v_2$
as long as $v_2$ is observed to be true 
(and similarly for $v_1$).  Thus,
with positive probability,
the optimal strategy either
ignores $v_1$ or ignores $v_2$.
Our formal analysis allows us to conclude that this is a widespread phenomenon.

The second problem we consider is what makes a concept (which we can
think of as being characterised by a formula) hard.  To address this,
we use our framework to define a notion of hardness.
Our notion is based on the minimum number of tests required
to have a chance of making a reasonable guess regarding whether the formula
is true.
  We show that,
according to this definition, XORs (i.e., formulae of the form $v_1
\oplus \cdots \oplus v_n$, which are true exactly if an odd number of the
$v_i$'s are true) and their negations are the hardest formulae.  We
compare this notion to other notions of hardness of concepts considered in the
cognitive psychology literature (e.g., \citep{Feldman,LMG04,SHJ61}).

\textbf{Organisation.}
The rest of the paper is organized as follows.
In Section \ref{sec:iagames}, we formally
define the games that we use to model our decision problem and analyse
the optimal strategies for a simple example. The detailed calculations for this 
example can be found in Appendix \ref{app:twovaror}. In Section \ref{sec:optstrat},
we look at the problem of determining optimal strategies more generally.
We discuss the difficulty of this problem and analyse a simple heuristic,
developing our understanding of the connection between payoffs and certainty
in the process.
In Section \ref{sec:ratinatt},
we formally define rational inattention and discuss the intuition
behind our definition.
After considering 
some examples of when RI occurs under our definition, we show that
there is a close connection 
between rational inattention and particular sequences of observations
(\emph{optimal test outcome sequences}) that may occur while testing.
We use this connection to obtain a quantitative
estimate of how common RI is in formulae involving up to 10 variables.
The theory behind this estimate is presented in Appendix
\ref{sec:inattentionlp},
where we
relate the optimal test outcome sequences to the solution polytope of 
a particular linear program (LP). While we are not aware of any explicit
connections, our method should be seen in a broader tradition of applying
LPs to decision problems such as multi-armed bandits \cite{chen86}, and may
be of independent interest for the analysis of information acquisition.
Finally, in Section \ref{sec:testcomp}, we introduce our notion of test
complexity, prove that XORs are the formulas of greatest test
complexity (the  
details of the proof are in Appendix \ref{sec:proofxorhardest}), and
discuss the connections to various other notions of formula complexity
in the cognitive and computational science literature.

\section{Information-acquisition games}\label{sec:iagames}

We model the \emph{information-acquisition game} as a single-player
game against nature, 
that is, one in which actions that are not taken by the player
are chosen at random.
The game is characterised by five parameters:
\begin{itemize}
\item a Boolean formula $\phi$ 
over
 variables $v_1, \ldots,
  v_n$ for some $n > 0$;
\item a probability distribution $D$ on truth assignments to $\{v_1,
  \ldots, v_n\}$;
\item a bound $k$ on the number of tests;
\item an \emph{accuracy vector} $\mv = (\alpha_1, \ldots, \alpha_n)$, with
    $0 \le \alpha_i \le 1/2$
  (explained below);
\item payoffs $(g,b)$, where $g > 0 > b$ (also explained below).
\end{itemize}
We denote this game as $G(\phi,D,k,\m,g,b)$.  

In the game $G(\phi,D,k,\m,g,b)$, nature first chooses a truth
assignment to the variables 
$v_1, \ldots, v_n$ according to distribution $D$.
While the parameters
of the game are known to the agent, the assignment chosen by nature
is not.
For the next $k$ rounds, 
the agent
 then chooses one of the $n$ variables to test 
 (possibly
 as a
function of history), and nature responds with either $T$ or $F$.  The
agent then must either guess the truth value of $\phi$ or choose not
to guess.

We view a truth assignment $A$ as a function from variables to
truth values ($\{T,F\}$); we can also view a formula 
as a function from truth assignments to truth values.
If the agent chooses to test $v_i$, then nature
returns $A(v_i)$ (the right answer) with probability $1/2 + \alpha_i$
(and thus returns $\neg A(v_i)$ with probability $1/2 - \alpha_i$).%
\footnote{
Note that this means that the
probability of a false positive and that of a false negative are the
same.  While we could easily extend the framework so as to allow the
accuracy in a test on a variable $v$ to depend on whether $A(v)$ is $T$
or $F$, doing so would complicate notation and distract from the main
points that we want to make.}
Thus, outcomes are independent, conditional on a truth assignment.
Finally, if the agent choses not to guess at the end of the game, his
payoff is 0.  If he chooses to guess, then his payoff is $g$ (good) if
his guess coincides with the actual truth value of $\phi$ on assignment $A$
(i.e., his guess is correct) and $b$ (bad) if his guess is wrong.
It is occasionally useful to think of a formula $\phi$ as a function
from assignments to truth values; we thus occasionally write $\phi(A)$
to denote the truth value of $\phi$ under truth assignment $A$.
A strategy for an agent in this game can be seen as a pair of functions:
one that determines which test the agent performs after observing
a given sequence of test outcomes of length $<k$,
and one that determines the whether to make a guess and, if so, which
guess to make, given all $k$ test outcomes.

\begin{example}\label{twovaror} Consider the information-acquisition
  game over the formula $v_1 \lor v_2$, with $k=2$ tests,
  a uniform distribution on truth assignments,
  accuracy
    vector $(1/4,1/4)$, correct-guess reward $g=1$ and
  wrong-guess penalty $b= -16$.
\fullv{As we show (see Appendix~\ref{app:twovaror})}
\shortv{As we show (see Appendix~A in the full paper)}
  this game has two  optimal strategies:
\begin{enumerate}
\item test $v_1$ twice, guess $T$ if both tests came out $T$,
  and make no guess otherwise; 
  \item test $v_2$ twice, guess $T$ if both tests came out
        $T$, and make no guess otherwise.  \hfill \wbox
\end{enumerate}
\end{example}
Thus, in this game, an optimal strategy either ignores $v_1$ or
ignores $v_2$.
As we show in Appendix~\ref{app:twovaror}, the strategy ``test $v_1$
and then $v_2$, then guess $T$ if both tests came 
out $T$'' is strictly worse than these two;
in fact, its expected payoff is negative!

If we increase $k$,
the situation becomes more nuanced.
For instance, if $k=4$, an optimal strategy tests $v_1$
once, and if the test comes out $F$, tests $v_2$ three times and 
guesses $T$ if all three tests came out $T$. However, it always
remains optimal to keep testing one variable as long as the tests
keep coming out true.
That is, all optimal strategies exhibit RI in the
sense that there are test outcomes that result in either $v_1$ never
being tested or $v_2$ never being tested, despite their obvious
relevance to $v_1 \lor v_2$.

For our results, we need to analyze the probability of various events
related to the game. Many of the probabilities that
we care about depend on only a few parameters of the game.
Formally, we put a probability on \emph{histories} of an
information-acquisition game. A history is a tuple of the form
$(A,S,a)$,
where $A$ is the assignment of truth values to the $n$ variables
chosen by nature, 
$ S=(v_{i_1}\mt b_1, \ldots, v_{i_k}\mt b_k) $
is a \emph{test-outcome sequence} in which
$v_{i_j}\mt b_j$ indicates that the $j$th test was performed on variable
$v_{i_j}$ and that nature responded with the test outcome $b_j$, and
$a$ is the final agent action 
of either making no guess or guessing some truth value for the formula.
A game $G(\phi,D,k,\m,g,b)$ and agent strategy
$\sigma$
for this game 
then induce a probability $\PrT{G,\sigma}$ on this sample space.
\begin{example}
In Example \ref{twovaror}, $\PrT{G,\sigma}(\varphi)$ is $3/4$,
as we know only that there is a probability of $3/4$ that nature picked
a satisfying assignment. After observing a single test outcome
suggesting that $v_1$ is false, the posterior probability
$\PrT{G,\sigma}(\varphi\mid (v_1\approx F))$ drops to $5/8$.
If the same test is performed and the outcome is again $F$, 
the posterior drops further to
$\PrT{G,\sigma}(\varphi\mid (v_1\approx F,v_1\approx F)) = 11/20$.
\wbox
\end{example}
The only features of the game $G$ that affect the probability are the prior
distribution $D$ and the accuracy vector $\alpha$,
so we write 
$\PrT{D,\alpha,\sigma}(\phi)$ rather than $\PrT{G,\sigma}(\phi)$.
If some component of the subscript does not affect the probability,
then we typically omit it.
In particular, we show in Appendix B \shortv{of the full paper} that
the strategy $\sigma$ does not affect $\PrT{G,\sigma}(\phi\mid S)$, so we
write $\Pr(\phi\mid S)$.
Finally, the utility (payoff) received by the agent at the end of the
game is a
real-valued random variable that depends on 
parameters $b$ and $g$. 
We can define the expected utility $\E_{G,\sigma}(\mathrm{payoff})$
as the expectation of this random variable.

\section{Determining optimal strategies}
\label{sec:optstrat}
It is straightforward to see that the game tree%
\footnote{For the one-player games that we are considering, a game
  tree is a graph whose nodes consist of 
all valid partial sequences of actions in the game, including the
empty sequence, and two nodes have an edge between them if they differ
by appending one action.} 
 for
the game $G(\phi,D,k,\m,g,b)$  
has $3(2^n)(2n)^k$ leaves: there is a
branching factor of $2^n$ at the root (since there are $2^n$ truth
assignments) followed by $k$ branching factors of $n$ (for the $n$
variables that the agent can choose to test) and 2 (for the two
possible outcomes of a test).  At the end there are three choices
(don't guess, guess $T$, and guess $F$).  A straightforward backward
induction can then be used to compute the optimal strategy.
Unfortunately, the complexity of this approach is polynomial in the
number of leaves,
 and hence grows exponentially in $k$
even for a fixed number of variables $n$, quickly becoming infeasible.

In general, it is unlikely that the dependency on $2^n$ can be
removed.
In the special case that $b=-\infty$ and $\alpha_i=\frac{1}{2}$ for all $i$
(so tests are perfectly accurate, but the truth value of the formula must be
established for sure), determining whether
there is a strategy that gets a positive expected payoff
when the bound on tests is $k$
reduces to the problem
of finding a conjunction of length $k$ that implies a given Boolean
formula. Umans \citeyearpar{Umans1999} showed that this problem is
$\Sigma_2^p$-complete, so it lies in a complexity class that is
at least as hard as both NP and co-NP.
A simple heuristic (whose choice of variables is independent of $\phi$) would be
to simply test each variable in the formula $k/n$ times,
and then choose the action that maximises the
expected payoff given the observed test outcomes.
We can calculate in time polynomial in $k$ and $n$ the
expected payoff of a guess, conditional on a sequence of test outcomes.
Since determining the best guess involves checking the likelihood of
each of the $2^n$ truth assignments conditional on the outcomes, this
approach takes time polynomial in $k$ and $2^n$.
We are most interested in formulae where $n$ is small
(note $k$ still can be large, since we can test a variable multiple times!),
so this
time complexity would be
acceptable.  However, this approach can be arbitrarily worse than the
optimum.  As we observed in Example {\ref{twovaror}}, the expected
payoff of this strategy is
negative, while there is a strategy that has positive expected
payoff.  

An arguably somewhat better heuristic,
which we call the \emph{random-test heuristic},
is to choose, at every step,
the next variable to test uniformly at random, and again,
after $k$ observations, choosing the action that maximises the expected payoff.
This heuristic
clearly has the same time complexity as the preceding one, while 
working better in information-acquisition games that require
an unbalanced approach to testing.

\begin{proposition} \label{prop:randvars}
  If there  exists a strategy that has positive expected payoff
in the information-acquisition game $G$,
then the random-test heuristic has positive expected payoff.
\end{proposition}

To prove Proposition~\ref{prop:randvars}, we need a preliminary lemma.  
Intuitively, an optimal strategy 
should try to generate test-outcome sequences 
$S$
that maximise $|\Pr(\phi \mid S) - 1/2|$, since
 the larger $|\Pr(\phi \mid S) - 1/2|$ is, the more certain the agent
is regarding whether $\phi$ is true or false.  The following lemma
characterises how large $|\Pr(\phi \mid S) - 1/2|$ has to be to get a
  positive expected payoff.

 \begin{definition}
Let $q(b,g)= \frac{b+g}{2(b-g)}$ be the \emph{threshold} associated
with payoffs $b,g$.
\hfill
\wbox
\end{definition}

\begin{lemma}\label{lem:payoff}
 The expected payoff of $G(\phi,D,k,\m,g,b)$ 
  when making a guess after observing
    a sequence $S$ of test outcomes is positive iff
\begin{equation}\label{eq1}
    \left| \Pr(\varphi\mid S) - 1/2 \right| >
q(b,g).
\end{equation}
\end{lemma}

\begin{proof} 
    The expected payoff when guessing that the formula is true is
  $$ g\cdot \Pr(\varphi \mid S) + b \cdot (1-\Pr(\varphi \mid S)). $$
This is greater than zero iff
$$ (g-b) \Pr(\varphi \mid S) + b > 0, $$
that is, iff
$$ \Pr(\varphi \mid S) - 1/2 > \frac{b}{b-g} - \frac{1}{2} =
q(b,g).$$
When guessing that the formula is false, we simply exchange $\Pr(\varphi \mid
S)$ and $1-\Pr(\varphi\mid S)$ 
in the derivation. So the payoff is then positive iff
$$ (1-\Pr(\varphi \mid S)) - \frac{1}{2} = -(\Pr(\varphi \mid S) - \frac{1}{2})>
q(b,g).$$
Since $|x|=\max \{x,-x\}$, at least one of these two inequalities must
hold if (\ref{eq1}) does,  
so the corresponding guess will have positive expected payoff.
Conversely, since $|x| \geq x$, either inequality holding implies (\ref{eq1}).
\end{proof}

\begin{proof}[Proof of Proposition \ref{prop:randvars}]
  Suppose that $\sigma$ is a strategy for $G$ with positive expected
  payoff.
The test-outcome sequences of length $k$ partition the space of 
paths in the game tree, 
so we have
$$\sum_{\{S: |S|=k\}}
\PrS{\sigma}(S) \, \E_{G,\sigma}(\text{payoff}\mid S). $$ 
Since the payoff is positive, at least one of the summands on the
right must be, say the one due 
to the sequence $S^*$.
By Lemma~\ref{lem:payoff},
$\left| \Pr(\varphi\text{ is true}\mid S^*) - 1/2 \right| >
q(b,g)$.

Let $\tau$ denote the random-test heuristic.
Since $\tau$ chooses the optimal action after making $k$ observations,
it will not 
get a negative expected payoff for any sequence $S$ of $k$ test outcomes
(since it can always obtain a payoff of 0 by choosing not to guess). 
On the other hand, with positive probability, the variables that make up
the sequence $S^*$ will be chosen and the outcomes in $S^*$ will be
observed for these tests; that is $\PrS{\tau}(S^*) > 0$.  
It follows from Lemma~\ref{lem:payoff} that $\E_{G,\tau}(\mbox{payoff}
\mid S^*) > 0$.  Thus, $\E_{G,\tau}(\mbox{payoff}) > 0$, as desired.
\end{proof}

\section{Rational inattention}
\label{sec:ratinatt}

We might think that an optimal strategy for learning about $\phi$
would test all variables that are
relevant to $\phi$ (given a sufficiently large test budget).  As shown
in Example~\ref{twovaror}, this may not 
be true.  For example, an optimal $k$-step strategy for $v_1 \lor v_2$
can end up never testing $v_1$, no matter what the value of $k$, if it
starts by testing $v_2$ and keeps discovering that $v_2$ is true.
It turns out that RI is quite
widespread.  

It certainly is not surprising that if a variable $v$ does not occur in
$\phi$, then an optimal strategy would not test $v$.  More generally,
it would not be surprising that a variable that is not
particularly relevant to $\phi$ is not tested too often,
perhaps because it makes a difference only in rare edge cases.
In the foraging animal example from the introduction, the
possibility of a human experimenter having prepared a safe food
to look like a known poisonous plant would impact whether it
is safe to eat, but is unlikely to play a significant role in day-to-day
foraging strategies. 
What might seem more surprising is if a variable $v$ is (largely)
ignored while another variable $v'$ that is no more relevant than $v$
is tested.  This is what happens in Example~\ref{twovaror}; although
we have not yet defined a notion of relevance, symmetry considerations
dictate that $v_1$ and $v_2$ are equally relevant to $v_1 \lor v_2$,
yet an optimal strategy might ignore one of them.

The phenomenon of rational inattention observed in
Example~\ref{twovaror} is surprisingly widespread.  To make this claim
precise, we need to define ``relevance''.  
There are a number of reasonable ways of defining it; we focus
on one below.%
\footnote{We checked various other reasonable definitions
  experimentally; qualitatively, it seems that our results continue to
  hold for all the variants that we tested.}
The definition of the relevance of $v$ to $\phi$ that we use counts
the number of truth assignments for which changing the truth value of
$v$ changes the truth value of $\phi$.  
\begin{definition} Define the relevance ordering $\le_{\phi}$ on the
  variables in $\phi$ by taking 
$$    \begin{array}{lll}
 & &   v \leq_\phi v'  \text{ iff } \\
  & &   |\{ A : \phi(A[v \mapsto \t]) \ne \phi(A[v \mapsto \f]) \}| \\ &\le& 
  |\{ A : \phi(A[v' \mapsto \t]) \ne \phi(A[v' \mapsto \f]) \}|,
  \end{array}$$
where $A[v \mapsto b]$ is the assignment that agrees with $A$ except
that it assigns truth value $b$ to $v$.
\hfill \wbox
\end{definition}
Thus, rather than saying that $v$ is or is not relevant to $\phi$, we
can say that $v$ is (or is not) at least as relevant to $\phi$ as
$v'$.
Considering the impact of a change in a single variable to the truth
value of the whole formula in this fashion has been done both in the cognitive
science and the computer science literature: for example, Vigo
\citeyearpar{Vigo2011} uses the 
\emph{discrete (partial) derivative} to capture this effect, and Lang
et al. \citeyearpar{lang03} 
define the related notion of \emph{Var-independence}.

We could also consider taking the probability of the set of truth
assignments where a variable's value makes a difference, rather than
just counting how many such truth assignments there are.
This would give a more detailed quantitative view of relevance, and
is essentially how relevance is considered in Bayesian
networks. Irrelevance is typically identified with independence.
Thus, $v$ is relevant to $\phi$ if a change to $v$ changes the
probability of $\phi$.  (See Druzdzel and Suermondt \citeyearpar{druzdzel94b}
for a review of work on relevance in the context of Bayesian networks.)
We did not consider a probabilistic notion of relevance because then
the relevance order would depend on the game
(specifically, the distribution $D$, which is one of the parameters of
the game).  Our definition makes the relevance order depend only on
$\phi$.  That said, we believe that essentially the same results as
those that we prove could be obtained for a probabilistic notion of relevance
ordering.  

Roughly speaking, $\phi$ exhibits RI
if, for all optimal strategies $\sigma$ for the game $G(\phi, D, k,
\vec{\alpha}, b, g)$, 
 $\sigma$ tests a
variable $v'$ frequently while hardly ever testing a variable
$v$ that is at least as relevant to $\phi$ as $v'$.  
We still have to make precise ``hardly
ever'', and explain how the claim depends on the choice of $D$, 
$\vec{\alpha}$, $k$, $b$, and $g$.   For the latter point, note that in
Example~\ref{twovaror}, we had to choose $b$ and $g$ appropriately to
get RI.  This turns out to be true in general; given
$D$, $k$, and $\vec{\alpha}$, the claim holds only for an appropriate
choice of $b$ and $g$ that depends on these.
In particular, for any fixed choice of $b$ and $g$ that depends only on
$k$ and $\vec{\alpha}$, there exist choices of priors $D$ for which
the set of optimal strategies is fundamentally uninteresting: we
can simply set $D$ to assign a probability to some truth assignment $A$
that is so high 
that the rational choice is always to guess $\phi(A)$,
regardless of the test outcomes.

Another way that the set of optimal strategies can be rendered
uninteresting is when, from the outset, there is no hope
of obtaining sufficient certainty of the formula's truth value
with the $k$ tests available. Similarly to when the truth value
is a foregone conclusion, in this situation,
an optimal strategy can perform arbitrary tests, as long as it makes no
guess at the end.
More generally, even when in general the choice of variables to test does
matter, a strategy can reach a situation where there is sufficient
uncertainty that no future test outcome could affect the final choice.
Thus, a meaningful definition of
RI that is based on the variables tested by
optimal strategies must consider only tests performed
in those cases in which a guess actually should be made (because the
expected payoff of the optimal strategy is positive).%
\footnote{One way to avoid these additional requirements
is to modify the game so that performing a test is associated has 
a small but positive cost,
so that an optimal strategy avoids frivolous testing
when the conclusion is foregone.
The definitions we use have
essentially the same effect, and are easier to work with.}  We now
make these ideas precise. 

\begin{definition} A function $f: \IN \rightarrow \IN$ is
  \emph{negligible} 
if $f(k)=o(k)$, 
that is, if $\lim_{k \rightarrow \infty} f(k)/k  = 0$.
\hfill \wbox
\end{definition}
The idea is that $\phi$ exhibits RI if, as the
number $k$ of tests allowed increases, the 
fraction of times that some variable $v$ is tested is negligible
relative to the number of times that another variable $v'$ is tested,
although $v$ is at least as relevant to 
$\phi$ as $v'$.  We actually require slightly more: we want $v'$ to 
be tested a linear number of times (i.e., at least $ck$ times, for
some constant $c > 0$).
(Note that this additional requirement makes it harder for a variable
to exhibit RI.)

Since we do not want our results to depend on correlations
between variables, we restrict attention to probability distributions
$D$ on truth assignments that are product distributions.
\begin{definition} A probability distribution $D$ on
  truth
  assignments to
  $v_1, \ldots, v_n$
  is a \emph{product distribution} if 
    $\PrT{D}(A) = \PrT{D}(v_1 = A(v_1)) \cdots \PrT{D}(v_n=A(v_n))$
  (where, for an arbitrary formula $\phi$, $\PrT{D}(\phi) = \sum_{\{A:\;
    A(\phi) = \t\}} \PrT{D}(A)$). 
  \hfill \wbox
\end{definition}
As discussed earlier, to get an interesting notion of
RI, we need to allow the choice of payoffs $b$ and $g$ 
to depend on the  prior distribution $D$; for fixed $b$, $g$, and
testing bound $k$, if the distribution $D$ places sufficiently high
probability on a single assignment, no $k$ outcomes can change the
agent's mind.
Similarly, assigning prior probability 1 to any one variable being
true or false means that no tests will change the agent's mind about
that variable, and so testing it is pointless (and the game is
therefore equivalent to one played on the formula in $n-1$ variables
where this variable has been replaced by the appropriate truth value).
We say that a
probability distribution that gives all truth assignments positive
probability 
is \emph{open-minded}. 
 
With all these considerations in hand, we can finally define
RI formally.
\begin{definition} \label{RIviastrategies} The formula $\varphi$ \emph{exhibits
    rational inattention}
if, for all open-minded product distributions $D$
and uniform accuracy vectors $\m$ (those with ($\alpha_1=\ldots=\alpha_n$)),
 there exists a negligible function $f$ and a constant $c>0$ such that
 for all $k$,
 there are
 payoffs $b$ and 
 $g$ such that
all optimal strategies in the information-acquisition game
$G(\phi,D,k,\m,b,g)$
have positive expected payoff and,
in all histories of the game, 
either make no guess or
\begin{itemize}
  \item test a variable $v'$ at least $ck$ times, but
      \item test a variable $v$ such that $v'\leq_{\phi}v$ at
    most $f(k)$ times.
    \hfill \wbox
\end{itemize}
\end{definition}

We can check in a straightforward way whether some natural classes of
formulae exhibit RI in the sense of
this definition.
\begin{example} \label{ex:ri} (Rational inattention)
\begin{enumerate}[1.]
\item Conjunctions $\phi=\bigwedge_{i=1}^N \ell_i$ and disjunctions
  $\phi=\bigvee_{i=1}^N \ell_i$ of $N\geq 2$ literals (variables
  $\ell_i=v_i$ or their negations $\neg v_i$) exhibit
  RI. In each
  case, we can pick $b$ and $g$ such that all optimal strategies pick
  one variable and focus on it, either to establish that the formula
  is false (for conjunctions) or that it is true (for
  disjunctions). By symmetry, all variables $v_i$ and $v_j$ are
  equally relevant, so $v_i\leq_\phi v_j$. 

\item The formulae $v_i$ and $\neg v_i$
  do not exhibit RI. There is no variable $v \ne
  v_i$ such that 
    $v_i \leq_{(\neg)v_i} v$, and  for all choices of $b$ and $g$, 
  the strategy of testing only $v_i$ and ignoring all other variables
  (making an
  appropriate guess in the end) is clearly optimal for $(\neg) v_i$.

\item More generally, we can say that all XORs in $\geq 0$ variables
do not exhibit RI. 
For the constant formulae $T$ and $F$, any testing strategy that ``guesses'' correctly
is optimal; for any XOR in more than one variable, an optimal strategy
must test all of them as any remaining uncertainty about the truth
value of some variable leads to at least equally great uncertainty
about the truth value of the whole formula.
Similarly, negations of XORs do not exhibit RI.  
Together with the preceding two points, this means
that the only formulae in $2$ variables exhibiting rational
inattention are those equivalent to one of the four conjunctions
$\ell_1 \wedge \ell_2$ 
or the four disjunctions $\ell_1 \vee \ell_2$ in which each
variable occurs exactly once and may or may not be negated.

\item For $n>2$, formulae $\phi$ of the form $v_1 \vee (\neg v_1
   \wedge v_2 \wedge \ldots \wedge v_n))$ 
      do not exhibit RI.
   Optimal strategies that can
attain a positive payoff at all will start by testing $v_1$;
if the tests come out true, it will be optimal to continue
testing $v_1$, ignoring $v_2 \ldots v_n$. However, for
formulae $\phi$ of this form, $v_1$ is strictly more relevant
than the other variables: there are only $2$
assignments where changing $v_i$ flips the truth value of the formula
for $i > 1$
(the two where $v_1\mapsto F$ and  $v_j\mapsto T$ for $j \notin \{1,i\}$)
but $2^n-2$ assignments where changing $v_1$ does (all but the two where
$v_j \mapsto T$ for $j \ne 1$). Hence, in the event that
all these tests actually succeed, the only variables that are ignored
are not at least as relevant as the only one that isn't, 
so $\phi$ does not exhibit RI.
\item For $n>3$, formulae $\phi$ of the form $(v_1\vee v_2)\wedge (v_2 \oplus \ldots \oplus v_n)$ exhibit RI. Optimal 
strategies split tests between $v_1$ and $v_2$, and try to establish that both variables are false and hence $\phi$ is;
to establish that the formula is true would require ascertaining the truth of the XOR term, and hence splitting the
testing budget at least 3 ways. However, $v_1$ is comparatively
irrelevant, as it determines only whether $\phi$
is true in $1/4$ of all assignments (when $v_2$ is false, and the XOR
is true). All other variables 
determine $\phi$'s truth value unless $v_1\vee v_2$ is false, that is, in
$3/4$ of all assignments. These 
formulae 
(and other similar families)
satisfy an even stronger definition of RI, as a strictly less
 relevant variable is preferred. 
\end{enumerate}
\hfill
\wbox
\end{example}

Unfortunately, as far as we know, determining the optimal
strategies is hard in general. 
To be able to reason about
whether $\phi$ exhibits RI in a tractable way, we find it useful to
consider optimal test-outcome sequences.

\begin{definition} \label{def:optseq} A sequence $S$ of test outcomes
    is \emph{optimal} for a formula $\phi$, prior $D$, and accuracy vector
$\m$ 
if it minimises the conditional uncertainty about the truth value of
$\phi$ among 
all test-outcome sequences of the same length.  That is,
$\left|\Pr(\phi\mid S) -\frac{1}{2}\right| \geq \left|\Pr(\phi\mid
S')-\frac{1}{2}\right|$ 
for all $S'$ with $|S'|=|S|$.
\hfill \wbox
\end{definition} 
Using this definition, we can derive a sufficient (but not
necessary!) condition 
for formulae to exhibit RI.

\begin{proposition} \label{RIviatestseqs}
  Suppose that, for a given formula     $\phi$,
for all open-minded product distributions $D$ 
 and uniform accuracy
  vectors $\m$, there exists a negligible function $f$ and a constant $c>0$ such that
    for all 
testing bounds $k$,
the test-outcome sequences $S$ optimal for $\phi$, $D$, and $\m$ 
of length $k$ have the
following two properties: 
\beginsmall{itemize}
\item $S$ has at least $ck$ tests of some variable $v'$, but
\item $S$ has at most $f(k)$ tests of some variable $v \geq_\phi v'$.
\endsmall{itemize}
Then $\phi$ exhibits RI.
\end{proposition}

\begin{proof}
  Let $P(\phi,D,\m,f,c,k)$ denote the statement that for all
test-outcomes sequences $S$ that are optimal for $\phi$, $D$, and $\m$,
there exist 
 variables $v \geq_\phi v'$ such that $S$ contains $\geq ck$ tests of $v'$ 
and $\leq f(k)$ tests of $v$.
We now prove that for all $\phi$, $D$, $\m$,
$f$, $c$, and $k$,  
$P(\phi,D,\m,f,c,k)$ implies
the existence of $b$ and $g$ such that $\phi$ exhibits RI in the game
$G(\phi,D,k,m,b,g)$.
It is easy to see that this suffices to prove the proposition.

Fix $\phi$, $D$, $\m$, $f$, $c$, and $k$,
and suppose that $P(\phi,D,\m,f,c,k)$ holds.
 Let $$q^* = \max_{\{S:|S|=k\}}
\left|\Pr(\varphi|S)-\frac{1}{2}\right|.$$
 Since there 
 are only finitely many test-outcome sequences of length $k$, there must be some
 $\varepsilon > 0 $ sufficiently small such that for all $S$ with $|S|=k$, 
  $|\Pr(\varphi|S)-\frac{1}{2}|>q^*-\varepsilon$ iff
$|\Pr(\varphi|S)-\frac{1}{2}|=q^*$.
Choose the payoffs $b$ and $g$ such that the threshold
$q(b,g)$ is $q^* - \varepsilon$.
We show that $\phi$ exhibits RI in the game $G(\phi,D,k,m,b,g)$.

Let $\mathcal{S}_k = \{S: |S| = k \mbox{ and } 
|\Pr(\varphi|S)-\frac{1}{2}|=q^*\}$ be the set of
test-outcome sequences of length $k$ optimal for $\phi$, $D$, and $\m$. 
If
$\sigma$ is an optimal strategy for the game $G(\phi,D,k,\m,g,b)$, 
the only sequences of test outcomes after which $\sigma$ makes
a guess are the ones in $\mathcal{S}_k$.
For if a guess is made after seeing some test-outcome sequence
$S^* \not\in \mathcal{S}_k$,
by Lemma~\ref{lem:payoff} and the choice of $b$  and $g$, 
the expected payoff of doing so must be negative, 
so the strategy $\sigma'$ that is identical to $\sigma$ except that it makes 
no guess if $S^*$ is observed is strictly better than $\sigma$, contradicting
the optimality of $\sigma$. So whenever a guess is made,
it must be after a sequence $S \in \mathcal{S}_k$ was observed.  
Since sequences in $\mathcal{S}_k$ are optimal for $\phi$, $D$, and
$\m$, and $P(\phi,D,\m,f,c,k)$ holds by assumption,
this sequence $S$ must 
contain $\ge ck$ test of $v'$ and $\le f(k)$ test of $v$.

\commentout{  
    $\Pr_\sigma(S)$ must be greater than zero for some $S\in
    \mathcal{S}_k$. For suppose not. Then the probability of the set of
    observation sequences of length $k$ that are not in $\mathcal{S}_k$ is
    outcome sequences of length $k$ that are not in $\mathcal{S}_k$ is
    1. By Lemma~\ref{lem:payoff} and the choice of $b$  and $g$, after each such
    outcome sequence, 
    either guess about the truth value of the formula gives 
    a negative expected payoff. Therefore, as $\sigma$ is optimal, it
    would  not make a guess with probability 1. So the expected
    payoff of $\sigma$ is 0.
}

All that remains to show that $\phi$ exhibits RI in the game
$G(\phi,D,k,\m,g,b)$ is to 
establish that all optimal strategies have
positive expected payoff.  To do this, it suffices to show that there
is a strategy that has positive expected payoff.  
Let $S$ be an arbitrary test-outcome  sequence in
$\mathcal{S}_k$.  Without loss of generality, we can assume that 
$\Pr(\phi \mid S)>1/2$.
Let  $\sigma_S$ be the strategy 
that tests 
every variable the number of times that it occurs in $S$ in the order
that the variables occur in $S$, and
guesses that the formula is true iff $S$ was in fact the
test-outcome sequence observed (and makes no guess otherwise).
Since $S$ will be observed with positive probability, it follows from
Lemma~\ref{lem:payoff} that $\sigma_S$ has positive expected payoff.
This completes the proof.
\end{proof}
Applying Proposition~\ref{RIviatestseqs} to test whether a formula
exhibits RI is not trivial.
It is easy to show that all that affects $\Pr(\phi \mid S)$ is the
number of times that
each variable is tested and the outcome of the
test, not the order in which the tests were made.
It turns out that to determine whether a formula $\phi$ exhibits RI,
we need to consider, for each truth assignment $A$ that satisfies $\phi$ and
test-outcome sequence $S$,
the \emph{$A$-trace} of $S$; this is a
tuple that describes, for each variable $v_i$, the fraction of times $v_i$ is
tested
(among all tests)
 and the outcome agrees with $A(v_i)$ compared to the fraction of times
that the outcome disagrees with $A(v_i)$.

In \shortv{the full paper}\fullv{Appendix \ref{sec:inattentionlp}},
we show that whether a formula exhibits 
RI can be determined by considering properties
of the $A$-traces of test-outcome sequences.   Specifically, we show
that the set of $A$-traces of
optimal test-outcome sequences 
 tends to a convex polytope as the length of $S$ increases. This polytope has
a 
characterisation as the solution set of an $O(n2^n)$-sized linear
 program (LP),
 so we can find points in the polytope in time polynomial in $2^n$. Moreover, 
 conditions such as a variable $v$ is ignored while a variable $v'$
 that is no more relevant than $v$ is not ignored correspond to
 further conditions on the LP, and thus can also be checked in time
 polynomial in $2^n$.
It follows that we can get a sufficient condition for a formula to
exhibit RI or not exhibit RI by evaluating a number of LPs of this
type. 

Using these insights, we were able to
exhaustively test all formulae that involve at most 4 variables
to see whether, as the number of tests in the game increases, optimal 
strategies were testing a more relevant variable a negligible number
of times relative to a less relevant variable.  
Since the criterion that we use is only a sufficient condition, not a
necessary one, we can give 
only a lower bound on the true number of formulae
that exhibit RI.  

In the following table, we summarise our results. The first column
lists the number of formulae that we are certain exhibit RI;
the second column
lists the remaining formulae, whose
status is unknown.  (Since RI is a semantic condition, when we say
``formula'', we really mean ``equivalence class of logically
equivalent formulae''. There are $2^{2^n}$ equivalence classes
of formulae with $n$ variables, so the sum of the two columns in 
the row labeled $n$ is $2^{2^n}$.)
As the results show, at least $15\%$ of formulae exhibit
RI. %

\begin{center}
\begin{tabular}{@{\hspace{1em}}r@{\hspace{4em}}c@{\hspace{4em}}
c@{\hspace{1em}}}
  \toprule
$n$ & exhibit RI 
 & unknown \\
\midrule
1 & 0 &  4 \\ 
2 & 8 &  8 \\ 
3 & 40 &  216 \\
4 & 9952 & 55584 \\
\bottomrule  
\end{tabular}
\end{center}

\commentout{
\begin{center}
\begin{tabular}{|r|ccc|}
\hline 
$n$ & exhibit RI & NTC $\Rightarrow$ RI & unknown \\
\hline
1 & 0 & 0 & 4 \\
2 & 8 & 0 & 8 \\
3 & 40 & 56 & 160 \\
4 & 9952 & 8248 & 47334 \\
\hline  
\end{tabular}
\end{center}
}
Given the numbers involved,
we could not exhaustively check what happens for $n \ge 5$.
However, we did randomly sample 4000 formulae that involved $n$
variables for $n = 5, \ldots, 9$. This is good enough for statistical
reliability: 
we can model the process as a simple random sample of a binomially
distributed parameter (the presence of RI),
and in the worst case (if its probability in the population of formulae
is exactly $\frac{1}{2}$), the 95\% confidence interval still has width
$\leq z\sqrt{\frac{1}{4000}\frac{1}{2}\left(1-\frac{1}{2}\right)} \approx 0.015$,
which is well below the fractions of formulae exhibiting RI that
we observe (all above $0.048$).
As the following table shows, RI continued to be
quite common.  Indeed, even for formulae with 9 variables, about 5\%
of the formulae we sampled  
exhibited RI.
\begin{center}
\begin{tabular}{@{\hspace{1em}}r@{\hspace{4em}}c@{\hspace{4em}}
c@{\hspace{1em}}}
  \toprule
$n$ & exhibit RI 
 & unknown \\
\midrule
5 & 585 & 3415 \\
6 & 506 & 3494 \\  
7 & 293 & 3707 \\
8 & 234 & 3766 \\
9 & 194 & 3806 \\
\bottomrule  
\end{tabular}
\end{center}

\commentout{
\begin{center}
\begin{tabular}{|r|ccc|}
\hline 
$n$ & exhibit RI & NTC $\Rightarrow$ RI & unknown \\
\hline
5 & 585 & 313 & 3102 \\
6 & 506 & 138 & 3356 \\  
7 & 293 & 63 & 3644 \\
8 & 234 & 30 & 3736 \\
9 & 194 & 10 & 3796 \\
\hline  
\end{tabular}
\end{center}
}

The numbers suggest that the fraction of formulae exhibiting RI
decreases as the number of variables increases.  However, since the formulae
that characterise situations of interest to people are likely to
involve relatively few variables (or have a structure like disjunction
or conjunction that we know exhibits RI), this suggests that RI is a
widespread phenomenon.
Indeed, if we weaken the notion of RI slightly (in what we believe is
quite a natural way!), then RI is even more widespread.
As noted in Example \ref{ex:ri}, formulae of the form
$v_1\vee (\neg v_1 \wedge v_2 \wedge \ldots \wedge v_n)$ do not
exhibit RI in the sense of our definition. However,
for these formulae, if we choose the payoffs $b$ and $g$
appropriately, an optimal
strategy may start by testing $v_1$, but if sufficiently many test
outcomes are $v_1 \mt F$, it will then try to establish that
the formula is false by focussing on one variable of the conjunction
$(v_2 \wedge \ldots \wedge v_n)$, and ignoring the rest.
Thus, for all optimal strategies, we would have RI, not for all
test-outcome sequences (i.e., not in all histories of the game), but on a
set of test-outcome sequences that occur with positive probability.

We found it hard to find formulae that do not
exhibit RI in this weaker sense.  In fact, we
conjecture that the 
only family of formulae that do not exhibit
RI in this weaker sense are equivalent to XORs in
zero or more variables $(v_1\oplus \ldots \oplus v_n)$ and their negations
(Note that this family of formulae includes $v_i$ and $\neg v_i$.)
If this conjecture is true, we would expect to quite often see 
rational agents (and decision-making computer programs) ignoring
relevant variables in practice.  

\section{Testing as a measure of complexity}
\label{sec:testcomp}
The notion of associating some ``intrinsic difficulty'' with concepts
(typically characterised using Boolean formulae) has been a topic of
continued interest in the cognitive science community \citep{Vigo2011,Feldman,LMG04,SHJ61}.
We can use our formalism to define a notion of difficulty for concepts.
Our notion of difficulty is based on the number of tests that are needed to 
guarantee a
positive expected payoff for the game $G(\phi,D,k,\m,g,b)$.  This
will, in general, depend on $D$, $\m$, $g$, and $b$.  Actually,
by Lemma \ref{lem:payoff},
what matters is not $g$ and $b$, but $q(b,g)$ (the threshold
determined by $g$ and $b$).  Thus, our complexity measure takes $D$,
$\m$, and $q$ as parameters.

\begin{definition}

  Given a formula $\phi$, accuracy vector $\m$, distribution $D$, and
  threshold $0<q\leq 
  \frac{1}{2}$, the \emph{($D, q,\m$)-test complexity} $\cpl(\phi)$ of $\phi$
  is the least $k$ such that there
exists a strategy with positive payoff for $G(\phi,D,k,\m,g,b)$,
where
$g$ and $b$ are chosen such that $q(b,g)=q$.
\hfill \wbox
\end{definition}

To get a sense of how this definition works, consider what happens if
we consider all formulae that use two variables, $v_1$ and $v_2$, with
the same settings as in Example~\ref{twovaror}: $\vec{\alpha} =
(1/4,1/4)$, $D$ is the uniform distribution on assignments, $g=1$, and
$b = -16$: 
\begin{enumerate}

\item If $\phi$ is simply $T$ or $F$, any strategy that guesses the appropriate
truth value, regardless of test outcomes, is optimal and gets a positive
expected payoff, even when $k=0$. 
So $\cpl(\phi)=0$.

\item If $\phi$ is a single-variable formula of the form $v_1$ or $\neg v_1$, then the
greatest certainty $| \Pr(\phi \mid S) - 1/2 |$ that is attainable
with any sequence
of two tests is $2/5$, when $S=(v_1\mt T, v_1\mt T)$ or the same with
$F$. This is 
smaller than $q(b,g)$, and so it is always optimal to make no guess;
that is, 
all strategies for the game with $k=2$ have expected payoff at most 0.
If $k=3$ and $S=(v_1\mt T, v_1\mt T, v_1\mt T)$, then $(\Pr(\phi \mid
S) - 1/2)  = 13/28 > q(b,g)$.
Thus, if $k=3$, the strategy that tests $v_1$ three times and guesses
the appropriate 
truth value iff all three tests agree has positive expected 
payoff.  It follows that $\cpl(\phi)=3$.

\item If $\phi$ is $v_1 \oplus v_2$, then the 
  shortest test-outcome 
    sequences $S$ for which $\Pr(\phi \mid S) - 1/2$ is greater than
    $q(b,g)$ have length 7, and 
involve both variables 
being tested. Hence, the smallest value of $k$ for which strategies with
payoff above $0$ exist is 
$7$, and $\cpl(\phi)=7$.

\item Per Example \ref{twovaror}, $\cpl(v_1 \vee v_2)=2$, and likewise for all
other conjunctions and disjunctions by symmetry.
\end{enumerate}

It is not hard to see that $T$ and $F$ always have complexity 0, while 
disjunctions and conjunctions
have low complexity. 
We can also
 characterise the most difficult concepts, according
to our complexity measure,
at least in the case of a uniform distribution $D_u$ on truth
assignments (which is the one most commonly considered in practice).

\begin{theorem}\label{thm:xor}
  Among all Boolean formulae in $n$ variables, for all $0<q\leq
  \frac{1}{2}$ and accuracy vectors
    $\m$, the $(D_u,q,\m)$-test complexity is maximised by 
formulae equivalent to   
the  $n$-variable XOR
$ v_1 \oplus \ldots \oplus v_n $ 
or its negation.
\end{theorem}
  \begin{proof}[Proof sketch]
        Call a formula $\phi$ \emph{antisymmetric} in variable $v$ if
$\phi(A) = \neg \phi(A')$ 
for all pairs of assignments $A$, $A'$ that only differ in the truth value of $v$.
It is easy to check that a formula is antisymmetric in all
variables iff it is equivalent to an XOR 
or a negation of one.
Given a formula $\phi$, the
\emph{antisymmetrisation $\phi_v$ of $\phi$   along $v$} is 
is the formula $$\phi_v = (v\wedge \phi\subst{v}{\t})\vee (\neg v\wedge
\neg \phi\subst{v}{\t}),$$ 
where $\phi\subst{v}{x}$ denotes the formula that results from replacing
all occurrences of $v$ in $\phi$ by $x$.  It is easy to chek that
$\phi_v$ is indeed antisymmetric in $v$.
We can show that the $(D_u,q,\m)$-test complexity of $\phi_v$
is at least as high as that of $\phi$, and that if $v' \ne v$, then
$\phi_v$ is antisymmetric in $v'$ iff $\phi$ is antisymmetric in $v'$.
So, starting with an arbitrary formula $\phi$, we antisymmetrise every
variable in turn.  We then end up with an XOR or the negation of one.
Moreover, each antisymmetrisation step in the process gives a formula
whose test complexity is at least as high as that of the formula in
the previous step.  The desired result follows.
A detailed proof can be found in
\shortv{the appendix of the full paper.}
\fullv{Appendix \ref{sec:proofxorhardest}.}
  \end{proof}

  Theorem~\ref{thm:xor} does not rule out the possibility that there are formulae
other than those equivalent to the $n$-variable XOR or its negation that
maximise test complexity.
We conjecture that this is not the case except when $q=0$; this
conjecture is supported by experiments we've done with formulas that
have fewer than 8 variables.

It is of interest to compare our notion of ``intrinsic difficulty''
with those considered in the cognitive science literature.
That literature can broadly be divided up into purely
experimental approaches, typically
focused on comparing the
performance of human subjects in dealing with different categories,
and more theoretical ones that posit some structural hypothesis
regarding which categories are easy or difficult. 

The work of Shepard, Hovland, and Jenkins \citeyearpar{SHJ61} is a good
example of the former type; they compare concepts that can be defined
using three variables  in terms of how many
examples (pairs of assignments and corresponding truth values of the
formula) it takes human subjects to understand and remember
a formula $\phi$, as defined by a subject's ability to predict the truth value
of $\phi$ correctly for a given truth assignment.
We can think of this work as measuring how hard it is to work with a
formula; our formalism is measuring how hard it is to learn the truth
value of a formula.  
The difficulty ranking found experimentally by Shepard et al. mostly
agrees with our ranking, except that they find two- and three-variable
XORs to be easier that some other formulae, whereas we have shown that
these are the hardest formulae.  This suggests that there might be
differences between how hard it is to work with a concept and how
hard it is to learn it.

Feldman \citeyearpar{Feldman} provides a good example of the latter approach.
He proposes the notion of the 
\emph{power spectrum} of a formula $\varphi$.
Roughly speaking, this counts the number of antecedents in the
conjuncts of a formula when it is written as a conjunction of
implications where the antecedent is a conjunction of literals and the
conclusion is a single literal.
For example, the formula $\phi = (v_1 \land (v_2 \lor v_3))\lor (\neg
v_1 \land (\neg v_2\land \neg v_3))$ can be written as
the conjunction of three such implications: $(v_2\rightarrow
v_1)\land (v_3\rightarrow v_1) \land (\neg v_2 \land v_1 \rightarrow v_3)$.
Since there are no conjuncts with 0 antecedents, 2 conjuncts with 1
antecedent, and 1 conjunct with 2 antecedents, the power spectrum of
$\phi$ is $(0,1,2)$.
Having more antecedents in an implication is viewed as making concepts
more complicated, so 
a formula with a power spectrum of
$(0,1,1)$ is considered  more complicated than one with a power spectrum of
$(0,3,0)$, and less complicated than one with a power spectrum of
$(0,0,3)$.  

A formula with a power spectrum of the form $(i,j,0, \ldots,0)$ (i.e., a
formula that can be written as the conjunction of literals and
formulae of the form $x \rightarrow y$, where $x$ and $y$ are literals)
is called a \emph{linear category}.  
Experimental evidence
suggests that human subjects generally find linear categories easier
to learn than nonlinear ones \citep{Feldman,LMG04}.
(This may be related to the fact that such formulae are linearly
separable, and hence learnable by support vector machines \citep{VapnikLerner}.)
Although our complexity measure does not completely agree with the
notion of a power spectrum, both notions classify 
XORs
and their negations as the most complex; these formulae can be shown
to have a power spectrum of the form 
$(0,\ldots,0,2^{n-1})$. 

Another notion of formula complexity is the notion of \emph{subjective
  structural complexity} 
introduced by Vigo \citeyearpar{Vigo2011}, where the subjective
structural complexity of a formula $\phi$ is
$|Sat(\phi)|e^{-\|\vec{f}\|_2}$,
where $Sat(\phi)$ is the set of truth assignments that satisfy
$\phi$, $f=(f_1,\ldots, f_n)$, $f_i$ is the fraction of truth
assignments that satisfy $\phi$ such that changing the truth value of $v_i$
results in a truth assignment that does not satisfy $\phi$,
and $\|\vec{f}\|_2 = \sqrt{(f_1)^2 + \cdots + (f_n)^2}$ represents the
$\ell^2$ norm.  
Unlike ours, with this notion of complexity, $\phi$ and $\neg \phi$
may have different complexity (because of the $|Sat(\phi)|$ factor).
However, as with our notion,  
XORs and their negation have maximal complexity.

In computer science and electrical engineering, \emph{binary decision
  diagrams} (BDDs)  
\cite{lee59} are used as a 
compact representation of Boolean functions.
BDDs resemble our notion of a testing strategy, although they do not
usually come with a notion 
of testing error or acceptable error margins on the output (guess).
Conversely, we could view testing strategies as a generalisation
of BDDs, in which we could ``accidentally'' take the wrong branch
(testing noise), a given variable can occur multiple times, leaf nodes
can also be labelled ``no guess'', and the notion of correctness of a BDD
for a formula is relaxed to require only that the output be correct
with a certain 
probability.
The \emph{expected decision depth} problem of Ron, Rosenfeld, and
Vadhan \cite{rosenfeld07} 
asks how many nodes of a BDD need to be visited in expectation
in order to evaluate a Boolean formula; this can also be
seen as a measure of complexity.
In our setting, an optimal strategy
for the ``noiseless'' information-acquisition game ($\alpha=1/2$,
$-\infty$ payoff for guessing wrong) exactly corresponds to a BDD for
the formula; asking about the depth of the BDD amounts to asking
about whether the strategy uses more than a given number of
tests.
\section{Conclusion}
\label{outlook}
We have presented the information-acquisition game, a game-theoretic model
of gathering information to inform a decision whose outcome depends
on the truth of a Boolean formula.
We argued that it is hard to find optimal strategies for this
model by brute force, and presented the random-test heuristic, a simple
strategy that  has only weak guarantees but is computationally tractable.
It is an open question whether better guarantees can be proven for the
random-test heuristic, and whether better approaches to testing
that are still more computationally efficient than brute force exist.
We used our techniques to show that RI is a
widespread phenomenon, at least, for formulae that use at most 9
variables. We argue that this certainly covers most concepts that
naturally arise in human discourse. Though it is certainly the case
that many propositions (e.g., the outcome of elections) depend on many
more variables, human speech and reasoning, for reasons of utterance
economy if nothing else, usually involves reducing these to simpler
compound propositions (such as the preferences of particular blocks of
voters). 
We hope in future work to get a natural structural criterion
for when formulae exhibit RI that can be applied
to arbitrary formulae.

Finally, we discussed how the existence of good strategies in our game
can be used as a measure of the complexity of a Boolean formula.
It would be useful to get a better understanding of whether test
complexity captures natural structural properties of concepts.  

Although we have viewed the information-acquisition game as a
single-agent game, there are natural extensions of it to multi-agent
games, where agents are collaborating to learn about a formula.
We could then examine different degrees of coordination for these
agents.  For example, they could share 
information at all times, or share information only
at the end (before making a guess).  The goal would be to understand
whether there is some structure in formulae that makes them
particularly amenable to division of labour, and to what extent
it can be related to phenomena such as rational inattention
(which may require the agents to coordinate on deciding which variable to
ignore).
In our model, we allowed agents to choose to make no guess for a
payoff of 0.  We could have removed this option, and instead 
required them to make a guess.    
We found this setting to be less amenable to analysis,
although there seem to be analogues to our results.
For instance, as in our introductory example, it is still rational to keep
testing the same variable in a disjunction with a probability that is
bounded away from zero, no matter how many tests are allowed. However, since
giving up is no longer an option, there is also a probability, bounded away
from both 0 and 1, that all variables have to be tested (namely when the formula
appears to be false, and hence it must be ascertained that all variables are).
The definition of test complexity makes sense in the alternative
setting as well, 
though the values it takes change; we conjecture that the theorem about XOR
being hardest can be adapted with few changes.

\commentout{
Possible future directions. Applications to medical diagnosis, DB query planning.

Different degrees of coordination for multiple agents collaborating to gather information (towards a guess whose payoff is shared by all), ordered by power:
\begin{enumerate}
\item can share information at all times;
\item can only share information in the beginning (before any tests) and end (before making a guess);
\item can only share information in the end (before making a guess). Also assume that everyone gets a random equivalent formula and variables have been renamed (so there is no canonical ordering).
\end{enumerate}

The second one can parallelise on $x\vee y$ (agree that everyone measures $x$ or everyone measures $y$ in advance) and act optimally, but will fail on something like $(x\wedge (a\oplus b) )\vee (\neg x\wedge  (c\oplus d))$, where the choice of $a$ or $b$ actually depends on the \emph{outcome} of measuring $x$. The third one can only parallelise something like $x\oplus y$ where multiple tests of each variable are necessary (splitting them evenly among all agents).
}

\iffullv
\appendix 

\section{Calculations for Example~\ref{twovaror}}
\label{app:twovaror}
In this section, we fill in the details of the calculations for
Example~\ref{twovaror}.  
We abuse notation by also viewing formulas, assignments, and
test-outcome sequences as events in (i.e., subsets of) the space of
histories of the game described in Section~\ref{sec:iagames}.
Specifically,
\begin{itemize}
\item we identify a truth asignment $A$ to the $n$ variables in the game
  with the event consisting of all histories where $A$ is the assignment
chosen by nature; 
\item we identify the formula $\varphi$ with the event consisting of
  all histories where $\phi$ is true under the assignment $A$ chosen
  by nature;  thus, $\phi$ is the disjoint union of all 
events $A$ such that $\varphi(A)=T$;
\item we identify a
test-outcome sequence $S=(v_{i_1} \mt b_1,\ldots,v_{i_k} \mt b_k)$ of length $k$
with the
event consisting of all
 histories where at least $k$ tests are performed, and the
  outcomes of the first $k$ are described by $S$.
\end{itemize}
\commentout{
    Finally, let $v_i=b$ (for $1\leq i\leq n$ and $b\in \{\t,\f\}$) denote the
    event that $A(v_i)=b$ for the assignment $A$ chosen by nature. 
    The event $v_i$ (for the formula) then is equal to the event $v_i=\t$ (for the variable),
    and more complex formulae can be decomposed in terms of their satisfying assignments,
    whose probabilities in turn factor into the probabilities of the individual variables taking
    the appropriate values as long as we are using a product distribution: for
    instance, 
$$\begin{array}{lll}
      & & \Pr(v_1 \vee v_2) \\
    &=& \Pr(v_1=\t)\Pr(v_2=\t) + \Pr(v_1=\t)\Pr(v_2=\f) + \Pr(v_1=\f)\Pr(v_2=\t) \\
    &=& 1-\Pr(\neg (v_1 \vee v_2)) = 1-\Pr(v_1=\f)\Pr(v_2=\f).
          \end{array}$$
    Since the assignments Nature can choose partition the space of events, the same decomposition
    works when conditioning on some event.
}

Observe that with the ``good'' payoff being $+1$ and the ``bad'' payoff
being $-16$,
the expected payoff from guessing that the formula 
is true after observing $S$ is $\Pr(\phi\mid S)\cdot 1 - \Pr(\neg\phi\mid S)\cdot 16$,
so it is greater than 0 if and only if $\Pr(\phi\mid S)>16/17$.

Henceforth, for brevity, let $A_{bb'}$ ($b,b' \in \{\t,\f\}$) refer to the assignment $\{v_1\mapsto b, v_2\mapsto b'\}$.
By assumption, all test outcomes are independent
conditional on a fixed assignment. 
Suppose first the player tests the same variable twice, say $v_1$. 
Then, for the ``ideal'' test outcome sequence $S=(v_1\mt \t, v_1\mt \t)$, 
the conditional probability of $S$ given that nature picked $A$ is
$(3/4)\cdot(3/4)$ if $A(v_1) = \t$, and $(1/4)\cdot
(1/4)$ otherwise. 
It follows that 
$$\begin{array}{lll}
   & & \Pr(v_1 \vee v_2 \mid S ) \\
 &=& \Pr(A_{\t\t} \mid S) + \Pr(A_{\t\f} \mid S) + \Pr(A_{\f\t} \mid S) \\
 &=& \frac{ \Pr( S \mid A_{\t\t} ) \Pr(A_{\t\t}) + \ldots + \Pr( S\mid A_{\f\t} )\Pr(A_{\f\t}) }{ \Pr( S ) } \\
 &=& \frac{ \Pr( S \mid A_{\t\t} ) \Pr(A_{\t\t}) + \ldots + \Pr( S\mid A_{\f\t} )\Pr(A_{\f\t}) }{\sum_{A} \Pr(S \mid A) \Pr(A)  } \\ 
 &=& \frac{ ( (3/4)\cdot (3/4) + (3/4)\cdot (3/4) + (1/4)\cdot (1/4) ) \cdot (1/4) }{ ( (3/4)\cdot (3/4) + (3/4)\cdot (3/4) + (1/4)\cdot (1/4) + (1/4)\cdot (1/4)) \cdot (1/4) } \\
 &=& \frac{ ( 19/16 )\cdot (1/4) }{ (20/16) \cdot (1/4) } \\
  &=& 19/20 > 16/17.
 \end{array}$$
Thus, the agent will guess true after observing $S$, and get a
positive expected payoff (since $S$ will be observed with positive
probability) as a consequence of testing $v_1$ twice.
Symmetrically, testing $v_2$ twice gives a positive expected payoff.

On the other hand, suppose the player tests two different
variables. The best case would 
 be to get $S=(v_t\mt \t, v_2\mt
\t)$. As before, 
the probability of $S$ conditioned on some assignment is the product
of the probabilities for each of its entries being observed; 
for instance, $\Pr(S\mid A_{\t\f}) = (3/4)\cdot (1/4)$. So we get
$$\begin{array}{lll}
   & & \Pr(v_1 \vee v_2 \mid S ) \\
 &=& \Pr(A_{\t\t} \mid S) + \Pr(A_{\t\f} \mid S) + \Pr(A_{\f\t} \mid S) \\
 &=& \frac{ \Pr( S \mid A_{\t\t} ) \Pr(A_{\t\t}) + \ldots + \Pr( S\mid A_{\f\t} )\Pr(A_{\f\t}) }{ \Pr( S ) } \\
 &=& \frac{ \Pr( S \mid A_{\t\t} ) \Pr(A_{\t\t}) + \ldots + \Pr( S\mid A_{\f\t} )\Pr(A_{\f\t}) }{\sum_{A} \Pr(S \mid A) \Pr(A)  } \\ 
 &=& \frac{ ( (3/4)\cdot (3/4) + (3/4)\cdot (1/4) + (1/4)\cdot (3/4) ) \cdot (1/4) }{ ( (3/4)\cdot (3/4) + (3/4)\cdot (1/4) + (1/4)\cdot (3/4) + (1/4)\cdot (1/4)) \cdot (1/4) } \\
 &=& \frac{ ( 15/16 )\cdot (1/4) }{ (16/16) \cdot (1/4) } \\
  &=& 15/16 < 16/17.
  \end{array}$$
An analogous calculation shows that if either of the tests comes out false, the conditional
probability is even lower.
Thus, after testing different variables, the agent will not make a
guess, no matter what the outcome, and so has an expected payoff of 0.

So, indeed, measuring the same variable twice is strictly better than
measuring each of them once.

\section{Quantifying rational inattention}
\label{sec:inattentionlp}

Our goal is to show that a large proportion of Boolean formulae exhibit RI.
To this end, we would like a method to establish that a particular formula
exhibits RI that is sufficiently efficient that we can run it on all formulae
of a given size, or at least a statistically significant sample.
Throughout this section, we focus on some arbitrary but fixed
formula $\phi$ in $n$ variables $v_1$, $\ldots$, $v_n$.
Proposition \ref{RIviatestseqs} gives a sufficient criterion for $\phi$
to exhibit RI
in terms of the structure of the optimal sequences of test outcomes of
each length. 
To make use of this criterion, we introduce some machinery to
reason about optimal sequences of test outcomes. 
The key definition turns out to be that of the
\emph{characteristic fraction} of $S$
for $\phi$,
denoted $\cf(\phi,s)$,
which is a quantity that is inversely ordered to
$\Pr(\phi\mid S)$ (Lemma~\ref{lem:inverse})  
(so the probability is maximised iff the characteristic fraction is
minimised and vice versa), 

while exhibiting several convenient properties that enable the
subsequent analysis. 
Let $o_i$ represent the odds of making a correct observation of
$v_i$, namely, the probability of observing $v_i \mt b$ conditional on
$v_i$ actually being $b$ divided by the probability of observing
$v_i \mt b$ conditional on $v_i$ not being $b$.  If we assume that $o_i
= o_j$ for all variables $i$ and $j$, and let $o$ represent this
expression, then 
$\cf(\phi,S)$ 
is the quotient of two polynomials, and has the the form
$$ \frac{ c_1 o^{ d_1 |S| } + \ldots + c_{2^n} o^{ d_{2^n} |S| } }{
    e_1 o^{ f_1 |S| } + \ldots + e_{2^n} o^{ f_{2^n} |S| } } , $$ 
where $c_j$, $d_j$, $e_j$, and $f_j$ are terms that depend on the
truth assignment $A_j$, so we have one term for each of the $2^n$
truth assignments, and $0 \le d_j, f_j \le 1$.  
\commentout{
Our analysis is based on the (unproven and not even rigorously stated, but morally important) circumstance that as $|S|$ gets large,
only leading terms (that is, those summands where the multiplier $d_i$ before $|S|$ is greatest) ``matter'' in this fraction. 
If $S$ is optimal, then the leading term in the denominator is always
$c o^{|S|}$ (\ref{lem:noncontra}). 
}
For a test-outcome sequence $S$ that is optimal for $\phi$, we can show that
$f_j = 1$ for some $j$.  Thus, the most significant term in the
denominator (i.e., the one that is largest, for $|S|$ sufficiently
large) has the form $co^{|S|}$.  We call the factor $d_i$ before $|S|$ 
in the exponent of the leading term of the numerator the
\emph{max-power} (Definition~\ref{def:maxpower}) of the characteristic
function. 
We can show that the max-power is actually independent of $S$ (if $S$
is optimal for $\phi$).  Since we are interested in the test-outcome
sequence $S$ for which $\cf(\phi,S)$ is minimal (which is the
test-outcome sequence for which $\Pr(\phi|S)$ is maximal), for each
$k$, we want to find that $S$ of length $k$ whose max-power is minimal.
As we show, we can find the sequence $S$ whose max-power is minimal by
solving 
a linear program (Definition~\ref{def:conflictlp}). 

\commentout{
Not only does each sequence $S$ translate
to an $A$-trace, but we can conversely also find (\ref{lem:approx}), for any length $k$ and any vector $v$ that satisfies some sanity checks,
a sequence $S_k(v)$ whose $A$-trace is approximately $v$ (arbitrarily close to, as $k$ gets large).
Using this, we can formally state the necessary variant of the earlier statement that only leading terms matter when $|S|$ is large:
if the $A$-trace of test-outcome sequence $S$ is too far removed from
any solution point of the LP,
then so must be its max-power as a consequence of LP continuity
(\ref{lem:objsep}), and hence (\ref{lem:seqtolp}) $|S|$ is either
small
or $S$ 
is not optimal, as approximations $S_k(v)$ to any solution point $v$
have a higher-order leading term and so in fact can be shown 
to give higher conditional probability of $\phi$'s truth or falsity.

So as optimal test sequences $S$ get longer, the set of their
$A$-traces converges to the set of solutions to the LP described above
(in the sense 
that there is a negligible function $\delta$ bounding their distance
from this set depending on the length $|S|$). 
It turns out that this lifts (\ref{thm:lptorieasy}) the criterion of
Proposition \ref{RIviatestseqs} to a condition on the LP solution set:
when a fixed entry
 is zero in all LP solution points, then $\delta$ bounds above the
 number of times the variable corresponding to it can be measured 
in an optimal test-outcome sequence (and so the variable is necessarily ignored by any such sequence), and when another entry is $C>0$, the
corresponding variable may not be ignored. Therefore, the existence of such entries in all solution points to the LP is sufficient to
conclude that all optimal test-outcome sequences for $\phi$ satisfy
the precondition of \ref{RIviatestseqs}, and hence $\phi$ exhibits
rational inattention. 
}
\subsection{Preliminaries}

\commentout{
A naive calculation of the probability that a formula is true given
a sequence of test outcomes, as we showed in Appendix
\ref{app:twovaror},  
involves conditioning on every single test outcome in order, 
calculating an updated conditional probability that the tested
variable is true each time. In this section, we present a
streamlined version of this process based on the notion of \emph{odds},
defined for an event $A$ as the quotient $\Pr(A)/\Pr(\neg A)$.
The odds view of testing captures several properties of tests,
such as their conditional independence given a fixed truth assignment,
in a natural way. If the statistics of a sequence
of outcomes is already known, it enables us to calculate
conditional probabilities in time $O(\log k)$ by standard techniques
of fast exponentiation.
Most importantly, it enables us to reason about the limiting case
as $k\rightarrow\infty$ analytically, forming a key component
of our estimation of the commonness of RI.
}

In this subsection, we present some preliminary results that will
prove useful in quantifying RI.  We start with a lemma that gives a
straightforward way of calculating 
$\Pr(A \mid S)$
for
an assignment $A$ and a test-outcome sequence $S$. 
The lemma also shows that, as the notation suggests, the probability is
independent of the strategy $\sigma$.  
In the proof of the lemma, we use the following abbreviations:
\def\o#1#2{r_{D,\m}({#1},{#2})}
\begin{itemize}
  \item $o_i = \frac{1/2 + \alpha_i}{1/2 - \alpha_i}$. 
We  can think of $o_i$ as the odds of making a correct observation of
$v_i$; namely, the probability of observing $v_i \mt b$ conditional on
$v_i$ actually being $b$ divided by the probability of observing
$v_i \mt b$ condition on $v_i$ not being $b$.
\item $\nSAp = |\{j: S[j] = (v_i \mt A(v_i))\}|$.  Thus, $\nSAp$ is the
number of times that $v_i$ is observed to have the correct value
  according to truth assignment $A$ in test-outcome sequence $S$.
  \item $\o{A}{S} = \Pr(A) \prod_{\{i: v_i \textrm{ is in the domain of $A$}\}} o_i^{\nSAp}$
\end{itemize}
\commentout{
By Bayes' rule, for all truth assignments $A$ and sequences
$S=[v_{i_1}\mt b_1, \ldots, v_{i_k}\mt b_k]$ of test outcomes, we have  
\begin{equation} 
\begin{array}{*3{>{\displaystyle}l}}
  \PrS{\sigma}(A \mid S) &=& \frac{ \PrS{\sigma}(S \mid A)\Pr(A) }{ \PrS{\sigma}(S) } \\
&= &\frac{ \PrS{\sigma}(S \mid A)\Pr(A) }{ \sum_{\text{truth assignments }A'}
  \PrS{\sigma}(S\mid A')\Pr(A') }. \label{eqn:prbayes}
\end{array}
    \end{equation}
For any variable $v_i$, let $\tau_j(v_i)$ denote the event that the player tested the variable $v_i$ at the $j$th timestep. Also, for $b\in\{T,F\}$, let $\rho_j(b)$ denote the event that Nature responded with the respective result to the player query at the $j$th timestep.
If $S = [v_{i_1}\mt b_1, \ldots, v_{i_k}\mt b_k]$,
for $1 \le j \le k$, define $S[j] = (v_{i_j} \mt b_j)$, the $j$th
component of the sequence $S$. 

For a formal pair $(v_{i_j}\mt b_j)$ of a variable and truth value, we
will let $\tau_j(v_{i}\mt b)=\tau_j(v_i)$ and $\rho_j(v_{i}\mt
b)=\rho_j(b)$. This is so we can compactly write $\tau_j(S[j])$ for
the event that the player's test at time $j$ was the same as the one
performed in the test-outcome sequence, and $\rho_j(S[j])$ for the
event that the result was the same. Thus, the event denoted by the
sequence $S$ becomes shorthand for the intersection of events
$\tau_1(S[1])\wedge \rho_1(S[1]) \wedge \ldots \wedge
\tau_{|S|}(S[|S|]) \wedge \rho_{|S|}(S[|S|])$.  

Observe that because Nature's response to a test only depends on the
variable tested and the assignment chosen, the conditional probability
$\PrS{\sigma}(\rho_j(b)\mid A,\tau_j(v_i))$ takes the same value

(namely $1/2+\alpha_i$ whenever $A(v_i)=b$ and $1/2-\alpha_i$
otherwise) whenever it is defined for any given $v_i$, $A$ and $b$,
regardless of $j$ or $\sigma$. 
We will therefore also write this probability as $\Pr( v_i\mt b \mid
A)$, with $v_i\mt b$ understood to denote a sort of ``conditional
event'' $\rho_j(b)\mid \tau_j(v_i)$. 

Equation (\ref{eqn:prbayes}) leaves open the possibility that the
conditional probability that Nature chose some assignment given a
record of tests and outcomes -- and hence what the player learns about
the truth of $\phi$ -- depends on the player strategy
$\sigma$. Intuitively, this is nonsense: since the strategy $\sigma$
can only determine the next player test in terms of the tests and

results observed so far, it should not be able to produce any
information about $\phi$ that the player doesn't know anyway. The
following Lemma makes this intuition precise. 

\begin{lemma} 
The probability that nature chose an assignment $A$ conditional on a
test-outcome sequence $S$ was observed is
\begin{eqnarray} \Pr(A\mid S) &=& 
 \frac{ \prod_{i=1}^{|S|}\Pr(S[i]\mid A)\Pr(A) }{ \sum_{\text{truth assignments }A'}
  \prod_{i=1}^{|S|}\Pr(S[i]\mid A')\Pr(A') }. \label{eqn:prbayes2}
\end{eqnarray}
In particular, it does not depend on the strategy $\sigma$.
\end{lemma}
\begin{proof}
\def\Prs{\PrS{\sigma}}
We will show this statement by induction on the length of test-outcome sequences $S$. The base case, with $|S|=0$, is clearly true.
For the inductive step, let $|S|=k$ and suppose that
$$ \Pr(A\mid S') =  \frac{ \prod_{i=1}^{|S'|}\Pr(S'[i]\mid A)\Pr(A) }{ \sum_{\text{truth assignments }A'}
   \prod_{i=1}^{|S'|}\Pr(S'[i]\mid A')\Pr(A') }$$
for all $S'$ with $|S'|<k$. 

Now,
\begin{eqnarray}
 & & \Prs(A \mid S) \nonumber\\
 &=& \Prs(A \mid \tau_1(S[1]) \wedge \rho_1(S[1]) \wedge \ldots \wedge \tau_k(S[k])\wedge \rho_k(S[k]) ) \nonumber\\
 &\overset{\text{Bayes}}{=}& \frac{ \Prs(\tau_k(S[k])\wedge \rho_k(S[k]) \mid A, \tau_1(S[1]) \wedge \ldots \wedge \rho_{k-1}(S[k-1]) ) \Prs(A \mid \tau_1(S[1]) \wedge \ldots \wedge \rho_{k-1}(S[k-1])) }{ \Prs(\tau_k(S[k])\wedge \rho_k(S[k]) \mid \tau_1(S[1]) \wedge \ldots \wedge \rho_{k-1}(S[k-1])) } \nonumber\\
 &\overset{\text{IH}}{=}& \frac{ \Prs(\tau_k(S[k])\wedge \rho_k(S[k]) \mid A, S_{1\ldots k-1} ) }{ \Prs(\tau_k(S[k])\wedge \rho_k(S[k]) \mid S_{1\ldots k-1}) } \cdot \frac{ \prod_{i=1}^{k-1}\Pr(S[i]\mid A)\Pr(A) }{ \sum_{A'} \prod_{i=1}^{k-1}\Pr(S[i]\mid A')\Pr(A') }.  \label{eqn:onestep}
\end{eqnarray}

The numerator of the left-hand fraction satisfies
\begin{eqnarray}
 & & \Prs(\tau_k(S[k])\wedge \rho_k(S[k]) \mid A, S_{1\ldots k-1} ) \nonumber\\
 &=& \Prs(\tau_k(S[k])\wedge \rho_k(S[k]) \mid A, S_{1\ldots k-1} ) / \Prs(\tau_k(S[k]) \mid A, S_{1\ldots k-1} ) \cdot \Prs(\tau_k(S[k]) \mid A, S_{1\ldots k-1} ) \nonumber\\
 &=& \Prs( \rho_k(S[k]) \mid A, S_{1\ldots k-1}, \tau_k(S[k]) ) \cdot \Prs(\tau_k(S[k]) \mid A, S_{1\ldots k-1} ) \nonumber\\
  &=& \Pr( S[k] \mid A ) \cdot \Prs(\tau_k(S[k]) \mid A, S_{1\ldots k-1} ) \nonumber\\
 & & \text{(as $S[k]$ only depends on $A$)} \nonumber\\
 &=& \Pr( S[k] \mid A ) \cdot \Prs(\tau_k(S[k]) \mid S_{1\ldots k-1} ) \label{eqn:expnumerator} \\
 & & \text{(as strategies dictate the next test solely based on past test outcomes).} \nonumber
\end{eqnarray}
Also, the denominator satisfies
$$\begin{array}{lll}
 & &   \Prs(\tau_k(S[k])\wedge \rho_k(S[k]) \mid S_{1\ldots k-1}) \\
 &=& \sum_{A'} \Prs(\tau_k(S[k])\wedge \rho_k(S[k]) \mid A', S_{1\ldots k-1}) \Pr(A' \mid S_{1\ldots k-1}) \\
 &\overset{\text{IH}}{=}& \sum_{A'} \Prs(\tau_k(S[k])\wedge \rho_k(S[k]) \mid A', S_{1\ldots k-1}) 
\cdot \frac{ \prod_{i=1}^{k-1}\Pr(S[i]\mid A')\Pr(A') }{ \sum_{A''} \prod_{i=1}^{k-1}\Pr(S[i]\mid A'')\Pr(A'') } \\
 &\overset{\text{(\ref{eqn:expnumerator})}}{=}& \sum_{A'} \Pr( S[k] \mid A' ) \cdot \Prs(\tau_k(S[k]) \mid S_{1\ldots k-1} )  \cdot \frac{ \prod_{i=1}^{k-1}\Pr(S[i]\mid A')\Pr(A') }{ \sum_{A''} \prod_{i=1}^{k-1}\Pr(S[i]\mid A'')\Pr(A'') }.
\end{array}$$
Plugging (\ref{eqn:expnumerator}) and this equality into (\ref{eqn:onestep}), we note that $\sum_{A''} \prod_{i=1}^{k-1}\Pr(S[i]\mid A'')\Pr(A'')$ and $\Prs(\tau_k(S[k]) \mid S_{1\ldots k-1} )$ cancel, removing all dependencies on $\sigma$ and leaving
$$ \Prs(A \mid S) = \frac{ \Pr( S[k] \mid A ) }{ \sum_{A'} \Pr( S[k] \mid A' ) \prod_{i=1}^{k-1}\Pr(S[i]\mid A')\Pr(A') } \cdot \prod_{i=1}^{k-1}\Pr(S[i]\mid A)\Pr(A) $$
as required.
\end{proof}
}
\commentout{
    $S$ denotes the event that the test-outcome sequence $S$ was observed,
    that is, that the player acting according to
    $\sigma$ chose to test variables in the order $v_{i_1}$, $\ldots$,
    $v_{i_k}$, and  
    nature responded with the truth values $b_1$, $\ldots$, $b_k$ in order.
    This is an intersection of two events, and so we can expand
    $$ \PrS{\sigma}(S\mid A) = \Pr(\text{outcomes as in $S$}\mid A, \text{vars tested as in $S$}) \PrS{\sigma}(\text{vars tested as in $S$}\mid A), $$
    where the first probability on the right-hand side does not depend on $\sigma$ as nature chooses outcomes independently
    conditioned on the selected truth assignment $A$.

    If we now cancel $\PrS{\sigma}(\text{vars tested as in $S$}\mid A)$ in
    the fraction in Equation (\ref{eqn:prbayes}),
    the remaining expression no longer depends on $\sigma$ at all, thereby justifying the notation $\Pr(A\mid S)$.
    \begin{eqnarray} \Pr(A\mid S) &=& 
     \frac{ \Pr(\text{outcomes as in $S$}\mid A, \text{vars tested as in $S$})\Pr(A) }{ \sum_{\text{truth assignments }A'}
      \Pr(\text{outcomes as in $S$}\mid A', \text{vars tested as in $S$})\Pr(A') }. \label{eqn:prbayes2}
    \end{eqnarray}
}

\commentout{
In fact, we can note that the factors due to each test outcome $v_i\mt b_i$
only can take one of two different forms, depending on whether the assignment
that is being conditioned on has $v_i\mapsto b_i$ or $v_i\mapsto \neg b_i$.
This allows us to simplify the expression further using the following
two definitions.
}

\def\o#1#2{r_{D,\m}({#1},{#2})}

\begin{lemma}\label{lem:denormtoprob} For all 
accuracy vectors $\vec{\alpha}$, product distributions $D$, 
assignments $A$, and
  test-outcome sequences $S$, 
$$ \Pr(A \mid S) =  \frac{ \o{A}{S} }{ \sum_{\text{truth assignments
      }A'}  \o{A'}{S} }. $$
  Thus,
$$ \Pr(\phi\mid S) = \sum_{\{A:\; \phi(A)=\t\}} \Pr(A \mid S) = \frac{
    \sum_{\{A:\; \phi(A)=\t\}} \o{A}{S} }{ \sum_{A'} \o{A'}{S} }. $$
These probabilities do not depend on the strategy $\sigma$.
\end{lemma}

\begin{proof}
  By Bayes' rule, for all truth assignments $A$ and sequences
$S=[v_{i_1}\mt b_1, \ldots, v_{i_k}\mt b_k]$ of test outcomes, we have  
\begin{equation}
\begin{array}{*3{>{\displaystyle}l}}
  \PrS{\sigma}(A \mid S) &=& \frac{ \PrS{\sigma}(S \mid A)\Pr(A) }{ \PrS{\sigma}(S) } \\
&= &\frac{ \PrS{\sigma}(S \mid A)\Pr(A) }{ \sum_{\text{truth assignments }A'}
  \PrS{\sigma}(S\mid A')\Pr(A') }. \label{eqn:prbayes}
\end{array}
    \end{equation}
Suppose that $S = (v_{i_1}\mt b_1, \ldots, v_{i_k}\mt b_k)$.
We want to compute $\PrS{\sigma}(S \mid A')$ for an arbitrary truth
assignment $A'$.   Recall that a strategy $\sigma$ is a function from
test-outcome sequences to a distribution
over actions. 
We write
$\sigma_S(\textrm{test } v)$ to denote the
probability that $\sigma$ tests $v$ given test-outcome sequence $S$ and
and use $(\,)$ for the empty sequence;
more generally, we denote
\def\test#1#2{\textrm{test}_{#1}(#2)}
by $\test{j}{v}$ the event that the $j$th variable chosen
was $v$. Then,
$$\begin{array}{lll}
    \PrS{\sigma}(S \mid A') = &\sigma_{(\,)}(\test1{v_{i_1}})
\PrS{\sigma}((v_{i_1} \mt b_1) \mid \test1{v_{i_1}}, A') \ldots\\
&\sigma_{ (v_{i_1} \mt b_{1}, \ldots, v_{i_{k-1}} \mt   b_{k-1})}(\test{k}{v_{i_k}})
\PrS{\sigma} ((v_{i_k} \mt b_{k}) \mid
  \test{k}{v_{i_k}}, A' ).
  \end{array}$$ 
Here, we were able to write
$\PrS{\sigma} ( (v_{i_j} \mt b_j) \mid \test{j}{v_{i_j}}, A' )$
without conditioning on the entire test-outcome sequence up to $v_{i_{j-1}}$ because 
by the definition of the information-acquisition game,
all observations are independent of each other conditioned on the
assignment $A'$. 
Observe that the terms $\sigma_{(\,)}(\test1{v_{i_1}})$, $\ldots$,
$\sigma_{ (v_{i_1} \mt b_{1}, \ldots, v_{i_{k-1}} \mt   b_{k-1})}(\test{k}{v_{i_k}})$
are common to \mbox{$\PrS{\sigma}(S \mid A')$} for all truth assignments $A'$,
so we can pull them out of the numerator and denominator in (\ref{eqn:prbayes}) and
cancel them.
Moreover, probabilities of the form
$\PrS{\sigma} ((v_{i_j} \mt b_{j}) \mid \test{j}{v_{i_j}}, A' )$ do not
depend on the strategy $\sigma$, so we can drop it from the subscript of
$\PrS{\sigma}$; the probability also does not depend on the results of
earlier tests (since, by assumption, test outcomes are independent,
conditional on the truth assignment).
Thus,
  it follows that
  $$  \PrS{\sigma}(A \mid S) =
   \frac{ \left[\prod_{j=1}^{k}\Pr(v_{i_j} \mt b_j \textrm{ observed} \mid
     v_{i_j} \textrm{ chosen}, A) \right]\Pr(A) }{ \sum_{\text{truth assignments }A'}
   \left[ \prod_{j=1}^{k}\Pr(v_{i_j} \mt b_j \textrm{ observed} \mid
     v_{i_j} \textrm{ chosen}, A') \right] \Pr(A') }.$$

Next, we multiply both the numerator and the denominator of this fraction by 
$ \prod_{j=1}^k \frac{1}{1/2-\alpha_{i_j}}. $
This amounts to multiplying the $j$th term in each product by
$\frac{1}{1/2-\alpha_{i_j}}$.
Thus, in the numerator, 
if $b_j = A(v_{i_j})$, then the $j$th term in the product
equals $o_{i_j}$; if $b_j =  \neg A(v_{i_j})$, then the $j$th
term in the product is 1.  It easily follows that this expression is
just $\o{A}{S}$.  A similar argument shows that 
the
 denominator is
$\sum_{\text{truth assignments }A'}  \o{A'}{S}$.  This proves the
first and third statements in the lemma.  The second statement is immediate from
the first.
\end{proof}

\commentout{
Then note that all those factors in each product where $A(v_{i_j})=b_j$ are now $$\frac{1/2 + \m_{i_j}}{1/2 - \m_{i_j}} = o_{i_j},$$
and all those factors where $A(v_{i_j})=\neg b_j$ are $$\frac{1/2 - \m_{i_j}}{1/2 -\m_{i_j}} = 1.$$

For every given variable $v_k$, the product of those factors that are
due to tests of that variable (i.e., indices $j$ where $i_j=k$) 
is therefore just $o_{i_j}$ to the power of however many times the tests of that variable agreed with the assignment.
Recoalescing the product to be over the $n$ variables in the formula rather than over the $k$ positions in the sequence of test outcomes $S$, we obtain the statement of the lemma.
\end{proof}
}

The next lemma gives 
an 
intuitive
 property of those
test-outcome sequences $S$ that are \emph{optimal} for $\phi$, $D$, and $\m$.
\begin{lemma} \label{lem:noncontra} If $S$ is
  a test-outcome sequence that is optimal for $\phi$, $D$, and $\m$,
and $\Pr(\phi|S) \ne \Pr(\phi)>0$, then $S$ does not contain
observations both of the form $v_i \mt T$ and of the form $v_i \mt F$
for any $v_i$. 
\end{lemma}
\begin{proof}
  Suppose that $S$ is optimal for $\phi$, $D$, and $\m$, 
    $\Pr(\phi \mid S) \ne \Pr(\phi)$, 
  there are $n_1>0$ instance of $v_i\mt T$ in $S$, 
and $n_2>0$ instances of $v_i\mt F$ in $S$. Without loss of generality,
suppose that  $n_1>n_2$.   
Let $S_0$ be the sequence that results from $S$ by removing the $n_2$
occurrences of $v_i \mt F$ and the last $n_2$ occurrences of $v_i \mt
T$.  Thus, $|S_0| = |S| - 2n_2 < |S|$.
It is easy to see that, for each truth assignment $A$, we have
$\nSAp = n_{S_0,A,i}^+ + n_2$.  It thus follows 
from Lemma~\ref{lem:denormtoprob} that $\Pr(\phi\mid S) = \Pr(\phi \mid S_0)$.
We can  similarly remove all other ``contradictory'' observations to
get a sequence $S_0$ that does not contradict itself such that $|S_0|
< |S|$ and $\Pr(\phi\mid S) = \Pr(\phi \mid S_0)$.

Suppose without loss of generality that $\Pr(\phi) - 1/2 \ge 0$.
Since it cannot be the case that for every test-outcome sequence $S_0$
of length $|S|$ we have $\Pr(\phi \mid S_0) - 1/2 < \Pr(\phi) -
1/2$, and $S$ is optimal for $\phi$, $D$, and $\m$, we must have 
\begin{equation}\label{eq:0}
  \Pr(\phi \mid S) - 1/2 \ge |\Pr(\phi) - 1/2|.
  \end{equation}
We want to show that we can add tests to $S_0$ to get a sequence $S^*$
with $|S^*| = |S|$ such that $\Pr(\phi\mid S^*) > \Pr(\phi\mid S_0) =
\Pr(\phi \mid S)$.  This will show that $S$ is not optimal for $\phi$,
$D$, and $\m$, giving us the desired contradiction.

Suppose that $S_0 = (v_{i_1} \mt b_{1}, \ldots, v_{k} \mt
b_{k})$. 
Define test-outcome sequences $S_1, \ldots, S_{k}$ inductively by taking 
$S_j$ to be $S_{j-1}$ with $v_{i_j} \mt b_{j}$ removed if
$\Pr(\phi \mid S_{j-1}) \le \Pr(\phi \mid S_{j-1}\setminus (v_{i_j}
  \mt b_{j}))$ and otherwise taking $S_j = S_{j-1}$.  It is immediate
  from the construction that $\Pr(\phi \mid S_k) \ge \Pr(\phi \mid
  S_0) = \Pr(\phi \mid S)$ and $|S_k| \le |S_0| < |S|$.  It cannot be
  the case that $|S_k| = 0$, for then $\Pr(\phi) \ge \Pr(\phi \mid
  S)$.  Since $\Pr(\phi) \ne \Pr(\phi \mid S)$ by assumption, we would
 have $\Pr(\phi) > \Pr(\phi \mid S)$, contradicting (\ref{eq:0}).

  Suppose that $v_i \mt b$ is the last test in $S_k$.
Let $S_k^- = S_k \setminus (v_i \mt b)$, so that $S_k = S_k^- \cdot
(v_i \mt b)$.  
  By construction,
$\Pr(\phi \mid S_k) > \Pr(\phi \mid S_k^-)$.  That
  is, observing $v \mt b$ increased the conditional probability of
  $\phi$.  We now show that observing $v \mt b$ more often 
  increases the conditional probability of $\phi$ further; that is, for all $m$,
  $\Pr(\phi  \mid (S_k \cdot (v_i \mt b)^m) > \Pr(\phi \mid S_k)$.  We can
  thus take $S^* = (S_k \cdot (v_i \mt b)^{|S| - |S_k|})$.  

  It follows from Lemma~\ref{lem:denormtoprob} that
  $$\begin{array}{*2{>{\displaystyle}l}}
  \Pr(\phi \mid S_k) = \sum_{\{A:\; \phi(A) = T\}} \Pr(A \mid S_k) =
  \frac{\sum_{\{A:\; \phi(A) = T\}} \o{A}{S_k}}{ \sum_{\text{truth assignments
      }A'}  \o{A'}{S_k} }\\
 \mbox{and }    \Pr(\phi \mid S_k^-) = \sum_{\{A:\; \phi(A) = T\}} \Pr(A \mid S_k^-)
  \frac{\sum_{\{A:\; \phi(A) = T\}} \o{A}{S_k^-}}{ \sum_{\text{truth assignments
      }A'}  \o{A'}{S_k^-} }.
  \end{array}$$
Note that for any truth assignment $A'$, $\o{A'}{S_k^-} = \o{A'}{S_k}$
if $A(v_i) \ne b$, and   %
 $\o{A'}{S_k^-} = o_i \o{A'}{S_k}$ if
$A(v_i) = b$.  Thus, there exist $x_1, x_2, y_1, y_2$ such that
$\Pr(\phi\mid S_k^-) = \frac{x_1 + x_2}{y_1 + y_2}$ and
  $\Pr(\phi\mid S_k) = \frac{o_ix_1 + x_2}{o_iy_1 + y_2}$.
Indeed, we can take 
\begin{align*} 
x_1 &= \sum_{\mathclap{\{A: \phi(A) = T, A(v_i) = b\}}} \, \o{S_k}{A}, 
&x_2 &= \sum_{\mathclap{\{A: \phi(A) = T, A(v_i) \ne b\}}} \, \o{S_k}{A}, \\
y_1 &= \sum_{\mathclap{\{A: A(v_i) = b\}}} \, \o{S_k}{A},\text{ and} 
&y_2 &= \sum_{\mathclap{\{A: A(v_i) \ne b\}}} \, \o{S_k}{A}.
\end{align*}

    Since $\Pr(\phi \mid S_k) > \Pr(\phi \mid S_k^-)$, we must have
\begin{equation}\label{eq:1}
  \frac{o_ix_1 + x_2}{o_iy_1 + y_2} > \frac{x_1 + x_2}{y_1 + y_2}.
  \end{equation}
    Since $x_1, x_2, y_1, y_2 \ge 0$, crossmultiplying shows that 
   (\ref{eq:1}) holds iff 
$$x_2 y_1 + o_i x_1 y_2 > x_1 y_2 + o_ix_2 y_1.$$
Similar manipulations show that
$$\begin{array}{*2{>{\displaystyle}l}}
  &\Pr(\phi \mid S_k\cdot (v_i \mt b) > \Pr(\phi \mid S_k)\\
    \mbox{iff } &\frac{o_i^2x_1 + x_2}{o_i^2y_1 + y_2} > \frac{o_ix_1 +
    x_2}{o_iy_1 + y_2}\\ 
\mbox{iff } &x_2 y_1 + o_i x_1 y_2 > x_1 y_2 + o_ix_2 y_1.
\end{array}
$$
Thus, $\Pr(\phi \mid S_k\cdot (v_i \mt b)) > \Pr(\phi \mid S_k)$.  A
straightforward induction shows that
$\Pr(\phi \mid S_k\cdot (v_i \mt b)^h) > \Pr(\phi \mid S_k)$ for all
$h$, so $\Pr(\phi \mid S^*) > \Pr(\phi \mid S_k) = \Pr(\phi \mid S)$,
as desired.
\end{proof}

\subsection{Characteristic fractions and the limit of traces}

\def\cf{\mathrm{cf}}
\begin{definition} \label{def:cf} 
The \emph{characteristic fraction} of a test-outcome sequence $S$ 
for $\phi$ is
$$\cf(\phi,S) = \frac{ \sum_{\{A:\; \varphi(A)=\f\}} \o{A}{S}  }{
  \sum_{\{A:\; \varphi(A)=\t\}} \o{A}{S} }. $$ 
\hfill \wbox
\end{definition}

The importance of this quantity is due to the following:

\begin{lemma}\label{lem:inverse}
$\Pr(\phi\mid S)> \Pr(\phi\mid S')\text{ iff }\mathrm{cf}(\phi,S) <
\mathrm{cf}(\phi,S').$
\end{lemma}

\begin{proof}
Since for $x,y> 0$, we have that $ x<y$ iff $(1/x)>(1/y)$, so
it follows from Lemma~\ref{lem:denormtoprob} that
$ \Pr(\phi \mid M)<\Pr(\phi \mid M') $
iff
$$ \frac{ \sum_{A} \o{A}{S} }{ \sum_{\{A:\; \phi(A)=\t\}} \o{A}{S} } >
\frac{ \sum_{A} \o{A}{S'} ) }{ \sum_{\{A:\; \phi(A)=\t\}} \o{A}{S'}\} }, $$ 
which is true iff
$$\begin{array}{lll}
 & & \frac{ \sum_{\{A:\; \phi(A)=\t\}} \o{A}{S} + \sum_{ \{A:\;
      \phi(A)=\f\} } \o{A}{S}  }{ \sum_{ \{A:\; \phi(A)=\t\} } \o{A}{S}
  } \\ 
  &>&
   \frac{ \sum_{\{A:\; \phi(A)=\t\}} \o{A}{S'} + \sum_{ \{A:\; \phi(A)=\f\} } \o{A}{S'}  }{ \sum_{ \{A:\; \phi(A)=\t\} } \o{A}{S'} }, 
   \end{array}$$
that is, if and only if
$$ \frac{ \sum_{ \{A:\; \phi(A)=\f\} } \o{A}{S}  }{ \sum_{ \{A:\; \phi(A)=\t\} } \o{A}{S} } > \frac{ \sum_{ \{A:\; \phi(A)=\f\} } \o{A}{S'}  }{ \sum_{ \{A:\; \phi(A)=\t\} } \o{A}{S'} }. $$
The statement of the lemma follows.
\end{proof}

\begin{example} Let $\phi=(v_1 \wedge v_2) \vee (\neg v_2 \wedge \neg v_3)$ and
  $S=(v_2\mt F, v_1\mt T)$, and suppose that the prior $D$ is uniform
and the testing accuracy is the same for all variables,
  so $o_1=\ldots=o_n=o$.
  This formula has four satisfying assignments, namely $\{TTT, TFF, TTF,
  FFF\}$ (letting $xyz$ denote the assignment $\{v_1\mapsto x,
v_2\mapsto y, v_3\mapsto z\}$, for brevity). The other four assignments,
namely $\{FFT, TFT, FTT, FTF\}$, make the formula false. For each
assignment $A$, the corresponding summand $\o{A}{S}$ is $\Pr(A)$ times
a factor of $o$ for every test outcome in $S$ that is compatible with
$A$, where a test outcome $v_i \approx b$ is \emph{compatible} with
$A$ if $b = A(v_i)$.  
For instance, the falsifying assignment $FFT$ is compatible with $v_2\mt F$
but not $v_1 \mt T$, so it gives rise to a summand of $\Pr(A) \cdot o$
in the numerator
of the characteristic fraction of $S$.
 On the other hand,
if $A$ is the the satisfying assignment $TFF$, then both
$v_1\mt T$ and $v_2\mt F$ are compatible with $A$, yielding $\Pr(A)\cdot
o^2$ in the denominator.
Performing the same analysis for the other
assignments and cancelling the common factors of $\Pr(A)$ (as the
prior is uniform), we find that  
  $$ \mathrm{cf}(\phi,S) = \frac{ o^1 + o^2+ o^0 + o^0}{o^1 + o^2 + o^1 + o^1}. $$
    For a more general example, suppose that
    $S=((v_1\mt T)^{c_1 k}, (v_2\mt F)^{c_2 k}, (v_3\mt F)^{c_3 k} )$ for some
    positive integer $k$ and real constants $0 \leq c_1, c_2, c_3 \leq 1$
with $c_1+c_2+c_3=1$. Then
$$ \mathrm{cf}(\phi,S) = \frac{ o^{c_2 k + c_1 k} + o^{c_2 k} + o^{c_3
    k} + o^0 }{ o^{c_1 k } + o^{c_1 k + c_3 k} + o^{c_2 k + c_3 k} +
  o^{c_1 k + c_2 k + c_3 k} }. $$ 
\hfill
\wbox
\end{example}

In the second example above, the characteristic fraction
of $S$ depends only on the factors $c_1$, $c_2$ and $c_3$,
that is, how often each test appeared in $S$.
We will in general be interested in the number of times each test
outcome compatible with a truth assignment appears in a test-outcome
sequence $S$.

\begin{definition} \label{def:trace}
Given a test-outcome sequence $S$ and truth assignment $A$, the
\emph{$A$-trace} of $S$, denoted $\Tr_A(S)$, is  
the vector
$ {\Tr}_A(S)=(n^+_{S,A,1}/|S|, \ldots, n^+_{S,A,n}/|S|)$.
\hfill \wbox
\end{definition}

\begin{example} Consider the sequence of test outcomes
$ S = ( v_1 \mt T, v_2 \mt T, v_1 \mt T, v_1 \mt T, v_1 \mt F )$.
This sequence has three instances of $v_1\mt T$, one instance of $v_1\mt F$
and one instance of $v_2\mt T$. So
the $\{v_1 \mapsto T, v_2\mapsto T\}$-trace of $S$ is $(\frac{3}{5},\frac{1}{5})$;
the $\{v_1 \mapsto F, v_2\mapsto T\}$-trace of $S$ is $(\frac{1}{5},\frac{1}{5})$.
The sequence
$$ S' = [ v_1 \mt T, v_2 \mt F, v_1 \mt T, v_1 \mt T, v_1 \mt T ] $$
has $4$ instances of $v_1\mt T$ and 1 of $v_2\mt F$, so 
the $\{v_1 \mapsto T, v_2\mapsto F\}$-trace of $S'$ is $(\frac{4}{5},\frac{1}{5})$.
\hfill
\wbox 
\end{example}

\begin{definition} \label{def:trace2}
    If $\vec{c}=(c_1,\ldots,c_n)$, $\phi$
        is a formula in the $n$ variables $v_1, \ldots, v_n$ and $A$ is a truth assignment,
then the \emph{characteristic
      fraction of the $A$-trace} is the function 
$\mathrm{cf}_A$, where
$$\begin{array}{lll}
  \mathrm{cf}_A(\phi, \vec{c}, k) 
&=& \frac{ \sum_{\{B:\phi(B)=\f\}} \Pr(B)
\prod_{ \{v_i:A(v_i)=B(v_i)\} } o_i^{c_i k} }
{
    \sum_{\{B:\phi(B)=\t\}} \Pr(B) \prod_{ \{v_i:A(v_i)=B(v_i)\}} o_i^{ c_i k } }. 
\end{array}$$
\hfill \wbox
\end{definition}

\begin{definition}
The test-outcome sequence $S$ is \emph{compatible with} truth
assignment $A$ if
all test outcomes in $S$ are consistent with $A$:
that is, $S$ contains no observations of the form
$ v_i\mt \neg A(v_i). $
\hfill \wbox
\end{definition}

The quantities $\mathrm{cf}(\phi,S)$ and $\mathrm{cf}_A(\phi, \vec{c}, k)$ are
clearly closely related.  The following lemma makes this precise.
\begin{lemma} \label{lem:cfiscfA}
  For all truth assignments $A$ compatible with $S$, we have
    $$ \mathrm{cf}(\phi, S) = \mathrm{cf}_A(\phi, {\Tr}_A(S), |S|).$$
\end{lemma}
\begin{proof}
If $A$ is compatible with $S$, then $(\Tr_A(S))_i =  \nSAp/|S|$ for all $i$,
so the result is immediate from the definition.
  \end{proof}

\commentout{
Recall that our overall goal is to understand the structure of the 
set of optimal sequences of test outcomes for a particular formula
and choice of parameters.
We achieve this by understanding the set of
$A$-traces
of optimal sequences of test outcomes, for each
truth
assignment $A$ that candidate sequences may be compatible with.
We show that as the length $k$ increases, the set of $A$-traces of
optimal test-outcome sequences of length $k$ converges
to a  convex polytope
in a way 
that can be derived from the asymptotic behaviour of the
characteristic fraction of the $A$-trace. 

If we determine, for every truth assignment $A$, the set $\mathcal{S}_A$
of sequences $S$ that maximise $|\Pr(\phi\mid S)-\frac{1}{2}|$
among the sequences compatible with $A$, 
then the set of optimal test-outcome sequences for $\phi$ consists of the
union of those sets $\mathcal{S}_A$ such that for $S \in
\mathcal{S}_A$ and $|\Pr(\phi\mid S)-\frac{1}{2}|$ is a maximum over
$\cup_{\mathrm{truth\, assignments\, }A'}\{|\Pr(\phi\mid S')-\frac{1}{2}|: S'
  \in \mathcal{S} _{A'}\}$.
  Moreover, since $$\Pr(\neg\phi\mid S) = 1-\Pr(\phi\mid S),$$
  we have
    $$|\Pr(\neg\phi\mid S)-\frac{1}{2}|=|\Pr(\phi\mid S)-\frac{1}{2}|.$$
Hence, we can equivalently determine, for every $A$, the set
$\mathcal{S}^+_A$ of sequences that maximise
$\Pr(\phi\mid S)$ and among sequence compatible with $A$ and the set
$\mathcal{S}^-_A$ of sequences that maximise
$\Pr(\neg\phi\mid S)$
among sequences compatible with $A$,
and then determine the set of optimal sequence
for $\phi$, $D$, and $\m$ by taking 
the union 
of those sets $\mathcal{S}^+_A$ and $\mathcal{S}^-_A$ whose elements $S$
attain the maximum 
$|\Pr(\phi\mid S)-\frac{1}{2}|$.
By Lemma \ref{lem:inverse}, characteristic fractions
are ordered inversely to the conditional probabilities %
that $\phi$ is true; therefore, we can determine those sets by finding
the sequences that minimise $\cf(\phi,S)$ or
$\cf(\neg\phi,S)$, respectively. 

\commentout{
Consider what happens to $\cf(\phi,S)$ as the test-outcome sequence $S$
gets longer.  
Eventually, we would expect only the highest-order terms to dominate;
since $o$ is greater than $1$,
characteristic fractions whose order is lower would be smaller.
When $\phi(A)=T$ for the assignment $A$ that $S$ supports, it is easy
to see that there is always a term of the form $o^{|S|}$ in the denominator
of $\cf(\phi,S)$, and that $o^{|S|}$ is the highest power of $o$ that can
occur in any term in either the numerator or the denominator.

Suppose that $\phi(A)$ is in fact $T$. Then 
the $A$-trace $\vec{c}$ of an optimal sequence for $\phi$ would have to minimise
the highest-order term in the numerator -- that is, the maximum over
assignments $B$ such that $\phi(B)=F$ of the corresponding exponent
$\sum_{\{i:A(v_i)=B(v_i)\}} c_i k$. Finding the possible $\vec{c}$ that 
attain this minimum is a linear problem, and being the minimum turns
out to not only be a necessary but also a ``nearly'' sufficient condition,
falling short only because we have ignored the constant factors before the
leading term due to multiplicity and priors. For example,
characteristic fractions of the form
$$ \frac{C_1 o^{0.7k}+\ldots }{C_2 o^k+\ldots} $$
will always be strictly smaller than ones of the form
$$ \frac{C_1 o^{0.8k}+\ldots }{C_2 o^k +\ldots} $$
for sufficiently large $k$ (assuming that the leading term is the one
with the largest exponent), but may or may not be strictly smaller
than ones of the form 
$$ \frac{2C_1 o^{0.7k}+\ldots }{C_2 o^k + \ldots} $$
depending on what is hidden inside the ``$\ldots$''.
}
}

Recall that our goal is to find optimal test-outcome sequences
for $\phi$, that is, sequences $S$ that maximise $|\Pr(\phi\mid
S)-\frac{1}{2}|$.  
By Lemma~\ref{lem:inverse}, this means that we want to either
minimise $\cf(\phi,S)$ or $\cf(\neg\phi,S)=1/\cf(\phi,S)$. 
By Lemma \ref{lem:cfiscfA},
we want to minimise $\cf_A(\phi,S)$ or $\cf_A(\neg\phi,S)$
for a truth assignment $A$ compatible with $S$.
Using Lemma~\ref{lem:denormtoprob}, it is easy to show
that if $S$ is sufficiently long and 
compatible $A$ and $\phi(A) = T$, then we must have $\Pr(\phi|S)\geq
\Pr(\neg\phi|S)$,  while if $\phi(A) = F$, the
opposite inequality must hold. So we need to minimise
$\cf_A(\phi,S)$ 
if $\phi(A)=T$ and to minimise $\cf_A(\neg\phi,S)$ if $\phi(A)=F$.
It suffices to find a sequence $S$ and a truth assignment $A$
that is compatible with $S$ %
for which the appropriate $\cf_A$ is minimal.

\medskip \noindent
{\bf Assumption:} We assume for ease of exposition in the remainder of the paper
that the measurement accuracy of each variable is the same, that is,
$\alpha_1 = \cdots = \alpha_n$.  This implies that $o_1 = \cdots = o_n$; we
use $o$ to denote this common  value. 
While we do not need this assumption for our results, allowing non-uniform
measurement vectors $\m$ would require us to parameterize
RI by the measurement accuracy; the formulae
that exhibit $(0.1,0.1)$-RI might not be the same as those that exhibit
$(0.1,0.3)$-RI.

With this assumption, we can show that $\cf_A(\phi,S)$ is essentially
characterised by the  terms in its numerator and denominator with the
largest exponents.  Every optimal test-outcome sequence $S$ is
compatible with some assignment $A$.  Since all test outcomes in $S$
are consistent with $A$,  
if $\phi(A)=T$, the summand due to $A$ in the denominator of
$\cf(\phi,S)=\cf_A(\phi,\Tr_A(S),|S|)$
is of the form $\Pr(A)o^{|S|}$.
This term must be the highest power of $o$ that occurs in the
denominator. 
The highest power of $o$ in the numerator of $\cf_A(S)$,
which is due to some assignment $B$ for which $\phi(B)=F$, 
will in general be smaller than $1\cdot |S|$, and depends on the
structure of $\phi$.
On the other hand, if $\phi(A)=F$, we want to minimise the 
characteristic fraction for $\neg\phi$, for which the sets
of satisfying and falsifying assignments are the opposite of those with
$A$. So, in either case,
the greatest power in the numerator of the characteristic fraction
we care about will be due to an assignment $B$ for which $\phi(B)\neq \phi(A)$.
As Lemma~\ref{lem:maxpower} below shows, we can formalise
the appropriate highest power as follows:

\commentout{
Also, we claim that if $\phi(A)=T$, then the order of the dominating term 
in $\cf_A(\neg\phi,\vec{c},k)$ for all $A$-traces $\vec{c}$ is
strictly larger than the minimum order of 
the dominating term in $\cf_A(\phi,\vec{c},k)$. 
Indeed, in that case
$\neg\phi(A)=F$, and the numerator of
$\cf_A(\neg\phi,\vec{c},k)$ must contain a term of order $o^k$, and 
so the dominating term is at least order $\frac{o^k}{o^k}$. On
the other hand, by choosing
$\vec{c}=(\frac{1}{n},\ldots,\frac{1}{n})$, we can attain a dominating
term of order at most 
$\frac{o^{(1-1/n)k}}{o^k}$ in $\cf_A(\phi,\vec{c},k)$. Therefore, it is sufficient 
to minimise only the order of the numerator $\cf_A(\phi,\vec{c},k)$ for those $A$ where $\phi(A)=T$, and symmetrically
only the order of the numerator of $\cf_A(\neg\phi,\vec{c},k)$ for
those $A$ where $\phi(A)=F$. 

Formally, we can define the notion of the \emph{order of the numerator} used above as follows:
}
\def\mp{\mathrm{maxp}_{\phi,A}}

\begin{definition} \label{def:maxpower}
    The \emph{max-power} of a vector $\vec{c}\in \mathbb{R}^n$ is
$$ \mp(\vec{c}) = \max_{\{B:\phi(B)\neq \phi(A)\}}
\sum_{\{i:A(v_i)=B(v_i)\}} c_i. $$
\hfill \wbox
\end{definition}
\begin{lemma}\label{lem:maxpower}
    If $S$ is a test-outcome sequence compatible with $A$ and
    $\varphi(A)=T$ (resp., $\varphi(A)=F$),
  then the highest power of $o$ that occurs in the numerator of
    $\cf(\phi,S)$ (resp., $\cf(\neg \phi, S)$ is   $|S|\mp(\Tr_A(S))$.
\end{lemma}
\begin{proof} This follows from the definition of
  $\cf_A(\phi,\Tr_A(S),|S|)$, the 
  observation that 
all entries in $\Tr_A(S)$
  are non-negative, and Lemma~\ref{lem:cfiscfA}.
  \end{proof}

We now show that the search for the max-power
can be formulated as a linear program (LP).
Note that if $R$ is a %
compact subset of $\IR$, finding a maximal element of the set
is equivalent to to finding a minimal upper bound for it:
$$\max R = \min \{ m \mid \forall   r\in R: r\leq m\}.$$
Hence, finding the
vector $\vec{c}$ with $\sum_i c_i = 1$ and $c_i \ge 0$ 
that attains the greatest max-power, that is, that maximises
$\max_{\{B:\phi(B)\neq \phi(A)\}} (
\sum_{\{i:A(v_i)=B(v_i)\}} c_i)$ is equivalent to finding the
$\vec{c}$ and max-power $m$ that minimise $m$ subject to
$\max_{\{B:\phi(B)\neq \phi(A)\}} \sum_{\{i:A(v_i)=B(v_i)\}} c_i \le m$,
$\sum_{i} c_i = 1$, and $c_i \ge 0$ for all $i$.
These latter constraints are captured by the following LP.
\begin{definition} \label{def:conflictlp}
  Given a formula $\phi$ and truth assignment $A$,
  define the \emph{conflict LP} $L_A(\phi)$ to be the linear
  program 
\begin{align*}
\text{minimise}  & \,\, m  \\
\text{subject to} &\,\, \sum_{\{ i : A(v_i)=B(v_i)\}} c_i \leq
m \text{ for all $B$ such that }\phi(B) \neq \phi(A) \\ 
&\,\, \sum_i c_i = 1. \\
&\,\, c_i \ge 0 \text{ for } i = 1,\ldots,n; \\
&\, \, 0 \le m \le 1.  
\end{align*}
\hfill \wbox
\end{definition}
The constraint $0 \le m \le 1$ is not necessary; since the $c_i$'s are
non-negative and $\sum_i c_i = 1$, the minimum $m$ that satisfies the
constraints must be between 0 and 1.  However, adding this constraint
ensures that the set of tuples $(c_1,\ldots, c_n,m)$ that satisfy the
constraints form a compact (i.e., closed and bounded) 
 set.
It is almost immediate from the definitions that the solution to $L_A(\phi)$
is $\sup_{\vec{c}: \sum_{i=1}^n c_i = 1, \ c_i  \ge 0} \mp(\vec{c})$.

We call $L_A(\phi)$ a \emph{conflict LP} because we are considering
assignments $B$ that \emph{conflict} with $A$,  in the sense that
$\phi$ takes a different truth value on them than it does on $A$. 
To reason about conflict LPs, we first introduce some notation.
\begin{definition} Suppose that $L$ is a linear program in $n$ variables
  minimising an objective function $f:\IR^n\rightarrow \IR$
  subject to some constraints.
\begin{itemize}
  \item The \emph{feasible set} of $L$,
$\Feas(L)\subseteq \IR^n$, is the set of points that satisfy all
        the constraints of $L$. 
\item The \emph{minimum} of the LP, $\mathrm{MIN}(L)$, is the infimum
  $\inf_{p\in \Feas(L)} f(p)$ attained by the objective function over
the points in $\Feas(L)$.
\item The \emph{solution polytope} of $L$, $\mathrm{OPT}(L)\subseteq
  \Feas(L)\subseteq \IR^n$, is the set of all feasible points at which $f$
  attains the minimum, that is, 
    $\mathrm{OPT}(L) = \{ p\in \Feas(L): f(p)=\mathrm{MIN}(L) \}$. 
\end{itemize}
\hfill \wbox
\end{definition}
It now follows that if $(\vec{d},m) \in \mathrm{OPT}(L_A(\phi))$, then
$\mp(\vec{d}) = m = \mathrm{MIN}(L_A(\phi))$.

Our goal is to show that the solutions
to the conflict LPs tell us enough about the structure of optimal
test-outcome sequences to derive 
a sufficient condition for a formula to exhibit RI.
Roughly speaking,
$\textrm{MIN}(L_A(\phi))$ tells us how well any sequence
of test outcomes compatible with the assignment $A$ can do.
Since every optimal
sequence is compatible with some assignment, we therefore can find the max-power
of
optimal sequences by considering the minimum of the minima of all LPs:
\begin{definition} For a formula $\phi$, define the \emph{minimax
        power} $\textrm{MIN}^*(\phi)$ 
to be the minimum of minima:
$$ \textrm{MIN}^*(\phi) = \min_{\text{assignments }A} \textrm{MIN}(L_A(\phi)). $$
An assignment $A$, and the LP $L_A(\phi)$, are \emph{relevant} if
$\textrm{MIN}(L_A(\phi))=\textrm{MIN}^*(\phi)$.  
\hfill \wbox
\end{definition}

The significance of this quantity is formalised by the following theorem.

\begin{theorem} \label{thm:lptorieasy}
If there exists a constant $C>0$ such that for all relevant truth
assignments $A$ 
 and all solution
points $\vec{c}=(c_1,\ldots,c_n,m)\in \mathrm{OPT}(L_A(\phi))$,
there exist indices $i$ and $j$ such that 
$v_i\leq_\phi v_j$,
$c_i\geq C$, and $c_j=0$, 
then $\phi$ exhibits RI.
\end{theorem}

\commentout{
    One might rightly be alarmed to observe that in the LP definition we made above, the coordinates $c_i$ did not quite look
    like the trace entries $c_i$ used in the characteristic function of the trace: they are positive, sum to $1$ (rather than $k$) and in fact
    for all means and purposes bounded between $0$ and $1$. The rescaling is necessitated by the circumstance
    that we are looking to capture the behaviour as $k$ itself grows, but the \emph{relative} fraction of tests
    of a given variable stays the same. Specifically, the LP solution space is related the the trace by the following definition.
}

\commentout{ 
    We will refer to a variant of the rational inattention condition for optimal test sequences that Proposition \ref{RIviatestseqs} uses often enough that it is worth attaching a name and symbol to it:
    \begin{definition} Suppose $\phi$ is a formula in $n$
      variables. For a sequence of test-outcome sequences for the
      formula of increasing length $\{S_k\}_{k\in K}$ where $|S_k|=k$,
      product distribution $D$, a negligible function $f$ and constant
      $c>0$, a variable $v$ and an index $k\in K$, the
      \emph{inattention condition} $P = \Feas(\{S_k\},D,f,c,v,k)$ is the
      logical proposition that there exists a variable $v' \leq_\phi
      v$ such that 
    \begin{itemize}
    \item $S_k$ (the sequence of test outcomes of length $k$ in the big sequence) contains at least $ck$ tests of $v'$, and
    \item $S_k$ contains at most $f(k)$ tests of $v$.
    \end{itemize}
\end{definition}
}

To prove Theorem~\ref{thm:lptorieasy}, we show that the
antecedent of the theorem implies the antecedent of
Proposition~\ref{RIviatestseqs}.
The next lemma is a first step
towards this goal.
Proposition \ref{RIviatestseqs} involves a condition on test sequences that
intuitively says that some variable is tested often, but
another variable 
that is at least as important is tested very little.  This condition
arises repeatedly in the following proof, so we attach
a name to it. 
\begin{definition} \label{def:badgood}
    Given a constant $c$ and negligible function $f$,
  a test-outcome sequence $S$ is $(f,c,\phi)$-\emph{good} if there 
exist variables $v_i$ and $v_j$ such that $v_i\geq_\phi v_j$, $S$
contains at least $c|S|$ tests of $v_j$, and $S$ contains at most $f(|S|)$
tests of $v_i$.  
$S$ is $(f,c,\phi)$-\emph{bad} if it is not $(f,c,\phi)$-good.
\hfill \wbox
\end{definition}
Using this notation, Proposition \ref{RIviatestseqs} says that  a
formula $\phi$ exhibits RI if 
for all open-minded product distributions $D$ and accuracy vectors
$\vec{\alpha}$, there exists a negligible function $f$ and $c>0$ such that all
test-outcome sequences optimal for $\phi$, $D$, and $\m$ are $(f,c,\phi)$-good. 
The contrapositive of Proposition~\ref{RIviatestseqs} says that if
a formula does not exhibit RI, then 
for all $f$ and $c$, there is
an $(f,c,\phi)$-bad test-outcome sequence optimal for $\phi$, $D$, and
$\m$. Bad test-outcome sequences are
counterexamples to RI. The next lemma allows us to ``boost''
such counterexamples if they exist: whenever we have a single bad
test-outcome sequence, we in fact have an infinite family of
arbitarily long bad test-outcome sequences 
that can be considered 
refinements of the same counterexample. 

\begin{lemma} \label{lem:singletoseq}
If,  for all negligible functions $f$ and constants $c>0$,
there exists 
an $(f,c,\phi)$-bad test-outcome sequence 
that is  optimal for $\phi$, $D$, and $\m$,
then for all $f$ and $c$, 
there exists an infinite sequence $\{S_k\}$ of 
$(f,c,\phi)$-bad optimal test-outcome sequences of 
increasing length (so that $|S_{k+1}| > |S_k|$),
all optimal for $\phi$, $D$, and $\m$.
\commentout{
whose traces tend to a limit, that is, for all assignments $A$,
$\lim_{k\rightarrow\infty} \Tr_A(S_k)$ exists. 
}
\end{lemma}
\begin{proof}
We show the contrapositive.  Fix $\phi$, $D$, and $\m$.  We show
that if  
  there exist $f$ and $c$ for which there is no infinite sequence
  $\{S_k\}$ of $(f,c,\phi)$-bad test-outcome sequences optimal for $\phi$,
  $D$, and $\m$, 
  then,
  for all $D$ and $\m$,
there exist $f''$ and $c''$ for which there is not even a single
$(f'',c'')$-bad test-outcome sequence that is optimal for $\phi$, $D$,
and $\m$.
Choose $f$ and $c$ such that the premises
of the contrapositive hold.  Let
$\mathcal{S}_{f,c}$ be the set of all 
$(f,c,\phi)$-bad test-outcome sequences that are optimal for $\phi$, $D$ and $\m$.
We can assume $\mathcal{S}_{f,c}$ is nonempty; otherwise 
the claim trivially holds.
If there exist arbitrarily long sequences $S\in \mathcal{S}_{f,c}$, then we can
pick a sequence $\{S_k\}$ of test-outcome sequences in
$\mathcal{S}_{f,c}$ of increasing length from them, 
contradicting the assumption.
In fact, this must be the case. For suppose, by way of
contradiction, that it isn't. 
Then there must be an upper bound $\hat{k}$ on the lengths of
test-outcome sequences in $\mathcal{S}_{f,c}$.  
Moreover, since there are only finitely many test-outcome 
sequences of a given length, $\mathcal{S}_{f,c}$ itself must also be finite.
Thus,
$$ c' = \min_{S\in \mathcal{S}_{f,c}} \max_{\text{variables }v_i\text{ in 
      } \phi} |({\Tr}_A(S))_i| $$
is finite and greater than zero (as every sequence must test at least
one variable and not contradict itself, so we are taking the
minimum over finitely many terms greater than zero).  Hence, 
$ c'' = \min \{c,c'\} $
is also greater than 0.
Let
$$ f''(k) = \begin{cases} k & \text{if $k\leq \hat{k}$} \\ f(k) &
  \text{otherwise.} \end{cases} $$ 
Since $f$ is negligible and $f''$ agrees with $f$ for all $k >
\hat{k}$, $f''$ is also negligible.

We claim that 
no 
test-outcome sequence $S$ optimal for $\phi$, $D$, and $\m$ is $(f'',c'')$-bad. 
Indeed, all candidate sequences of length $|S|\leq \hat{k}$ are ruled
out, because setting both $v_i$ and $v_j$ to be whatever variable is
tested the most in $S$
discharges the existential
quantification of Definition \ref{def:badgood} (note $\leq_\phi$ is a partial order, so $v_i \leq_\phi
v_i$ for all $v_i$) as the number of tests is bounded below by the minimum $c'|S|$
and above by the length $|S|$. 
Any test-outcome sequence $S$ of length $|S|>\hat{k}$ must also be
$(f'',c'')$-good. 
Indeed, by choice of $\hat{k}$, $S$ is $(f,c,\phi)$-good.
Therefore, there must be
a variable pair $v_i\geq_\phi v_j$ such that $S$ contains $\geq c|S|$
tests of $v_j$ and $\leq f(|S|)$ tests of $v_i$. But $c''\leq c$ by
definition and $f''(|S|)=f(|S|)$, so $v_i$ and $v_j$ also 
bear witness to $S$ being $(f'',c'')$-good.
This gives the desired contradiction.

Thus, we have shown that there exists a sequence $\{S_k\}$ of
bad test-sequence outcomes
in $\mathcal{S}_{f,c}$
of increasing length.
\commentout{
However,
for each assignment $A$, 
the limit $\lim_{k\rightarrow\infty} \Tr_A(S_k)$ does not
necessarily exist. To remedy this, note that $\Tr_A(S_k)$ takes values in
$[-1,1]^n$, which is a compact subset of
$\mathbb{R}^n$. Therefore, by the Bolzano-Weierstrass theorem,
this sequence has a convergent subsequence.  This
subsequence satisfies all required properties
has a limit for $A$.
Repeating this process for all truth assignments gives us a sequence
$\{S_k\}$ such that $\lim_{k \rightarrow \infty} \Tr_{A}(S_k)$ exists for
all $A$.
}
\end{proof}

In the following, we use the standard notion of \emph{1-norm}, where
the 1-norm of a real-valued vector $\vec{v}=(v_1, \ldots, v_n)$ is 
$$\|\vec{v}\|_1=\sum_{i=1}^n |\vec{v}_i|, $$
the sum of absolute values of the entries of $\vec{v}$.
We often consider the 1-norm of the difference of two vectors.
Although the difference of vectors is defined only if they have same
length, we occasionally abuse notation and write $\|\vec{v} -
\vec{w}\|_1$ even when $\vec{v}$ and $\vec{w}$ are vectors of
different lengths. In that case, we consider only the common
components of the vectors.  For example, if $\vec{v}=(v_1,\ldots,v_n)$ and
$\vec{w}=(w_1,\ldots,w_m)$, then $$\|\vec{v}-\vec{w}\|_1 =
|v_1-w_1| + \cdots + |v_{\min \{n,m\}}- w_{\min \{n,m\}} |.$$

The following fact about LPs will prove useful.
\begin{lemma} \label{lem:objsep} 
If $L$ is an LP with objective function $f$ such  that $\Feas(L)$ is compact,
then for all $\varepsilon>0$, there exists an
$\varepsilon'>0$ such that 
all feasible points $\vec{p}\in \Feas(L)$, either $\vec{p}$ is within
$\varepsilon$ 
of a solution point, that is, 
$$\exists \vec{o}\in \mathrm{OPT}(L)(\|\vec{p}-\vec{o}\|_1 < \varepsilon),$$   
or $f(\vec{p})$ is more than
$\varepsilon'$ away from the optimum, that is,
$$ f(\vec{p}) - \mathrm{MIN}(L)>\varepsilon'. $$
\end{lemma}
\begin{proof}
We will argue by contradiction. 
Suppose that the claim does not hold, and 
let $Q$ be the set of all points in $\Feas(L)$ that do not
satisfy the first inequality; that is, 
$$Q=\{ \vec{p}\in \Feas(L) : \forall \vec{o} \in
\mathrm{OPT}(L) (\|\vec{p}-\vec{o}\|_1 \geq \varepsilon) \}.$$ 
This set is bounded and closed, hence compact. 
If $\inf_{\vec{q}\in Q} (f(\vec{q}) - \mathrm{MIN}(L)) > 0$, then we
can take $\epsilon' = \inf_{\vec{q}\in Q} (f(\vec{q}) -
\mathrm{MIN}(L))/2$ since then, for every point $\vec{p} \in \Feas(L)$, if
$f(\vec{p}) - \mathrm{MIN}(L) \leq \epsilon'$, then $\vec{p} \notin
Q$, and hence by definition of $Q$, $\vec{p}$ must be within $\epsilon$ of some solution point.

So suppose that $\inf_{\vec{q}\in Q} (f(\vec{q}) - \mathrm{MIN}(L)) = 
0$.  Then
there exists a sequence $(\vec{q}_i)_{i=1}^\infty$ of points in $Q$
such that $\lim_{i\rightarrow \infty} f(\vec{q}) = \mathrm{MIN}(L)$. 
By the Bolzano-Weierstrass Theorem, this sequence must have a convergent
subsequence $(\vec{q}'_i)_{i=1}^\infty$. 
Write $\vec{q}^*$ for $\lim_{i\rightarrow \infty} \vec{q}'_i$. 
This limit point is still in $Q$,
as $Q$ is compact. Since $f$ is linear, hence continuous,
$$\begin{array}{lll}
    f(\vec{q}^*) &=& f(\lim_{i\rightarrow \infty} \vec{q}'_i) \\
    &=& \lim_{i\rightarrow \infty} f(\vec{q}'_i)
    \\
    &=& \lim_{i\rightarrow \infty} f(\vec{q}_i) \\
 &=& \mathrm{MIN}(L).
    \end{array}$$
Thus, $\vec{q}^*\in \mathrm{OPT}(L)$ and $\vec{q}^* \in Q$,
which is incompatible with the definition of $Q$.  This gives the
desired contradiction.
\end{proof}

We have seen how to distill the information in a test-outcome sequence for a formula in $n$
variables into a vector in $\IR^n$ by taking $A$-traces. The
following lemma is to be understood as an approximate converse of this process: given
a vector in $\IR^n$, we construct a test-outcome sequence
of a given length $k$ 
whose $A$-trace is close (within an error term of $2n/k$) to that vector.

\begin{lemma} \label{lem:approx}
If $A$ is an assignment to the $n$ variables of $\phi$ and
  $\vec{d}\in \IR^n$ is such that all coordinates are non-negative and
sum to 1, then for all $k\in \IN$, there exists a test-outcome sequence
$S_{k,\vec{d},A}$ of length $k$ compatible with $A$ such that 
$|\mp(\Tr_A(S_{k,\vec{d},A}) - \mp(\vec{d})| < 2n/k$.
\end{lemma}
\begin{proof}
Define
$$S_{k,\vec{d},A} = ( (v_1 \approx A(v_1))^{\lfloor d_1
    k \rfloor}, \ldots,  (v_n \approx A(v_n))^{\lfloor d_n k
    \rfloor}, (v_n \approx A(v_n))^{e} ), $$
where $\lfloor x\rfloor$ is the floor of $x$ (i.e., the largest integer
$n$ such that $n \le x$) and $e = k - (\sum_{i-1}^n \lfloor d_i k
\rfloor)$ is whatever is needed to pad the sequence to having length $k$.
(e.g., if  $\vec{d}=(0.3,0.7,\mathrm{MIN}^*(\phi))$ and $k=2$, then although 
the $d_i$s sum to 1, $\lfloor d_1 k \rfloor = 0$ and $\lfloor d_2 k
\rfloor = 1$, so we would have $e=1$.)

Since $\sum_i d_i k=k$, $\sum_i \lfloor d_i k\rfloor \leq k$, and hence
$e\geq 0$. Also, $\Tr_A(S_{k,\vec{d},A})$ differs from $\vec{d}$ by at most
$1/k$ in the first $n-1$ coordinates (as $|d_1 k - \lfloor d_1 k\rfloor| \leq 1$)
  and by at most $n/k$ in the final coordinate (as $e\leq n$).
  Since, for each assignment $B$,
  $$\left|\sum_{\{i:A(v_i)=B(v_i)\}} d_i -
  \sum_{\{i:A(v_i)=B(v_i)\}} (Tr_A(S_{k,\vec{d},A}))_i\right| \le
  (n-1)\frac{1}{k} + \frac{n}{k} \leq \frac{2n}{k},$$
  and for an arbitrary vector $\vec{c}$,
  $$\mp(\vec{c}) = \max_{\{B: \phi(B) \ne \phi(A)\}} \sum_{\{i:A(v_i)=B(v_i)\}} c_i, $$
it follows that $|\mp(\Tr_A(S_{k,\vec{d},A)}) - \mp(\vec{d})| < 2n/k$,
as desired.
\end{proof}

We can finally relate the solutions of the conflict LP $L_A(\phi)$ to
traces to the traces of optimal test-outcome sequences.
While the traces of optimal sequences may not be 
in $\mathrm{OPT}(L_A(\phi))$, they must get arbitrarily close to it
as the length of the sequence gets larger.

\begin{lemma} \label{lem:seqtolp}
    If $D$ is open-minded, then
there exists a function $\delta:\IN\rightarrow\IR$, depending only on
$\phi$, $D$, and $\m$,
  such that 
\begin{itemize}
\item  $\lim_{k\rightarrow
  \infty} \delta(k) = 0$ and
\item for
  all assignments $A$ and test-outcome sequences $S$ compatible
with $A$ that are optimal 
for $\phi$, $D$, and $\m$, 
the $A$-trace of $S$ is 
within $\delta(|S|)$ of some solution $(\vec{d},m) \in
\mathrm{OPT}(L_A(\phi))$, 
that is,
$$ \exists (\vec{d},m) \in OPT(L_A(\phi)).\, \| \vec{d}- {\Tr}_A(S) \|_1 <
\delta(|S|). $$ 

\end{itemize}
\end{lemma}
\begin{proof}
  Fix $\phi$, $D$, and $\m$.
  Given $\varepsilon >0$, we show that there
 exists a constant $k_\varepsilon$ such that for all truth assignments $A$ and 
 all test-outcome sequences $S$ compatible with $A$  such that $|S| >
 k_\epsilon$ and
\begin{equation}\label{eqn:Scond}
    \forall (\vec{d},m) \in \mathrm{OPT}(L_A(\phi)).\, \|{\Tr}_A(S)-\vec{d}\|_1
        \geq \varepsilon,
\end{equation}
$S$ is not optimal for $\phi$, $D$, and $\m$.
This suffices to prove the result, since
we can then choose any descending sequence $\varepsilon_0, \varepsilon_1, \ldots$
and define $\delta(n)=\varepsilon_n$ for all
$k_{\varepsilon_{n}}< n\leq k_{\varepsilon_{n+1}}$.  
Fix $\varepsilon>0$ and $A$. 
Choose an arbitrary test-outcome sequence $S$
compatible with $A$
satisfying (\ref{eqn:Scond}).
Without loss of generality, we can assume that $\phi(A)=T$.
(If $\phi(A)=F$, then the lemma follows from applying the argument
below to $\neg \phi$ and the observation that sequences are optimal for $\phi$
 iff they are optimal for $\neg\phi$.)
Since the feasible set of the LP $L_A(\phi)$ is compact by construction, by
Lemma \ref{lem:objsep}, there exists some 
$\varepsilon_A>0$ 
such that for all feasible points $p=(\vec{c},m)\in \Feas(L_A(\phi))$,
 either $\|\vec{c}-\vec{d}\|_1 < \epsilon$ for
some $\vec{d} \in \mathrm{OPT}(L)$, or $|m -
\mathrm{MIN}(L_A(\phi))|>\epsilon_A$.
Set $$k_{\epsilon,A} = \max\left(\frac{4n}{\epsilon_A}, \frac{2}{\varepsilon_A}
\log_o \left(\frac{2^{2n}}{\Pr(A)\min_B \Pr(B)}\right)\right).$$
(Since $D$ is open-minded, $\min_B \Pr(B) > 0$, so this is well defined.)
  We now show that if $|S| > k_{\epsilon,A}$, then $S$ is not optimal
  for $\phi$, $D$, and $\m$.
  We can then take $k_{\epsilon} = \max_A   k_{\epsilon,A}$ to complete the proof.

Since $S$ satisfies (\ref{eqn:Scond}) by
assumption, $\|(\Tr_A(S),\mp(\Tr_A(S)))-\vec{d}\|_1>\varepsilon$
for all $\vec{d}\in \mathrm{OPT}(L_A(\phi))$, so 
$$\begin{array}{lll}
    |\mp({\Tr}_A(S))-\mathrm{MIN}(L_A(\phi))| &>& \varepsilon_A.
  \label{eqn:farfrommstar} 
  \end{array}$$
Since all the entries in $\Tr_A(S)$ are
non-negative, it follows from the definition that 
\begin{equation} \label{cfA-geq}
\begin{array}{llll}
 & & \cf_{A}(\phi,{\Tr}_A(S),|S|)\\
 &=& \frac{ \sum_{\{B:\phi(B)=\f\}} \Pr(B)
    o^{\sum_{\{v_i:A(v_i)=B(v_i)\}} \Tr_A(S)_i |S| } }{
    \sum_{\{B:\phi(B)=\t\}} \Pr(B) o^{\sum_{\{v_i:A(v_i)=B(v_i)\}}
      \Tr_A(S)_i |S| } }
\\
    &\ge& \frac{ \min_B \Pr(B) o^{\mp(\Tr_A(S))|S|} }{ 2^n o^{|S|} }
   & \mbox{[see below].}
\end{array}
\end{equation}
The  inequality holds because, as we observed before, the term in the
numerator with the greatest exponent has exponent $\mp(\Tr_A(S))|S|$.
Its coefficient is at least $\min_B \Pr(B)$.  The remaining terms in
the numerator (if any) are nonnegative.  Thus, the numerator is at
least as large as $\min_B \Pr(B) o^{\mp(\Tr_A(S))|S|}$. There are $2^n$
  terms in the denominator, each of which is at most $o^{|S|}$, since,
  as we observed earlier, $\sum_i \Tr_A(S)_i = 1$ (since $S$ is
  compatible with
  $A$).
  Thus, the denominator is at most $2^n o^{|S|}$.

Fix $(\vec{d},m)\in \mathrm{OPT}(L_A(\phi))$.  By Lemma \ref{lem:approx},
there exists a test-outcome sequence  $S_{|S|,\vec{d},A}$ such that 
$|\mp(\Tr_A(S_{|S|,\vec{d},A})) - \mp(\vec{d})| < 2n/|S|$.
For brevity, set $\vec{d}' = \Tr_A(S_{|S|,\vec{d},A}).$
So if $|S| > k_{\epsilon,A} \ge 4n/\epsilon_A$, then
  $|\mp(\vec{d}')-\mp(\vec{d})| 
< \epsilon_A/2.$ 
Since $(\vec{d},m)  \in \mathrm{OPT}(L_A(\phi))$, we have that
$\mp(\vec{d})=m = \mathrm{MIN}(L_A(\phi))$, so
$|\mp(\vec{d}')-\mathrm{MIN}(L_A(\phi))| < \varepsilon_A/2$.
Now using (\ref{eqn:farfrommstar}) and applying the triangle
inequality gives us that 
\begin{equation}\label{eqn:sep} 
    \mp({\Tr}_A(S))- \mp(\vec{d}') >
   \varepsilon_A/2. 
\end{equation}
Much as above, we can show that
\begin{equation} \label{cfAd-leq}
\begin{array}{lll}
  & & \cf_{A}(\phi,\vec{d}',|S|)
 \\
 &=& \frac{ \sum_{\{B:\phi(B)=\f\}} \Pr(B)
    o^{\sum_{\{v_i:A(v_i)=B(v_i)\}} d'_i |S| } }{
    \sum_{\{B:\phi(B)=\t\}} \Pr(B) o^{\sum_{\{v_i:A(v_i)=B(v_i)\}}
      d'_i |S| } }
 \\
 &\le& \frac{ 2^n o^{\mp(\vec{d}') |S|} }{ \Pr(A) o^{|S|} },
\end{array}
\end{equation}
where now the inequality follows because we have replaced every term
$\Pr(B)$ in the numerator by 1 and there are at most $2^n$ of them,
and the fact that $\Pr(A) o^{|S|}$ is one of the terms in the denominator
and the rest are non-negative.

Now observe that
\begin{equation}\label{upperlower}
  \begin{array}{lllll}
&&   \Pr(\phi \mid S_{|S|,\vec{d},A}) > \Pr(\phi \mid S)\\
    &\mbox{iff} &\cf(\phi,S_{|S|,\vec{d},A}) < \cf(\phi,S) &\mbox{[by
        Lemma \ref{lem:inverse}]}\\ 
&\mbox{iff} &\cf_{A}(\phi, \vec{d}', |S|) <  \cf_{A}(\phi, \Tr_A(S), |S|)  &\mbox{[by Lemma \ref{lem:cfiscfA}]}\\
    &\mbox{if} & \frac{ 2^n o^{\mp(\vec{d}') |S|} }{ \Pr(A) o^{|S|} } < \frac{ \min_B
      \Pr(B) o^{\mp(\Tr_A(S)) |S|} }{ 2^n o^{|S|} }
&\mbox{[by (\ref{cfA-geq}) and (\ref{cfAd-leq})]}\\
&\mbox{iff} &
\frac{{\min}_B \Pr(B) o^{\mp(\Tr_A(S)) |S|} }{ 2^n o^{|S|} } - \frac{
  2^n o^{\mp(\vec{d}') |S|} }{ \Pr(A) o^{|S|} } > 0
\\
&\mbox{iff}
&\frac{ \Pr(A) \min_B \Pr(B) o^{\mp(\Tr_A(S)) {|S|}} - 2^{2n}
  o^{\mp(\vec{d}') {|S|}}  }{ \Pr(A) 2^n o^{|S|} } > 0
\\
&\mbox{iff} &\Pr(A) {\min}_B \Pr(B) o^{\mp(\Tr_A(S)) |S| -
    \mp(\vec{d}') {|S|})} - 2^{2n} > 0
\\
&\mbox{iff} &(\mp(\Tr_A(S))  - \mp(\vec{d}')) {|S|} > \log_o
\left(\frac{2^{2n}}{\Pr(A)\min_B \Pr(B)}\right).
  \end{array}
  \end{equation}
By assumption, $|S| >\frac{2}{\varepsilon_A} \log_o \frac{2^n
  2^n}{\Pr(A)\min_B \Pr(B)}$; by (\ref{eqn:sep}),  
$(\mp(\Tr_A(S))  - \mp(\vec{d}')) > \varepsilon_A/2$. 
It follows that the last line of (\ref{upperlower}) is in fact
satisfied.
Thus $S$ is not optimal, as desired.
\end{proof}
\commentout{
We showed earlier that (\ref{eqn:sep}) holds
if $|S|>k_1$. We will now show that if this is the case, then there is
another constant $k_2$ such that
if $|S|>k_2$, then (\ref{eq:lemb21}) holds.
Given that, by  (\ref{eqn:sep}), $\mp(\Tr_A(S)) - \mp(\vec{d}')
>\epsilon_1/2$,
it follows that
$$ {|S|} \geq k_2 = \frac{2}{\varepsilon_A} \log_o \frac{2^n 2^n}{\Pr(A)\min_B
    \Pr(B)} $$
is a sufficient condition for
(\ref{eq:lemb21}) to hold.
Set $k_\varepsilon = \max \{k_1, k_2\}$.  It follows if
$|S|>k_\varepsilon$, then both (\ref{eqn:sep}) and (\ref{eq:lemb21})
hold, so 
so $S$ is not optimal, as desired.}
Moreover, unless the sequence in question is short, any optimal sequence
of test outcomes must be compatible with an LP that actually attains the minimax
power.

\begin{lemma} \label{lem:optimalrelevant}
There exists a constant $k_0$, depending only on $\phi$, $D$ and $\m$, such that
if a sequence $S$ of length $|S|\geq k_0$ is compatible with an assignment $A$,
then either $S$ is not optimal or $A$ is relevant.
\end{lemma}
\begin{proof}
The proof reuses many of the core ideas of Lemma \ref{lem:seqtolp} in a simplified setting.
For contradiction, suppose that $A$ is not relevant, but $S$ is optimal.
Let $B$ be an arbitrary relevant assignment.
Then 
$$\mathrm{MIN}(L_A)-\mathrm{MIN}(L_B) = \epsilon > 0.$$
We show that we can choose a $k_0$ such that if $|S|> k_0$, then
there is a test-outcome sequence $S'$ of the same length supporting $B$ that
is actually better, contradicting the optimality of $S$.

Indeed, set $$k_0 = \max\left\{ 4n/\epsilon,
\frac{2}{\epsilon}\log_o\left(\frac{2^{2n}}{\Pr(B)\min_C \Pr(C)}\right)
\right\}.$$ 
Since $\Tr_A(S)$ is a feasible point of $L_A$, we have $\mp{\Tr_A(S)} \geq \mathrm{MIN}(L_A) \geq \mathrm{MIN}(L_B) + \epsilon$.
On the other hand, let $(\vec{d},m)\in \mathrm{OPT}(L_B(\phi))$ be arbitrary.
Since $|S|>4n/\epsilon$, the $B$-trace
$\vec{d}'={\Tr}_B(S_{k,\vec{d},B})$ of the sequence
$S_{k,\vec{d},B}$
 of
Lemma \ref{lem:approx} satisfies 
$$|\mp(\vec{d}')-\mp(\vec{d})|=|\mp(\vec{d}') - \mathrm{MIN}(L_B)| < \epsilon/2.$$
So $\mp(\Tr_A(S))-\mp(\vec{d}') > \epsilon/2$.
As in the proof of Lemma \ref{lem:seqtolp}, we have
$$
 \cf_{A}(\phi,{\Tr}_A(S),|S|)
    \ge \frac{ \min_B \Pr(B) o^{\mp(\Tr_A(S))|S|} }{ 2^n o^{|S|} }
$$
for $S$ and
$$ \cf_{B}(\phi,\vec{d}',|S|) 
 \le \frac{ 2^n o^{\mp(\vec{d}') |S|} }{ \Pr(B) o^{|S|} }
 $$
for the synthetic sequence, and hence
\begin{equation*}
  \begin{array}{lllll}
&&   \Pr(\phi \mid S_{|S|,\vec{d},B}) > \Pr(\phi \mid S)\\
&\mbox{iff} &\cf(\phi,S_{|S|,\vec{d},B}) < \cf(\phi,S) \\
&\mbox{if} &   (\ldots)\\
&\mbox{iff} &(\mp(\Tr_A(S))  - \mp(\vec{d}')) {|S|} > \log_o
\left(\frac{2^{2n}}{\Pr(B)\min_C \Pr(C)}\right).
  \end{array}
  \end{equation*}
So $|S|> k_0 \geq \frac{2}{\epsilon}\log_o
\left(\frac{2^{2n}}{\Pr(B)\min_C \Pr(C)}\right)$, which implies 
that $S$ is indeed not optimal.
\end{proof}
\commentout{ \begin{proof}
Follows from a small modification to the proof of Lemma \ref{lem:seqtolp}.
Rather than assuming (\ref{eqn:Scond}), we simply assume that $A$ is not relevant.
Let $B$ be an arbitrary relevant assignment; then this implies
$$\mathrm{MIN}(L_A)-\mathrm{MIN}(L_B) = \epsilon_2 > 0$$
and hence, as ${\Tr}_A(S)$ is a feasible point of $L_A(\phi)$,
$$\mp({\Tr}_A(S)) -\mathrm{MIN}(L_B) > \epsilon_2/2 =: \epsilon_1 $$
can be used in (\ref{eqn:sep})'s stead.

Then, instead of comparing $S$ to a sequence $S_{|S|}(\vec{d}, A)$
synthesised from a solution point $\vec{d}\in \mathrm{OPT}(L_A)$, 
we take a sequence $S_{|S|}(\vec{d}, B)$ for a solution point $\vec{d}\in \mathrm{OPT}(L_B)$ of a relevant LP $L_B$.
As before, we conclude that if $|S|> \max \{ 4n/\epsilon_1, k_2\}$ ($k_2$ defined as in the preceding proof, but in terms of the alternative choice of $\epsilon_1$ here), a set of inequalities holds that together implies $S$ is not optimal, and $k_\epsilon$ only depends on $D$, $\m$ and the structure of $\phi$ (partially via our new $\epsilon_1$).
\end{proof} }

With these pieces, we can finally prove Theorem~\ref{thm:lptorieasy}.

\begin{proof}[Proof (of Theorem \ref{thm:lptorieasy})]
  Suppose, by way of contradiction, that the antecedent of
  Theorem~\ref{thm:lptorieasy} holds, but
  $\phi$ does not exhibit RI. 
 Let $\delta$ be the function of Lemma \ref{lem:seqtolp} and let
$C$ be the constant that is assumed to exist in the statement of
Theorem~\ref{thm:lptorieasy}.  Define $f$ by taking $f(k) =
\delta(k)k$.  Since $\lim_{k \rightarrow \infty}f(k)/k = \lim_{k
  \rightarrow \infty} \delta(k) = 0$, $f$ is negligible.
By Proposition \ref{RIviatestseqs}, there exists an open-minded product
  distribution $D$ and accuracy vector $\alpha$ such that 
  there exists an
$(f,C/2)$-bad test-outcome sequence optimal for $\phi$, $D$, and $\m$.
 So by Lemma \ref{lem:singletoseq},
there exists an infinite sequence $\{S_k\}$ of
$(f,C/2,\phi)$-bad test-outcome sequences that are optimal for $\phi$, $D$,
and $\m$ and are of increasing length.
  Thus,
\begin{equation}\label{eq:contradiction}
\parbox{0.85\textwidth}{for all $k$, there are no variables $v_j \geq_\phi v_i$ such
    that $v_j$ is tested at most $f(|S_k|)$ times, but $v_i$ is tested at
        least $C|S_k|/2$ times.}
\end{equation}
\commentout{
\item $\lim_{k\rightarrow \infty}   { \Tr}_A(S_k)$ exists for all
    truth assignments $A$.
\end{itemize}
}
We can assume
    without loss of generality 
        that all the sequences $S_k$ are compatible with the same
        assignment $A$,  
    since there must be an assignment $A$ that infinitely many of the
    sequences $S_k$ are compatible with, and we can consider the
    subsequence consisting 
    just of these test-outcomes sequences that are compatible with $A$.
    Moreover, by Lemma \ref{lem:optimalrelevant}, we can assume
    that
    $A$ is relevant,
since all but finitely many of the $S_k$ must be sufficiently long.

Let $k_1$ be sufficiently large that $\delta(k)<C/2$ for all
$k>k_1$. By Lemma \ref{lem:seqtolp}, for all $k>k_1$, we must have  
$$ \| \vec{d}-{\Tr}_A(S_k)\|_1 < \delta(k) < C/2 $$
for some solution $(\vec{d},m)$ to the LP $L_A(\phi)$.  Since $A$ is
relevant by construction, the assumptions of the theorem guarantee 
that there exist $i$ and $j$ such that $v_i \le_{\phi} v_j$, $d_i >
C$, and $d_j = 0$.  Since $\| \vec{d}-{\Tr}_A(S_k)\|_1 < \delta(|S_k|)$,
it follows that $(\Tr_A(S_k))_i > C - \delta(|S_k|) > C/2$ and
$(\Tr_A(S_k))_j < \delta(|S_k|)$.
Since each sequence $S_k$ is compatible with $A$, for each variable
$v_h$, $n_{S_k,A,h}^+$ is just the number of times that $v_h$
is tested in $S_k$, so $(\Tr_A(S_k))_h$ is the
number of times that $v_h$ is tested divided by $|S_k|$.   This means
that we have a contradiction to (\ref{eq:contradiction}).
\commentout{
Since this holds
for the 1-norm, it must also hold pointwise. 
Also, since $\vec{d}=(d_1,\ldots,d_n,\mathrm{MIN}^*(\phi))$ is in the
LP solution (note that the LP is relevant!), there must be variables $v_j \geq_\phi v_i$ such that
$d_j=0$ and $d_i>C$. 
So by the 1-norm bound, 
\begin{itemize}
\item $\Tr_A(S_k)_j < \delta(k)$
\item $\Tr_A(S_k)_i > C-\delta(k) > C/2$.
\end{itemize}

This contradicts the properties claimed for all $S_k$ by Lemma
\ref{lem:singletoseq}, and proves the desired result.
}
\end{proof}

\commentout{
    This leads to the following key lemma.

    \begin{lemma} \label{lem:lpasymptotic} Suppose $\{S_i\}_{i\in
        \mathbb{N}}$ is a sequence of test-outcome sequences of
            increasing length, all $S_i$ are compatible with a common assignment $A$ and
      the limit of trace views $\lim_{i\rightarrow \infty}
      \view_A(\Tr(S_i))$ exists. Suppose moreover that each $S_i$ is
      optimal for its length, in the sense that $|\Pr(\phi|S_i)-1/2|
      \geq |\Pr(\phi|S')-1/2|$ for all $S'$ with $|S'|=|S_i|$. 

    Then $\lim_{i\rightarrow \infty} \view_A(\Tr(S_i))$ is a solution
    to the LP $L_A(\phi)$. 
    \end{lemma}

    \begin{proof}
    Let $T_S$ be the limit $\lim_{i\rightarrow \infty}
    \view_A(\Tr(S_i))$, and $T_L$ be an arbitrary solution to
        $L_A(\phi)$. We claim that if $i$ (and hence the length of the sequence
     $S_i$) is sufficiently large, there exists a sequence $S_{L,i}$ of test
  outcomes such that $|\Pr(\phi\mid S_{L,i})-1/2| >
    |\Pr(\phi\mid S_i)-1/2|$, contradicting optimality of $S_i$. 

    We will choose $S_{L,i}$ to be a sequence the $A$-view of whose
    trace is a ``reasonable approximation'' to the LP solution $T_L$,
    that is, which tests the $j$th variable approximately a
    $T_{L,j}$-fraction of times resulting in an outcome that agrees
    with $A$. (The approximateness is needed since not every trace may
    be realisable with a sequence of given finite length, e.g. if the
    LP solution assigns $c_j=\frac{1}{3}$ but the number of tests is
    not divisible by 3.) Then as $\view_A(\Tr(S_i))$ gets sufficiently
    close to the limit $T_S$ and $\view_A(\Tr(T_{L,j}))$ gets
    sufficiently close to the LP solution $T_L$, $T_S$ not being an LP
    solution implies that the characteristic fraction of $T_L$ is
    strictly smaller. 

    \end{proof}

    The importance of this Lemma is in that it bridges the gap between
    the conflict LP and RI by constraining what
    ``most'' (all but finitely many) optimal sequences of test
    outcomes must look like in terms of solutions to the LP: for any
    countable family (sequence) of sequences of test outcomes, by
   the Bolzano-Weierstrass Theorem, there must be some subsequence whose traces

    converge, and the point to which they converge must be a solution
    to the LP. In particular, if we consider the countable family of
    all finite sequences of test outcomes that are optimal for their
    length, then whenever infinitely many of them contain a constant
    fraction of measurements of a particular variable $v_i$, the
    conflict LP must have a solution in which the corresponding $c_i$
    is nonzero. 

    \begin{lemma}

            (i) Suppose that all solutions to the conflict LP $L_A(\phi)$ for $\phi$
      have $c_i=0$. 

    Then there exists a negligible function $f_i$ such that all optimal sequences $S$ of test outcomes for $\phi$ that contain $A$ contain at most $f(|S|)$ tests of $v_i$.

    (ii) Conversely, suppose that all solutions have $c_i>c$ for some constant $c>0$.

    Then all optimal sequences $S$ of test outcomes for $\phi$ 
        compatible with $A$ contain at least $c'|S|$ tests of $v_i$ for some
    $c'>0$. 

    \end{lemma}

    \begin{proof} (i) Set $f_i$ to be the upper bound on the fraction of measurements of $v_i$ among optimal sequences of a given length, that is, $f_i(k) = \sup_{S\in \mathcal{S}_k} \#(\text{tests of $v_i$ in $S$})/k$.
    Since the number of possible sequences of test outcomes of a fixed length is finite, for each $k$, the supremum $f_i(k)$ must actually be attained by some sequence $S^*_k \in \mathcal{S}_k$. 

    Recall that a function $f$ is negligible if $f(k)/k \rightarrow 0$ as $k\rightarrow \infty$, and so for any $\varepsilon>0$, there is an $k_0$ such that for all $k\geq k_0$, $f(k) \leq \varepsilon k$. So conversely, for $f_i$ to not be negligible, there must be some $d_i>0$ such that there is no such $k_0$, and hence all $f_i(k)$ are $> d_i k$.

    The $\{S^*_k\}_k$ form a sequence of sequences of test outcomes of increasing length.
        Since all the $S^*_k$ are optimal for $\phi$ and are
        compatible with $A$ by
    Proposition \ref{prop:aviewoftrace}, the $i$th coordinates of
    their $A$-views correspond to the fraction of measurements of
    $v_i$ in them, and therefore $\view_A(\Tr(S^*_k))_i\geq d_i>0$.  

    Since at the same time $1\geq \view_A(\Tr(S^*_k))_i$, by the
    Bolzano-Weierstrass theorem (applied repeatedly for every
    coordinate), there must exist a subsequence of sequences of test
    outcomes $\{S^*_{f(j)}\}_{j\in \mathbb{N}}$ such that the limit
    $T_S=\lim_{j\rightarrow \infty} \view_A(\Tr(S^*_{f(j)}))$
    exists. So by Lemma \ref{lem:lpasymptotic}, $T_S$ must be a
    solution to $L_A$. However, since $\view_A(\Tr(S^*_k))_i \geq d_i
    > 0$, this bound must also apply to the $i$th coordinate of the
    limit, $(T_S)_i$, contradicting the assumption that all solutions
    to $L_A$ have the $i$th coordinate equal to $0$. 

    Therefore, $f_i$ must be negligible. $\Box$

    (ii) Analogous.

    \end{proof}
}

\commentout{
    Say that $M$ \emph{supports} a truth assignment 
    $A_M$ if $A$ is one of the most likely truth assignments conditional
    on $M$.
    Assume for now that $M$ contains no tests of the form $v\mt \neg
    A_M(v)$, so all tests in it agree with $A_M$,
    (i.e., $|M∩A_M|=|M|=k$).

    Then the best test sequence $M^*$ of length $k$ is defined by
$$        \begin{array}
     M^* &=& \underset{M:|M|=k}{\arg\min} \frac{ \sum_{\{A:\; \varphi(A)=\f\}} o^{ |A∩M|} }{ \sum_{\{A:\; \varphi(A)=\t\}} o^{|A∩M| } }.
         \end{array}$$
    As $k$ (and the number of tests of other variables) is allowed to go
    towards infinity, the highest-order terms in the numerator and
    denominator dominate. Assume henceforth WLOG that $A_M$ is a
    satisfying truth assignment for $\varphi$
        (i.e., $\varphi(A_M)=\t$). Then, as $M$ agrees with $A_M$, in the
    denominator, the dominant term is simply $o^k$ in every case. Our
    objective therefore reduces to (1) minimising the power in the
    numerator and (2) among all solutions optimal for (1), minimising the
    ratio of the coefficients of the highest-order terms in the numerator
    and denominator. We will call the solutions that are optimal for (1)
    \emph{asymptotically best}.

    In the numerator, the dominant term is going to be due to a falsifying truth assignment $B$ for which $|B∩M|$, the number of tests in $M$ that agree with $B$, is maximised: that is, a falsifying truth assignment that the largest number of test outcomes is consistent with. To minimise the fraction (thereby maximising the probability that the formula is true, given the test sequence), we want this maximum to be as low as possible. Since we are assuming that $M$ is supporting a single truth assignment $A_M$, the relevant test outcomes in $M$ are simply going to be those that are tests of variables where $A_M$ and $B$ agree. For example, if $M$ is $(x\mt \t, x\mt \t, y\mt \f)$, then $A_M$ is $\{x\mapsto \t, y\mapsto \f\}$, and $|A\cap \{x\mapsto \t, y\mapsto \t\}|=2$, whereas $|A\cap \{x\mapsto \f, y\mapsto \f\}|=1$.

    In general, any test sequence in support of one truth assignment
    is bound to lend partial support to at least some truth assignment
    on which the formula takes the opposite truth value -- the only
    way this could be avoided is if there is some variable on which
    $A_M$ and all truth assignments on which the formula has the
    opposite truth value disagree, so the formula can be rewritten in
    the form $v \vee  (\neg v\wedge (\ldots))$ or its negation. 

    A picture therefore emerges where among the test sequences that support a given fixed satisfying truth assignment, we are trying to distribute the tests (in support of the individual variable settings in the truth assignment) in such a way that the closest falsifying truth assignment to the formula receives as little ``collateral support'' as possible. For instance, consider the five-variable formula $\neg ((x\wedge y\wedge \neg a\wedge \neg b\wedge \neg c)\vee (\neg x\wedge \neg y\wedge a\wedge b\wedge c))$, which has two falsifying truth assignments. Then dividing up $k$ tests in support of $x\wedge y\wedge a\wedge b\wedge c$ evenly (so that $M$ has $k/5$ lots of $x\mt \t$, $k/5$ lots of $y\mt \t$, etc.) would result in the truth assignment $\neg x\wedge \neg y\wedge a\wedge b\wedge c$ contributing a term of $o^{3k/5}$ to the numerator. If we just distribute the tests among $x$ and $y$ ($k/2$ each), this truth assignment now contributes $o^0$, but actually things look even worse: the truth assignment $x\wedge y\wedge \neg a\wedge \neg b\wedge \neg c$ is now also supported by all tests, contributing a term of $o^k$! It turns out that in this case, the optimal distribution is to distribute a $\frac{1}{2}$-fraction of tests among $x$ and $y$, and a $\frac{1}{2}$-fraction again among $a$, $b$ and $c$, making both $\neg x\wedge \neg y\wedge a\wedge b\wedge c$ and $x\wedge y\wedge \neg a\wedge \neg b\wedge \neg c$ contribute a term of $o^{k/2}$.

    Formally, the asymptotically best solutions, among those supporting a fixed truth assignment $A$, are the solutions to the following program (in the OR sense):
    \begin{align*}
    \min &  \,\max \left\{ \sum \{ c_v \mid v\in |A ∩ B| \} \mid \varphi(B) \neq \varphi(A) \right\} \\
    \mathrm{s.t.} & \,\sum_v c_v = k,
    \end{align*}
    which can be rewritten in LP form as
    \begin{align*}
    \min & \,\, m  \\
    \mathrm{s.t.} &\,\, \sum \{ c_v \mid v\in |A ∩ B| \} \leq m \hspace{2em} \forall B: \,\varphi(B) \neq \varphi(A) \\
     &\,\, \sum_v c_v = k.  
    \end{align*}
    For each variable $v$, the constant result $c_v$ denotes how many tests out of the $k$ total ones should be of the form $v\mt A(v)$.

    The asymptotically best solutions for the whole formula are then
    simply the vector $\vec{c}_v$ that minimise the minimum of this LP
    across all choices of $A$. If all such solutions involve ignoring
    a variable (have $c_v=0$ for some variable $v$), then we can
    immediately conclude that the formula exhibits RI
    for sufficiently large $k$, whereas if none of them

    do, we can immediately conclude that exhibits none for
    sufficiently large $k$. Otherwise, for every solution, we need to
    determine the coefficients of the highest-order terms: that is,
    count the satisfying truth assignments that are supported by every
    test on the one hand, and the falsifying truth assignments $B$
    that attain the maximum value of $\sum \{ c_v \mid v\in |A ∩ B|
    \}$ on the other. The true optimal test sequences are those for
    which the ratio between the latter and the former is minimised. If
    all of these have a zero $c_v$, we once again conclude the formula
    has rational inattention, and if all of them don't, it doesn't.

    \hl{If both solutions with all $c_v$ nonzero and solutions where some $c_v$ are zero exist, the existence of rational inattention depends on the prior. TODO}

}

\subsection{LP lower bound for rational inattention}

\def\v#1{\vec{#1}}

Theorem \ref{thm:lptorieasy} gives us a criterion that is
sufficient to conclude that a given formula exhibits
rational inattention: 
there exists a $C$ such that for all relevant assignments $A$, 
and all $\vec{c}$ such that $(\vec{c},m)\in
\mathrm{OPT}(L_A(\phi))$ for some $m$, there exist 
entries $c_i$ and $c_j$ such that $v_i\leq_\phi v_j$, $c_i\geq C$, and
$c_j=0$.  We call this property $P_C$, and write $P_C(\vec{c})$ if
$\vec{c}$ satisfies the property.
To compute how many formulae exhibit RI, we want an efficient algorithm
that evaluates $P_C$.

LPs (such as $L_A(\varphi)$) are known to be solvable in polynomial
time (see, e.g., \cite{karmarkar84}). 
However, rather than finding a description of the entire solution polytope
$\mathrm{OPT}(L_A(\phi))$, standard linear programming algorithms such
as that of Karmarkar \cite{karmarkar84} 
compute only a single point inside the polytope. Since we are
interested in whether \emph{all} points in the polytope satisfy $P_C$,
we have  
to do some additional work before we can leverage standard LP solvers. 
A general way of checking if all points in $\mathrm{OPT}(L)$ for an LP
$L$ satisfy a
property $P$ is to separately determine the minimum $m^+$ of the objective
function among all feasible 
points of $L$ that satisfy $P$ and the minimum $m^-$ among all
feasible points that don't. Then if $m^+<m^-$, it follows that
all points in $\mathrm{OPT}(L)$ satisfy $P$.
This is because $m^-$ can be attained at a feasible point, and
so the points for which $m$ is $\geq m^+$ are not optimal.
Similarly, if $m^-<m^+$, it follows that no 
points in $\mathrm{OPT}(L)$ satisfy $P$. Finally, if $m^+ = m^-$, then
some points in $\mathrm{OPT}$ satisfy $P$ and other points  do not.

In general, the subset of feasible points that satisfy $P_C$
may not be a convex polytope, so it may not be possible
to use linear programming to determine $m^+$ and $m^-$.
Indeed, the property that we
care about, that is, the existence of indices $i$ and $j$
such that $v_i\leq_\phi v_j$, 
$c_i\geq C$, and $c_j=0$, is not even closed under convex combinations,
let alone expressible as a set of linear inequalities.  For example,
if $v_i \leq_\phi v_j$ and $v_j\leq_\phi v_i$ are two variables of
equal relevance,
and $C_1 = 0.15$, 
then the points $(\ldots,0,\ldots,0.2,\ldots)$ and
$(\ldots,0.2,\ldots,0,\ldots)$ 
(the filled-in entries correspond to coordinates $i$ and $j$)
satisfy the property for $i$ and $j$, but their average
$(\ldots,0.1,\ldots,0.1,\ldots)$ does not. 
However, for fixed $i$ and $j$, the condition that
$c_i\geq C$ and 
$c_j=0$
can be imposed easily on a feasible solution by adding the two
inequalities in question to 
the LP. 
The set of points that satisfy the existentially quantified condition
therefore
can be covered by a $O(n^2)$-sized family of convex polytopes, over
which we can minimise $m$ as a linear program, and determine the overall
minimum $m^+$ by taking the minimum over the individual minima.

\begin{definition}
  For all variables $v_i \neq v_j$ with $v_i \leq_\phi v_j$,
define
$$L^+_{A,i,j}(\phi,C) = L_A(\phi) \cup \{ c_j = 0, c_i \geq C \} $$
  (so, roughly speaking, in solutions to $L^+_{A,i,j}(\phi,C)$,
  variable $v_j$ is ignored while $v_i$ is tested
in a constant fraction of the tests).  
\hfill \wbox
\end{definition}
Clearly, $\bigcup L^+_{A,i,j}(\phi,C) = \{ \vec{p}\in
\Feas(L_A(\phi)) : P_C(\vec{p}) \},$ 
so $\min_{i,j} \mathrm{MIN}(L^+_{A,i,j}) = m^+$.
To determine $m^-$, we need to similarly cover the set of points on which $P_C$
is \emph{not} satisfied with convex polytopes. The negation of $P_C$
is
$$\forall j (\, c_j=0 \Rightarrow \forall i (\, v_i\leq_\phi v_j
\Rightarrow c_i< C)). $$ 
  Given an index $j$, let $I_j=\{j': v_{j'} \leq_\phi v_j \}$.
  Note that the sets $I_j$ are totally ordered by set
  inclusion, since $\leq_\phi$ is a total order.
Let
$$ S^-_{A,i}(\phi,C) = \{ (c_1,\ldots,c_n,m)\in \Feas(L_A(\phi)) :
  c_j<C\text{ for all $j\in I_i$, } c_j>0\text{ for all $j\not\in
    I_i$} \}. $$ 
Intuitively, $v_i$ is the ``last'' variable (in the $\leq_\phi$ 
ordering) such that $c_i = 0$.  Thus, for all $j$ with $v_i \le_\phi
v_j$, we must have $c_j < C$, and for all $j$ with $v_j <_\phi v_i$,
we must have $c_i > 0$.  It is easy to see that
$\bigcup S^-_{A,i}(\phi,C) \supseteq \{ \vec{p}\in \Feas(L_A(\phi)) : \neg
P_C(\vec{p}) \}.$  Unfortunately, the definition of
$S^-_{A,i}(\phi,C)$ involves some strict inequalities, so we cannot
use LP techniques to solve for $m^-$ in the same way as we solved for $m^+$.  

\commentout{
The following lemma will then allow us to work around the problem of the strict inequalities: 

\begin{lemma}
  \label{lem:convexity}
If $P$ is a convex polytope defined by linear inequalities
  $f_1(\v{p})\leq c_1$, $\ldots$, $f_m(\v{p})\leq c_m$ and 
$I \subseteq \{1,\ldots, m\}$, 
then $P$ contains a point for which none of the inequalities
with indices in $I$ are tight if and only if for every $i\in I$, $P$
contains a point for which the inequality $f_i(\v{p}) \leq c_i$ is not
tight. 
\end{lemma}

\begin{proof}
Clearly, if all inequalities in $I$ are not tight
  for $\vec{x}$, then each of them is not tight for
  $\vec{x}$.

  For the converse, suppose that $\v{p}_i$ is a point in $P$ such that
    $f_i(\v{p}_i) \leq c_i$ is not tight; thus, there  exists an
  $\varepsilon_i> 0$
such that $f_i(\v{p}_i) \leq c_i - \varepsilon_i$. 
By convexity, the point
$$\v{p}^* = \frac{ \sum_{i\in I} \v{p}_i }{ |I| }$$
is still in $P$. We claim that none of the inequalities in $I$ is tight for
this point. 
Indeed, for each $i\in I$, we have 
$f_i(\v{p}_i) \leq c_i - \varepsilon_i$.
Moreover, $f_i(\v{p}_j) \leq c_i$ for all $j \in I$, since $x_j \in P$.
So by the linearity of the inequalities, 
$$f_i(\v{p}^*) \leq c_i - \varepsilon_i/|I|,$$
and hence the $i$th inequality isn't tight for $\v{p}^*$, for all $i
\in I$.
This completes the proof.
\end{proof}
}

We are ultimately interested in whether $m^+ < m^-$.  This is the case
if, for all $A$ and $i$, there is no point in $S^-_{A,i}$ such that $m
\le m^+$; that is, if 
$$ T^-_{A,i}(\phi,C,m^+) = \{ (c_1,\ldots,c_n,m) \in S^-_{A,i}(\phi,C)
: m\leq m^+ \}
= \emptyset$$
for all $A$ and $i$.
The set $ T^-_{A,i}(\phi,C,m^+)$ is defined by linear inequalities.
Using a standard trick, we can therefore also determine
whether it is empty using an LP solver: 
\begin{proposition} We
  can decide, in time polynomial in the number of variables and
the number of bits required to describe the inequalities,
 whether a set that is defined by non-strict linear inequalities is empty.
\end{proposition}
\begin{proof}
Take the inequalities defining the set to be $f_1(\vec{x})\leq c_1$, $\ldots$, $f_n(\vec{x})\leq c_n$.
Then the LP 
$$\begin{array}{*2{>{\displaystyle}r}*2{>{\displaystyle}l}}
\text{minimise}  & \,\, s & & \\
\text{subject to} &\,\, f_1(\vec{x}) &\leq c_1 + s& \\
 & &\vdots & \\
 &\,\, f_n(\vec{x}) &\leq c_n + s &
\end{array}$$
has a solution $s\leq 0$ iff the set is nonempty:
If the LP has a solution $s\leq 0$ 
then the solution point $\vec{x}^*$ also satisfies the inequalities
defining the LP; 
conversely, a solution to the inqualities satisfies the LP with $s=0$.
This LP has one more variable than the original set of inequalities, and
clearly can be described using at most a polynomially greater number
of bits than 
the original under any reasonable encoding. The result follows by using 
Karmarkar's algorithm \cite{karmarkar84}.
\end{proof}

\commentout{
\begin{definition}
The inequality $f(\v{x}) \leq c$ is \emph{tight for $\v{x}$} if
  $f(\v{x})=c$. 
\hfill wbox
\end{definition}

Lemma~\ref{lem:convexity} enables us to find LP solutions with a
minimal number of 
tight inequalities
by solving $m$ copies of the LP, as follows:

\begin{corollary} If $L$ is an LP with inequalities $f_1(\v{x}) \leq
  c_1$, $\ldots$, $f_m(\v{x}) \leq c_m$ and
  objective $g(\v{x})$,  
let $L_i$ be $L$ with the $i$th inequality replaced by $f_i(\v{x})
\leq c_i - \varepsilon$, and set 
$$ k = \# \{ i : \mathrm{MIN}(L_i) = \mathrm{MIN}(L) \}. $$
then all solutions to $L$ have at least $k$ tight inequalities and
there is a solution that has exactly $k$ tight inequalities. 
\end{corollary}

\begin{lemma} \label{lem:seqtolpzeroes}
    If $\mathrm{OPT}(L_A(\phi))$ contains
  a point 
that is zero in some coordinate and minimises the number of tight
inequalities over $\mathrm{OPT}(L_A(\phi))$, 
then there exists a function $\delta': \IN\rightarrow \IR$, depending
only on 
$\phi$, $D$, and $\m$,
  such that 
\begin{itemize}
\item  $\lim_{k\rightarrow
    \infty} \delta'(k) = 0$ and 
\item for all optimal test-outcome sequences $S$ supporting $A$,
the $A$-trace of $S$ is 
within $\delta'(|S|)$ of some solution $\vec{d} \in \mathrm{OPT}(L_A(\phi))$ \emph{that is zero in some coordinate},
that is,
$$ \exists \vec{d}\in OPT(L_A(\phi)) (\| \vec{d}- {\Tr}_A(S) \|_1 <
\delta'(|S|)) \text{ and } \exists i (\vec{d}_i=0). $$ 
\end{itemize}
Moreover, $L_A$ attains the minimax power: $\mathrm{MIN}(L_A(\phi))=\mathrm{MIN}^*(\phi)$.
\end{lemma}
\begin{proof}
    Let $\vec{d}$ be a point of $\mathrm{OPT}(L_A(\phi))$ that is zero in some
coordinate and minimises the number of tight inequalities. 
We show that for all $\delta>0$, there exists a $k_\delta$ such
that if $|S|>k_\delta$, ${\Tr}_A(S)$ is within $\delta$ 
of some point $\vec{c} \in \mathrm{OPT}(L_A(\phi))$ that is nonzero in all coordinates and not within $\delta$ of $\vec{d}$, then $|S|$ can
not be optimal. Then done by chaining with Lemma \ref{lem:seqtolpzeroes} and constructing $\delta'$ from $k_\delta$ like the same lemma
constructed $\delta$ from $k_\epsilon$.

Minimising number tight inequalities $=$ set of tight inequalities is $\subseteq$-minimal! (corollary of preceding lemma)

Key part mirrors proof strategy of previous lemma: given constants $x$ and $y$, show that for sufficiently large $k$,
$$\begin{array}{lll}
\frac{ \Pr(B) o^{mk} + 2^n o^{m-y} k }{ \Pr(A) o^k + \Pr(A') o^k } &<& \frac{ \Pr(B) o^{mk} }{ \Pr(A) o^k + 2^n o^{(1-x)k} }
\end{array}$$
by constructing a test-outcome sequence approximating $\vec{d}$ and comparing to $S$. LHS gets an extra term in denominator $\Pr(A') o^k$ due to an assignment $A' \neq A$: 
if we ignore one variable (corresponding coordinate is zero) completely, then the sequence also supports $A'$!
\end{proof}

\begin{theorem} \label{thm:lptorihard}
    If there exists a constant $C$ such that for all truth assignments $A$
for which $\mathrm{MIN}(L_A(\phi))=\mathrm{MIN}^*(\phi)$ and all solution
vectors $\vec{c}=(c_1,\ldots,c_n,m)\in \mathrm{OPT}(L_A(\phi))$,
\begin{itemize}
\item \emph{either} there exist indices $i$ and $j$ such that $c_i>C$ and $c_j=0$
\item \emph{or} $\forall i. c_i>0$
\end{itemize}
and moreover some point of the former type minimises the number of tight inequalities over $\mathrm{OPT}(L_A(\phi))$, 
then $\phi$ exhibits RI.
\end{theorem}
\begin{proof}
Analogous to the proof of Theorem \ref{thm:lptorieasy}, but using Lemma \ref{lem:seqtolpzeroes} in place of Lemma \ref{lem:seqtolp}. 
\end{proof}

\begin{theorem} Let $m^-$ be the minimum of minima of all inattentive
    LPs, let $t^-$ the minimum number of tight 
inequalities across all inattentive LPs that attain $m^-$.
let $m^+$ be the minimum of minima of all attentive LPs, and
let $t^+$ be the minimum number of tight inequalities 
across all attentive LPs that attain $m^+$.
Then each of the following is a sufficient condition for $\phi$ to
exhibit RI:
\begin{enumerate}[(a)]
\item $m^-< m^+$.
\item $m^-= m^+$, $t^- = t^+$ and $m^+$ is not attained by any $L^+_{A,i}$ with $i>0$.
\end{enumerate}
\end{theorem}
} %

\begin{theorem}
  Fix $C > 0$ and set $m^+_C =
  \min_{A,i,j}\mathrm{MIN}(L^+_{A,i,j}(\phi,C))$.
  If $T^-_{A,i}(\phi,C,m^+_C) = \emptyset$ for all $A$ and $i$, 
then $\phi$ exhibits rational inattention.
\end{theorem}
\begin{proof}
As explained above, 
the sets $T^-_{A,i}$ being empty implies that there is no point
satisfying $\neg P_C$ and attaining a max-power of $m\leq m^+_C$.  
At the same time, $m^+_C$ being the minimum over all inattentive LPs
means that the minimum of $m$ over points satisfying $P_C$ in any
$L_A$ is $m^+_C$. Therefore, $m^+_C$ is the minimax power, and all
solution points of relevant LPs satisfy
$P_C$. Hence, by Theorem $\ref{thm:lptorieasy}$,  $\phi$ exhibits RI.
\end{proof}

\begin{corollary} We can compute a sufficient condition for the $n$-variable formula $\phi$ to exhibit RI
by solving $2^n O(n^2)$ LPs with $O(2^n)$ inequalities each, namely the $O(n^2)$ inattentive LPs
and the $O(n)$ attentive LPs associated with each of the $2^n$ assignments.
\end{corollary}

\section{Proof of Theorem~\ref{thm:xor}}
\label{sec:proofxorhardest}
\commentout{
    \begin{example} ($q,\m$-test complexity)
    \begin{itemize}
    \item The one-variable formula $\varphi=x$ has one- and two-sided $q,\m$-test complexity $\lceil \log_{\frac{1}{2}-\m} (\frac{1}{2}-q) \rceil$ for all $q,\m>0$: in the best case, any sequence of $k$ tests comes back the same (either $x\mt\t$ or $x\mt\f$), and so the probability of the variable having the opposite truth value becomes $(\frac{1}{2}-\alpha_x)^k$.
    \item The example in the introduction showed that the one-sided $\frac{3}{8},\frac{1}{4}$-test complexity of $\ell_1\vee \ell_2$ is 2.
    \item Among all formulae with $n$ variables, for all $q,\m>0$, the big XOR $\ell_1 \oplus \ldots \oplus \ell_n$ has greatest $q,\m$-test complexity. We will see a proof for this in Section 4.
    \item Some more concrete values:
    \begin{itemize}
    \item The majority function on 5 inputs $\mathrm{Maj}_5$ has $\frac{1}{3},\frac{1}{4}$-test complexity 4, and the optimal test-outcome sequences consist either of some four distinct variables being measured $\t$ or some four distinct variables being measured $\f$. These sequences in fact each give an absolute bias of $\frac{11}{32}$.
    \item The 5-variable XOR $a\oplus b\oplus c\oplus d\oplus e$ has
      $\frac{1}{3},\frac{1}{4}$-test complexity 15, and optimal
            sequences test each variable thrice with either outcome. 
    \item The formula $(a\oplus b)\vee (c\oplus d)\vee e$ has $\frac{1}{3},\frac{1}{4}$-test complexity 7, and the optimal sequences test either $a\mt\t$ or $b\mt\t$ twice, either $c\mt\t$ or $d\mt\t$ twice, and $e\mt\t$ or $e\mt\f$ thrice. On the other hand, the 3-variable XOR $a\oplus c\oplus e$ has $\frac{1}{3},\frac{1}{4}$-test complexity 8.
      \hbox
      \wbox
    \end{itemize}
    \end{itemize}
    \end{example}

    Among the above examples, XOR stands out as having the highest
    consistency. We shall aim to prove that this is true in general: 
    \begin{theorem}
    Among all formulae with $n$ variables, for all $q,\m>0$ (and all distributions $D$) the big XOR $\ell_1 \oplus \ldots \oplus \ell_n$ and its negation have greatest two-sided $q,\m$-test complexity.
    \end{theorem}
}

\commentout{
To prove Theorem~\ref{thm:xor}, we first need a 
definition and some lemmas.

\begin{definition}
Given a probability on some space $S$ and an event $E \subseteq S$,
define the \emph{bias} of $E$ to be 
$$Q(E) = \Pr(E)-\frac{1}{2}.$$
\hfill
\wbox
\end{definition}

\commentout{
    \begin{definition} We call a sequence of test outcomes $S$ \emph{unequivocal} about a truth assignment $v_1\mapsto b_1$, $\ldots$, $v_n\mapsto b_n$ if no tests of $v_i$ in $M$ are $v_i\mt \overline{b_i}$.
    \end{definition}
}

The following lemma will be helpful in establishing the complexity of
formulae.  

\begin{lemma}\label{pointwisetocpl} If, for all test-outcome sequences
  $S$, there exists a test-outcome sequence $S'$ such that $|S'|=|S|$
  and $$|\Q(\phi|S')|\geq |\Q(\psi|S)|,$$ %
then $$\cpl(\phi)\leq \cpl(\psi).$$
\end{lemma}
\begin{proof}
  Suppose that $\cpl(\psi) =k$. Then there must be some
  strategy $\sigma$ for $G(\psi,D,k,\m,g,b)$
  that
 has positive expected payoff.
Analogously to the argument in the proof of Proposition
\ref{prop:randvars}, there must therefore be some 
test-outcome sequence $S$ of length $k$ that is observed by $\sigma$ with
positive probability   
such that the expected payoff of making the appropriate guess is
positive. By Lemma \ref{lem:payoff}, $|\Q(\psi\mid S)|>q$.

Since $|\Q(\phi\mid S')|\geq |\Q(\psi\mid S)|$ by assumption, there must exist a
test-outcome sequence $S'$ such 
$|\Q(\phi\mid S')|>q$.  Let $\sigma'$ be the strategy for the
game $G(\phi,D,k,\m,g,b)$
that tests the same variables that are tested in $S'$, and makes the
appropriate guess
iff $S'$ is in fact observed.
By Lemma \ref{lem:payoff}, a guess with positive expected payoff can
be made if $S'$ is observed, which it is with positive probability. So 
$\sigma'$ has positive expected payoff, and hence $\cpl(\phi)$
is at most $k$. 
\end{proof}

\commentout{
    \begin{proposition}\label{permuteseq}Permutations of test
      sequences don't change probability conditioned on them
      (formalise).\end{proposition} 

    Even if permutations can be ignored, the number of degrees of
        freedom to choose a sequence of test outcomes still may appear
    daunting. Conveniently, however, it turns out that contradictory
    tests of a variable simply cancel out, so we can restrict our
    considerations to whether and by how much a test-outcome sequence
    supports each variable being true (false). 

    \begin{lemma} \label{cancel} Suppose $M$ is an arbitrary sequence
            of test outcomes. Then 
    $$\Q( v_i\mid  v_i:\overline{ v_i}:M)=\Q( v_i\mid M).$$
    \end{lemma}
    \begin{proof}
    Sufficient to show $\Pr( v_i|[ v_i\mt b_i,v_i \mt \bar b_i])=\Pr( v_i)$ (conditioning removes dependencies). By Bayes,
    \begin{eqnarray*}
    \Pr( v_i|[ v_i,\bar v_i]) &=& \Pr([ v_i,\bar v_i]| v_i)\Pr( v_i) / \Pr([ v_i,\bar v_i]) \\
     &=& \frac{ \Pr([ v_i,\bar v_i]| v_i) }{\Pr([ v_i,\bar v_i]| v_i)\Pr( v_i) + \Pr([ v_i,\bar v_i]|\bar v_i)\Pr(\bar v_i)} \Pr( v_i) \\
     &=& \frac{ (\frac{1}{2}+ \alpha_i)(\frac{1}{2}- \alpha_i) }{ (\frac{1}{2}+
      \alpha_i)(\frac{1}{2}- \alpha_i)p + (\frac{1}{2}+ \alpha_i)(\frac{1}{2}-

      \alpha_i)(1-p) } \Pr( v_i) \\ 
     &=& \frac{ (\frac{1}{2}+ \alpha_i)(\frac{1}{2}- \alpha_i) }{(\frac{1}{2}+
      \alpha_i)(\frac{1}{2}- \alpha_i)} \Pr( v_i) = \Pr( v_i).
    \end{eqnarray*}
    \end{proof}
}

\commentout{
As we mostly seek to trade off tests of one variable against tests of
another, we need to be able to compare effects across variables as
well. Luckily, their independence means that a given number of tests
of any variable has the same effect on the bias of the respective
variable:
}

The following consequence of the definition of a product distribution
will be very useful in this chapter.
\begin{lemma} \label{lem:product} For all product distributions $D$,
    accuracy vectors $\vec{\alpha}$, assignments $A$, and test-outcome
  sequences $S$, 
$$ \Pr(A\mid S) = \prod_{i=1}^n \Pr(v_i = A(v_i) \mid S). $$
\end{lemma}
\begin{proof}
The event $A$ is the intersection of the $n$ events $v_i = A(v_i)$, and
since $D$ is a product distribution, the events are independent of each other.
\end{proof}
}

We previously took the XOR $v_1 \oplus \ldots
\oplus v_n$ of $n$ variables (often denoted $\bigoplus_{i=1}^n v_i$)
to be true iff an odd number of the variables are true. This
characterisation is actually a %
consequence of 
the following standard definition in terms of basic Boolean
connectives, of which we also note some 
useful properties (whose proof is left to the reader).
\begin{definition} \label{xordef} The \emph{exclusive OR} (XOR) $\phi_1 \oplus \phi_2$ is equivalent
to the formula $(\phi_1 \wedge \neg
  \phi_2) \vee (\neg \phi_1 \wedge \phi_2)$. 
  \wbox
\end{definition}
\begin{proposition} \label{xorprops} (Properties of XOR)
\begin{enumerate}
\item[(a)] XOR is commutative: $\phi_1 \oplus \phi_2 \equiv \phi_2 \oplus \phi_1$;
\item[(b)] XOR is associative: $(\phi_1 \oplus \phi_2) \oplus \phi_3 \equiv \phi_1 \oplus (\phi_2 \oplus \phi_3)$;
  \item[(c)] $ v_1 \oplus \ldots \oplus
  v_n$ is true iff an odd number of the variables $v_i$ is; 
  \item[(d)] 
  $\neg\phi \equiv T \oplus \phi$, so $\phi_1 \oplus \neg \phi_2
    \equiv \neg \phi_1 \oplus \phi_2 \equiv \neg (\phi_1 \oplus
    \phi_2). $ 
\end{enumerate}
\end{proposition}
%
\if 0
\begin{proof}
%
Commutativity is immediate from the symmetry of the definition.
To see associativity, note that by applying De Morgan's laws,
%
%
%
both $(\phi_1 \oplus \phi_2)
\oplus \phi_3$ and $\phi_1 \oplus (\phi_2 \oplus \phi_3)$ expand to  
$$ (\phi_1 \wedge \neg \phi_2 \wedge \neg \phi_3) \vee (\phi_1 \wedge \phi_2 \wedge \phi_3) \vee (\neg \phi_1 \wedge \phi_2 \wedge \neg \phi_3) \vee (\neg\phi_1 \wedge \neg\phi_2 \wedge \phi_3). $$
%
%
Part (c) follows by induction: it is true for $n=1$ as $v_1$ is true iff
exactly one variable in $\{v_1\}$ is, and $v_1 \oplus \phi$ is true
%
%
%
%
iff either $v_1$ is true and $\phi$ is false (so, if $\phi$ is an XOR of
$n-1$ variables, an even number of the variables in $\phi$ are true, and
hence an odd number of variables including $v_1$ are true), or $v_1$ is
%
%
false and $\phi$ is true (so an odd number of variables in $\phi$ are
true). 
%
%
For part (d), the fact that $\neg \phi \equiv T \oplus \phi$ 
follows by plugging $T$ into the definition (and noting that
the second term vanishes, as $(F \wedge \phi_2) \equiv F$);
the fact that $\neg \phi_1 \oplus \phi_2 \equiv \neg (\phi_1 \oplus
\phi_2)$ follows by replacing $\neg \phi_1$ by $T \oplus \phi_1$ and
applying associativity; similarly, the fact that $\phi_1 \oplus \neg
\phi_2 \equiv \neg (\phi_1 \oplus \phi_2)$ follows by replacing
$\phi_2$ by $T \oplus \phi_2$ 
and applying commutativity and associativity. 
\end{proof}
\fi

%
%
%
%
%
%
%
%
%
%
%
%
%
%
%
%
%
%
%
%
%
%
%
%
%
%
%
%
%
%
%
%
%
%
%
%
%
%
%
%
%
%
%
%
%
%
%
%
%
%
%
%
%
%
%
%
%

%
%
%
%
%
%
%
%

%
%
\commentout{
\begin{lemma} \label{lem:projiscond}
Suppose $D$ is a product distribution. Then, given a sequence of test outcomes $S$,
the projection of a formula $\phi\subst{v_i}{b}$ has the same
conditional probability on $S$ as $\phi$ additionally conditioned on
$v_i=b$, that is: 
$$ \Pr(\phi \mid S,v_i=b) = \Pr(\phi\subst{v_i}{b}\mid S). $$ 
\end{lemma}
\begin{proof}
  We have
\begin{eqnarray*}
\Pr(\phi\subst{v_i}{b}\mid S) &=& \sum_{\{A:\;\phi\subst{v_i}{b}(A)=T\}} \Pr(A\mid S) \\
 &\overset{\ref{lem:product}}{=}& \sum_{\{A:\; \phi\subst{v_i}{b}(A)=T\}} \prod_{1\leq j\leq n} \Pr(v_j=A(v_j)\mid S). 
\end{eqnarray*}
Expanding the definition of the projection, this is
\begin{eqnarray*}
 &=& \sum_{\{A:\;\phi(A)=T, A(v_i)=b\}} \prod_{1\leq j\leq n} \Pr(v_j=A(v_j)\mid S)
    +\sum_{\{A:\;\phi(A[v_i\mapsto b])=T, A(v_i)=\neg b\}} \prod_{1\leq j\leq n} \Pr(v_j=A(v_j)\mid S) \\
 &=& \sum_{\{A:\;\phi(A)=T, A(v_i)=b\}} \left( \prod_{1\leq j\leq n} \Pr(v_j=A(v_j)\mid S) + \prod_{1\leq j\leq n} \Pr(v_j=A[v_i\mapsto \neg b](v_j)\mid S) \right) \\
 &=& \sum_{\{A:\; \phi(A)=T, A(v_i)=b\}} (\Pr(v_i=b\mid S)+\Pr(v_i=\neg b\mid S)) \prod_{1\leq j\leq n, j\neq i} \Pr(v_j=A(v_j)\mid S) \\
 &=& \sum_{\{A:\; \phi(A)=T, A(v_i)=b\}} \prod_{1\leq j\leq n, j\neq i} \Pr(v_j=A(v_j)\mid S) \\
 &=& \Pr(\phi, v_i=b \mid S)/\Pr(v_i=b \mid S) \\
 &=& \Pr(\phi\mid S, v_i=b) \text{ (Bayes).}
\end{eqnarray*}
\end{proof}
}
As we said in the proof sketch in the main text, our proof uses the
idea of antisymmetry.  

The notion of antisymmetry has the useful property
that $\phi_v$, the antisymmetrisation of $\phi$ along $v$ (recall
that $\phi_v$ was defined as $(v\wedge \phi\subst{v}{\t})\vee (\neg v\wedge
\neg \phi\subst{v}{\t})$)  is antisymmetric in $v$ and, as we now
show, also antisymmetric in all other variables $v'$ that $\phi$ was
antisymmetric in.

\begin{lemma}\label{antisymmetrise}
  If $\phi$ is antisymmetric in a variable
$v' \ne v$, then so is $\phi_v$.
\end{lemma}
\begin{proof}
Suppose that $\phi$ is antisymmetric in $v' \ne v$.  Then for all truth
assignments $A$, we have 
\begin{itemize}
  \item $\phi(A[v \mapsto \t]) = \neg{\phi(A[v'\mapsto \f])}$ and 
\item $\phi_v(A) = \begin{cases} \phi(A[v\mapsto\t]) & \text{ if $A(v)=\t$}
    \\ \neg{\phi(A[v\mapsto\t])} & \text{ if $A(v)=\f$.} \end{cases} $
\end{itemize}
Thus,
if $A(v) = \t$, then 
$$\begin{array}{lll}
    \phi_v(A[v'\mapsto\t]) &=&\phi(A[v\mapsto\t,v'\mapsto\t]) \\
 &=& \neg{\phi(A[v\mapsto\t,v'\mapsto\f])} \\
        &=& \neg{\phi_v(v'\mapsto\f])},
      \end{array}$$
and
if $A(v) = \f$, then 
$$\begin{array}{lll}
  \phi_v(A[v'\mapsto\t]) &=& \neg{\phi(A[v\mapsto\t,v'\mapsto\t])} \\
 &=& \phi(A[v\mapsto\t,v'\mapsto\f]) \\
    &=& \neg{\phi_v(A[v'\mapsto\f])}.
  \end{array}$$
Thus, no matter what $A(v)$ is, we have
$\phi_v(A[v'\mapsto\t]) = \neg{\phi_v(A[v'\mapsto\f])}$, as required. 
\end{proof}

Define $V(\phi)$, the number of variables a formula $\phi$ is \emph{not}
antisymmetric in, as 
$$V(\phi) = |\{ v : \phi \not\equiv
(v\wedge \phi\subst{v}{\t})\vee (\neg v\wedge \neg \phi\subst{v}{\t})\}|.$$ 

\begin{lemma}\label{xorsonly}
 The only formulae $\phi$ in the $n$ variables $v_1, \ldots, v_n$ for
  which $V(\phi)=0$ are  equivalent to either
    $\bigoplus_{i=1}^n v_i$ or $\neg \bigoplus_{i=1}^n v$. 
\end{lemma}
\begin{proof}
  By induction on $n$. If $n=1$, then it is easy to check
    that both $v_1$ and $\neg v_1$ are antisymmetric. Suppose that $n > 1$
and $\phi$ is antisymmetric in $v_1, \ldots, v_n$.  
Since $\phi \equiv (v_n\wedge \phi\subst{v_n}{\t})\vee (\neg v_n\wedge
\phi\subst{v_n}{\f})$ and $\phi$ is 
antisymmetric in $v_n$, by 
Definition~\ref{xordef}
 we have that
\begin{equation}\label{eq12}
  \phi\equiv (v_n\wedge \phi\subst{v_n}{\t})\vee (\neg v_n\wedge \neg
\phi\subst{v_n}{\t}) \equiv v_n \oplus 
\phi\subst{v_n}{\t}.
\end{equation}
It is easy to see that $\phi\subst{v_n}{\t}$ mentions only the variables
$v_1, \ldots, v_{n-1}$ and, by Lemma \ref{antisymmetrise}, is
antisymmetric in each of them.  
So by the induction hypothesis, 
$\phi\subst{v_n}{\t}$ is equivalent to either $\bigoplus_{i=1}^{n-1} v_i$ or
$\neg (\bigoplus_{i=1}^{n-1} v_i)$, and hence
by Proposition \ref{xorprops}(d) and (\ref{eq12}),
$\phi$ is equivalent to either $\bigoplus_{i=1}^{n} v_i$ or
$\neg (\bigoplus_{i=1}^{n} v_i)$.
\end{proof}

To complete the proof of Theorem~\ref{thm:xor}, we make use of
the following two technical lemmas.
For the remainder of the proof, we use $v=\t$ and $v=\f$ to denote the
events (i.e., the set of histories) where the variable $v$ is true
(resp., false).  (We earlier denoted these events $v$ and $\neg v$,
respectively, but for this proof the $v=b$ notation is more convenient.)

\begin{lemma} \label{lem:projiscond}
If $D$ is a product distribution and $S$ is a test-outcome sequence,
then
the projection of a formula $\phi\subst{v_i}{b}$ has the same
conditional probability on $S$ as $\phi$ additionally conditioned on
$v_i=b$, that is, 
$$ \Pr(\phi \mid S,v_i=b) = \Pr(\phi\subst{v_i}{b}\mid S). $$ 
\end{lemma}
\begin{proof}
Given a truth assignment $A$ on $v_1, \ldots, v_n$, let $A_i$ be $A$
restricted to all the variables other than $v_i$.  Since $D$ is a
product distribution, 
$\Pr(A)=\Pr(A_i) \times \Pr(v_i=A(v_i)).$

Note that the truth of $\phi\subst{v_i}{b}$ does not depend on
the truth value of $v_i$.  Thus, we can pair the truth assignment that
make $\phi\subst{v_i}{b}$ true into groups of two, that differ only in
the truth assignment to $v_i$.  
Suppose that the test $v_i \mt \t$ appears in $S$ $k_T$ times and the
test $v_i \mt \f$ appears in $S$ $k_F$ times.  
Using Lemma~\ref{lem:denormtoprob}, we have that
\begin{equation}\label{eq:condeq1}
    \begin{array}{lll}
  \Pr(\phi\subst{v_i}{b}\mid S) &=&
  \sum_{\{A:\; \phi\subst{v_i}{b}(A)=T\}} \Pr(A\mid S) \\ 
  &=& \frac{\sum_{\{A:\; \phi\subst{v_i}{b}(A)=T\}}  \o{A}{S} }{ \sum_{\text{truth assignments
      }A'}  \o{A'}{S} }\\
    &=& \frac{ \sum_{\{A:\; \phi\subst{v_i}{b}(A)=T\}} \o{A_i}{S}(\Pr(v_i =
  T)o^{k_T} + \Pr(v_i = F)o^{k_F})}{ \sum_{\text{truth assignments
      }A'}  \o{A'_i}{S}(\Pr(v_i = T)o^{k_T} + \Pr(v_i = F)o^{k_F}
)} \\
  &=& \frac{\sum_{\{A:\; \phi\subst{v_i}{b}(A)=T\}}  \o{A_i}{S}}{
        \sum_{\text{truth assignments 
      }A'}  \o{A'_i}{S}. }
\end{array}\end{equation}
  Using the same arguments as in (\ref{eq:condeq1}), we get that
$$    \Pr(\phi \land v_i = b\mid S) =  
  \frac{\sum_{\{A:\; (\phi \land v_i=b)(A)=T\}}  \o{A_i}{S} \Pr(v_i=b)o^{k_T}}{
  \sum_{\text{truth assignments }A'}  \o{A'}{S} }$$
and
    $$    \Pr(v_i = b\mid S) =  
\frac{\sum_{\{A:\; A(v_i)=b\}}  \o{A_i}{S}  \Pr(v_i=b)o^{k_T}}{
  \sum_{\text{truth assignments }A'}  \o{A'}{S}. }$$

Let $C=\Pr(v_i=b\mid S)=\frac{\Pr(v_i=b)k_b}{\Pr(v_i=T)k_T +
  \Pr(v_i=F)k_F}$ be the  
probability that $v_i=b$ \emph{after} observing the sequence. 
Note that
$$\sum_{\mathclap{\{A:\; (\phi \land v_i=b)(A)=T\}}}  \o{A_i}{S} \hspace{1.5em} = \hspace{1.5em} 
    \sum_{\mathclap{\{A:\; (\phi\subst{v_i}{b} \land v_i=b)(A)=T\}}}  \o{A_i}{S} \hspace{1em} =\hspace{1em}
C\cdot
\sum_{\mathclap{\{A:\phi\subst{v_i}{b}(A)=T\}}}  \o{A_i}{S}$$
    and
        $$\sum_{\mathclap{\{A:\; A(v_i)=b\}}}  \o{A_i}{S} \hspace{1em} = \hspace{1em} %
C\cdot 
    \sum_{\mathclap{\textrm{truth assignments } A}}  \o{A_i}{S}.$$
Since, by Bayes' Rule, 
$$  \Pr(\phi \mid S, v_i = b) =
\frac{ \Pr(\phi \land v_i = b\mid S)}{ \Pr(v_i = b \mid S)},$$
simple algebra shows that 
$  \Pr(\phi \mid S, v_i = b) =   \Pr(\phi\subst{v_i}{b}\mid S),$ as desired.
\end{proof}    

\begin{lemma}\label{pointwisetocpl} If, for all test-outcome sequences
  $S$, there exists a test-outcome sequence $S'$ such that $|S'|=|S|$
  and $|\Pr(\phi\mid S') - 1/2|\geq |\Pr(\psi\mid S) -1/2|$,
then $\cpl(\phi)\leq \cpl(\psi)$.
\end{lemma}
\begin{proof}
  Suppose that $\cpl(\psi) =k$. Then there must be some
  strategy $\sigma$ for $G(\psi,D,k,\m,g,b)$
  that  has positive expected payoff.
There must therefore be some 
test-outcome sequence $S$ of length $k$ that is observed with positive
probability when using $\sigma$ 
such that the expected payoff of making the appropriate guess is
positive. By Lemma \ref{lem:payoff}, $|\Pr(\psi\mid S) -1/2|>q$.

Since $|\Pr(\phi\mid S') -1/2|\geq |\Pr(\psi\mid S)-1/2|$ by assumption, there
must exist a test-outcome sequence $S'$ such 
$|\Pr(\phi\mid S') - 1/2|>q$.  Let $\sigma'$ be the strategy for the
game $G(\phi,D,k,\m,g,b)$
that tests the same variables that are tested in $S'$, and makes the
appropriate guess
iff $S'$ is in fact observed.
By Lemma \ref{lem:payoff}, a guess with positive expected payoff can
be made if $S'$ is observed, which it is with positive probability. So 
$\sigma'$ has positive expected payoff, and hence $\cpl(\phi)$
is at most $k$. 
\end{proof}

We can now finally prove Theorem~\ref{thm:xor}.
Note that this is the only part of the derivation that actually depends
on the assumption that we are working with the uniform distribution $D_u$.

\begin{proof}[Proof of Theorem~\ref{thm:xor}]
    We show by induction on $V(\phi)$ that for all formulae $\phi$,
there exists a formula $\phi_0$ with $V(\phi_0)=0$ such that $\cpl
\phi \leq \cpl \phi_0$. By Lemma \ref{xorsonly}, $\phi_0$ must
be equivalent to either $\oplus_{i=1}^{n-1} v_i$ or
$\neg (\oplus_{i=1}^{n-1} v_i)$.
If $V(\phi) = 0$, then we can just take $\phi_0 = \phi$.  Now suppose that 
$V(\phi)>0$.  There there must exist some variable $v$ such that
$\phi\subst{v}{\t}\neq \neg (\phi\subst{v}{\f})$.
(Here and below we are viewing formulas as functions on truth
assignments, justifying the use of ``='' rather than ``$\equiv$''.)
Note for future reference that,
by construction,
\begin{equation}\label{eq0}
  \phi_v\subst{v}{\t} =
\phi\subst{v}{\t} \mbox{ and } \phi_v\subst{v}{\f} =
\neg \phi\subst{v}{\t}.
\end{equation}
By Lemma~\ref{antisymmetrise}, if $\phi$ is antisymmetric in a
variable $v' \ne v$, then so is $\phi_v$.  In addition, $\phi_v$ is
antisymmetric in $v$.  Thus, $V(\phi_v) < V(\phi)$.
If we can show $\cpl(\phi) \le \cpl(\phi_v)$, then the result
follows from the induction hypothesis.
By Lemma \ref{pointwisetocpl}, it suffices to show that for 
all test-outcome sequences  $S_1$, there exists a sequence $S$ of the same
length as $S_1$ such that 
$|\Pr(\phi\mid S)-1/2| \geq |\Pr(\phi_v\mid S_1)-1/2|$.

Given an arbitrary test-outcome sequence $S_1$, 
let $p = \Pr(v = \t \mid S_1)$. %
Thus,
\begin{equation}\label{phim0}
    \begin{array}{llll}
  \Pr(\phi_v \mid S_1)
  &=& p \Pr(\phi_v \mid S_1, v=\t) + (1-p) \Pr(\phi_v \mid S_1, v=\f)\\
  &=& p \Pr(\phi_v\subst{v}{\t} \mid
  S_1) + (1-p) \Pr( 
\phi_v\subst{v}{\f} \mid S_1) &\mbox{[by Lemma~\ref{lem:projiscond}]}\\
  &=& p \Pr(\phi\subst{v}{\t} \mid S_1) + (1-p) \Pr(
\neg\phi\subst{v}{\t} \mid S_1) &\mbox{[by (\ref{eq0})]}\\
 &=& p \Pr(\phi\subst{v}{\t} \mid S_1) + (1-p)(1- \Pr(
  \phi\subst{v}{\t} \mid S_1).
\end{array}
  \end{equation}
Set %
$S_2 = S_1[v\mt \f \leftrightarrow v\mt \t]$,
that is, the sequence that is the same as $S_1$ except that all test
outcomes of $v$ are flipped in value. 
Since $\phi\subst{v}{\t}$ does not mention $v$, $\Pr(\phi\subst{v}{\t}
\mid S_1) 
=\Pr(\phi\subst{v}{\t}\mid S_2)$ and likewise for
$\phi\subst{v}{\f}$. Since $\phi\equiv
(v\wedge \phi\subst{v}{\t})\vee (\neg v\wedge \phi\subst{v}{\f})$, we
have (using an argument 
similar to that above)
\begin{eqnarray}\label{phim1} 
  \Pr(\phi\mid S_1)
   &=&
  p \Pr(\phi\subst{v}{\t} \mid S_1) + (1-p)\Pr(\phi\subst{v}{\f} \mid   S_1)  
\end{eqnarray}
and, taking $p' = \Pr(v=T \mid S_2)$, 
\begin{equation}\label{phim3}
\begin{array}{lll}
  \Pr(\phi\mid S_2)
  &=& p' \Pr(\phi\subst{v}{\t} \mid S_2) + (1-p') \Pr(\phi\subst{v}{\f} \mid
  S_2)\\
  &=& p'\Pr(\phi\subst{v}{\t} \mid S_1) + (1-p') \Pr(\phi\subst{v}{\f} \mid
  S_1).
\end{array}
\end{equation}

We claim that $p = 1-p'$.
Suppose that the test $v \mt \t$ appears
in $S_1$ $k_T$ times and the  test $v \mt \f$ appears in $S_1$ $k_F$ times.  
Thus, the test  $v \mt \t$ appears
in $S_2$ $k_F$ times and the  test $v \mt \f$ appears in $S_1$ $k_T$ times. 
All other tests appear the same number of times in both sequences.
By Lemma \ref{lem:denormtoprob},  
since the uniform distribution $D_u$ we are using is in particular a product
distribution,
for $j=1,2$, we have that
$$
\Pr(v=T\mid S_j) = \sum_{\{A:\; A(v)=\t\}} \Pr(A \mid S_j) = \frac{
  \sum_{\{A:\; A(v)=\t\}} \o{A}{S_j} }{ \sum_{A'} \o{A'}{S_j} }. $$
Suppose that $v$ is the $i$th variable $v_i$. Let
$r_1 = o_i^{k_T}$, let $r_2 = o_i^{k_F}$, let 
$ R_1 = \sum_{\{A:\; A(v_i)=T\}} \prod_{j=1, j\neq i}^n
o_j^{n^+_{S_1,A,j}} $, and let
$ R_2 = \sum_{\{A:\; A(v_i)=F\}} \prod_{j=1, j\neq i}^n o_j^{n^+_{S_1,A,j}} $.
For $j = 1,2$ we have that
$$\sum_{\{A:\; A(v)=\t\}} \Pr(A \mid S_j) = \frac{
    \sum_{\{A:\; A(v)=\t\}} \o{A}{S_j} }{ \sum_{A'} \o{A'}{S_j} }
  = \frac{ r_j R_j }{ r_1 R_1  + r_2 R_2 }$$
We claim that $R_1=R_2$. Indeed, for any assignment $A$
  such that $A(v_i)=T$, let $A'$ be the unique assignment such that $A'(v_i)=F$
and $A'(v_j)=A(v_j)$ for all $j\neq i$. Then each choice of $A$ occurs
once in the 
sum $R_1$ and never in the sum $R_2$, the corresponding $A'$ occurs once in
$R_2$ but not $R_1$.  Since we are working with the uniform
distribution $D_u$, the summands for $A$ and $A'$ are equal.
So we can conclude that $p=1-p'$. Combining this with
(\ref{phim3}), we get that
\begin{equation}\label{phim2} 
\begin{array}{lll}
  \Pr(\phi\mid S_2)
    &=& (1-p) \Pr(\phi\subst{v}{\t} \mid S_1) + p \Pr(\phi\subst{v}{\f} \mid
  S_1).
\end{array}
\end{equation}

\commentout{
Suppose $v$ is the $i$th variable $v_i$. 
Let $A$ be an arbitrary assignment such that $A(v_i)=T$ and $A'$ such that $A(v_i)=F$.
Then $n^+_{S,A,i}$ is independent of the choice of $A$ for all $S$,
and likewise $n^+_{S,A',i}$.
Let $r_1 = o_i^{n^+_{S_1, A, i}}$ be the factor associated with $v_i$ in all $\o{A}{S_1}$
where $A(v_i)=T$, and $r_2 = o_i^{n^+_{S_1,A',i}}$ be the factor
associated with $v_i$ in all $\o{A}{S_1}$ where $A(v_i)=F$.
Also, let
$$ R = \sum_{A:\; A(v_i)=T} \prod_{j=1, j\neq i}^n o_i^{n^+_{S_1,A,i}}. $$
Then by Lemma \ref{lem:denormtoprob} and uniformity of the prior (so all terms of the form $\Pr(A)$ cancel),
$$ \Pr(v_i=T \mid S_1) = \frac{ r_1 R }{ r_1 R  + r_2 R }. $$
For $S_2$, the remainder $R$ is the same, while $n^+_{S_1,A,i}=n^+_{S_2,A',i}$ and vice versa, and so likewise
$$ \Pr(v_i=T \mid S_2) = \frac{ r_2 R }{ r_1 R  + r_2 R }. $$
Therefore,
$$ \Pr(v_i=T \mid S_1) = 1- \Pr(v_i=T \mid S_2) $$
and hence
\begin{equation}
\begin{array}{lll}
 (1-p) \Q(\phi\subst{v}{\t} \mid S_1) + p\Q(\phi\subst{v}{\f} \mid
  S_1) \label{phim2}, 
\end{array}
\end{equation}
}
Let $Q(E) = \Pr(E)-\frac{1}{2}$. 
By adding $-1/2$ on both sides, 
equations
(\ref{phim1}) and (\ref{phim2}) hold with $\Pr$
replaced by $Q$, while (\ref{phim0}) becomes
$$Q(\phi_v \mid S_1) = p Q(\phi\subst{v}{\t} \mid S_1) - (1-p)Q(
  \phi\subst{v}{\t} \mid S_1).$$
We now show that  either $|\Q(\phi\mid S_1)| \ge |\Q(\phi_v\mid S_1)|$ or 
$|\Q(\phi\mid S_2)| \ge |\Q(\phi_v\mid S_1)|.$  This suffices to complete the proof.

\commentout{
Suppose $|\Q(\phi|S_1)|<|\Q(\phi_v|S_1)|$.
$|\Q(\phi|S_2)| \ge |\Q(\phi_v|S_2)|$, proving the desired result.

By ($\#$), implicitly using
the positivity assumption $(*)$, the first summands in 
$(\ref{phiprimemprime})$ and $(\ref{phim1})$ are positive, and the
second term in $(\ref{phiprimemprime})$ is negative; so this
necessitates that the second summand in $(\ref{phim1})$ is more
negative, i.e. 
$$\Q(\phi\subst{v}{\f} | S_1) < -\Q(\phi\subst{v}{\t} \mid S_1).$$
So
$$-p\Q(\phi\subst{v}{\f} \mid S_1) > p\Q(\phi\subst{v}{\t} |S_1),$$
and hence
$$\begin{array}{lll}
-\Q(\phi|S_2) &=& -p(\Q(\phi\subst{v}{\f} \mid S_1)) - (1-p) \Q(\phi\subst{v}{\t} \mid S_1) \\
 &>& p\Q(\phi\subst{v}{\t} |S_1) - (1-p) \Q(\phi\subst{v}{\t} \mid S_1) = \Q(\phi_v|S_1).
\end{array}$$
Taking absolute values and using positivity $(*)$ again, we conclude 
$$|\Q(\phi|S_2)| > |\Q(\phi_v|S)|.$$

So either $S_1$ or $S_2$ is a sequence of test outcomes that gives at
least as great absolute bias for $\phi$ as $S_1$ does for $\phi_v$. By
construction, each of these sequences is of the same length as $S$. So
done. 
}
To simplify notation, let $x = \Q(\phi\subst{v}{\t} \mid S_1)$ and let
$y = \Q(\phi\subst{v}{\f} \mid S_1)$.
By (\ref{phim0}), (\ref{phim1}),
and (\ref{phim2}), we want to show that either
$|p x + (1-p) y| \ge |p x - (1-p)x|$ or
$|(1-p) x + py| \ge |p x - (1-p)x|$.  So suppose that
$|p x + (1-p) y| < |p x - (1-p)x|$.  We need to consider four cases:
(1) $p \ge 1/2$, $x \ge 0$; (2) $p \ge 1/2$, $x < 0$; (3) $p < 1/2$, $x
\ge 0$; and (4)  $p < 1/2$, $x < 0$.  For (1) , note that if $p \ge
1/2$ and $x \ge 0$, then $0 \le px - (1-p)x \le px$.  We must have $y < -x$,
for otherwise $px + (1-p)y  \ge px - (1-p)x$.  But then $py + (1-p)x <
-(px - (1-p)x)$, so $|py + (1-p)x| > |px - (1-p)x|$.
For (2), note that if $p \ge 1/2$ and $x < 0$, then $px - (1-p)x < 0$.
We must have $y > -x$, for otherwise $px + (1-p)y \le px - (1-p)x$, and
$|px + (1-p)y| \ge |px - (1-p)x|$.  But then $py + (1-p)x > -px +
(1-p)x$, so $|py + (1-p)x| > |px - (1-p)x|$.
The arguments in cases (3) and (4) are the same as for (1) and (2),
since we can simply replace $p$ by $1-q$.  This gives us identical
inequalities (using $q$ instead of $p$), but now $q > 1/2$.  
\end{proof}

\commentout{
\section*{Notation}

\subsection*{Boolean logic}

Given Boolean formulae $\phi$ and $\psi$, 
\begin{itemize}
\item $\phi \wedge \psi$, read ``$\phi$ and $\psi$'', refers to the
  formula that is true if and only if both $\phi$ and $\psi$ are true;
  \item $\phi \vee \psi$, read ``$\phi$ or $\psi$'', refers to the formula that is true if and only if at least one of $\phi$ and $\psi$ is true;
\item $\phi \oplus \psi$, read ``$\phi$ xor $\psi$'', refers to the formula that is true if and only if exactly one of $\phi$ and $\psi$ is true;
\item $\neg\phi$, read ``not $\phi$'', refers to the formula that is true if and only if $\phi$ is false.
\end{itemize}

Note that $\phi \oplus \psi = (\phi \wedge \neg \psi) \vee (\neg\phi \wedge \psi)$.

For any $\odot \in \{\wedge,\vee,\oplus\}$, we use the ``big $\odot$'' 
$$ \bigodot_{i=1}^n \phi_i $$
as shorthand for the formula 
$$ \phi_1 \odot \ldots \odot \phi_n. $$

}

\fi

\section*{Acknowledgements}
We thank David Goldberg, David Halpern, Bobby Kleinberg, Dana Ron,
Sarah Tan, and Yuwen Wang as well as the anonymous reviewers for
helpful feedback, discussions and advice. This work was supported in
part by NSF grants IIS-1703846 and IIS-1718108,  
AFOSR grant FA9550-12-1-0040, 
ARO grant W911NF-17-1-0592, and a grant from the Open Philanthropy project.

\bibliographystyle{ACM-Reference-Format}
\bibliography{joe}

\end{document}